\documentclass[10pt]{article} 

\usepackage{geometry}
\usepackage{amssymb}
\usepackage{amsthm}
\usepackage{amsmath}
 
\usepackage{algorithmic}

\usepackage[T1]{fontenc}
\DeclareFontFamily{T1}{pzc}{}
\DeclareFontShape{T1}{pzc}{m}{it}{<-> s * [1.10] pzcmi7t}{}
\DeclareMathAlphabet{\mathpzc}{T1}{pzc}{m}{it}
\newcommand{\logvar}[1]{\mathpzc{#1}}
\newcommand{\vetlogvar}[1]{\vec{\logvar{#1}}}

\usepackage{tikz,pgf}
\usetikzlibrary{positioning}
\usetikzlibrary{shapes.geometric}
\usepackage{mdframed}

\newcommand{\pro}{\mathbb{P}}
\newcommand{\pr}[1]{\mathbb{P}\!\left(#1\right)}
\newcommand{\parents}[1]{\mathrm{pa}(#1)}
\newcommand{\pa}[1]{\parents{#1}}

\newcommand{\definitionaxiom}{\equiv\joinrel\equiv}

\newtheorem{Definition}{Definition}
\newtheorem{Theorem}{Theorem}
\newtheorem{Lemma}{Lemma}
\newtheorem{Proposition}{Proposition}
\theoremstyle{definition}
\newtheorem{Example}{Example}

\begin{document}



\title{The Complexity of Bayesian Networks \\ Specified by Propositional and Relational Languages}


\author{Fabio G. Cozman \\ Escola Polit\'ecnica \\ Universidade de S\~ao Paulo
\and 
Denis D. Mau\'a \\ Instituto de Matem\'atica \\ e Estat{\'\i}stica \\ Universidade de S\~ao Paulo}
\date{December 5, 2016}

\maketitle

\begin{abstract}
We examine the complexity of inference in Bayesian networks specified by logical 
languages. We consider representations that range from fragments of propositional 
logic to function-free first-order logic with equality; in doing so we cover a variety 
of plate models and  of probabilistic relational models. We study the complexity of 
inferences when network, query and domain are the input (the {\em inferential} and 
the {\em combined} complexity), when the network is fixed and query and domain 
are the input (the {\em query}/{\em data} complexity), and when the network and 
query are fixed and the domain is the input (the {\em domain} complexity). We 
draw connections with probabilistic databases and liftability results, and obtain
complexity classes that range from polynomial to exponential levels.
\end{abstract}




\section{Introduction}
\label{section:Introduction}

A Bayesian network can represent any distribution over a given set of random
variables \cite{Darwiche2009,Koller2009}, and this flexibility has been used to great 
effect in many applications \cite{Pourret2008}. Indeed, Bayesian networks are routinely
used to carry both deterministic and probabilistic assertions in a variety of knowledge
representation tasks. Many of these tasks contain complex decision problems, with
repetitive patterns of entities and relationships. Thus it is not surprising that
practical concerns have led to modeling languages where Bayesian networks  are
specified using relations, logical
variables, and quantifiers \cite{Getoor2007Book,Raedt2008Book}. Some of 
these languages enlarge Bayesian networks with plates \cite{Gilks93,Lunn2009},
while others resort to elements of  database schema
\cite{Getoor2007,Heckerman2007}; some others mix probabilities with logic
programming \cite{Poole2008,Sato2001} and even with functional 
programming \cite{Mansinghka2014,Milch2005,Pfeffer2001}.
The spectrum of tools that specify Bayesian networks by moving beyond
propositional sentences is vast, and their applications are remarkable. 

Yet most of the existing analysis on the complexity of inference with Bayesian
networks focuses on a simplified setting where nodes of a network are 
associated with categorial variables and distributions are specified by flat
tables containing probability values~\cite{Roth96,Kwisthout2011}. 
This is certainly unsatisfying: as a point of comparison, 
consider the topic of {\em logical} inference,
where much is known about the impact of specific constructs on computational 
complexity --- suffice to mention the beautiful and detailed study of satisfiability
in description logics \cite{Baader2003}.

In this paper we explore the complexity of inferences as dependent on the 
{\em language} that is used to specify the network. 
We adopt a simple specification strategy inspired by probabilistic 
programming \cite{Poole2010} and by structural equation models \cite{Pearl2009}:
A Bayesian network over binary variables is specified by a set of logical equivalences and
a set of independent random variables. Using this simple scheme, we can parameterize
computational complexity by the formal language that is allowed in the logical equivalences;
we can move from sub-Boolean languages to relational ones, in the way producing
languages that are similar in power to plate models~\cite{Gilks93} and to probabilistic 
relational models~\cite{Koller98}. 
Note that we follow a proven strategy adopted in logical formalisms:
we focus on minimal sets of constructs (Boolean operators, quantifiers) that
capture the essential connections between expressivity and complexity, and
that can shed light on this connection for more sophisticated languages if needed. 
Our overall hope is to help with the design of knowledge representation formalims, 
and in that setting it is important to understand the complexity introduced by 
language features, however costly those may be.

To illustrate the sort of specification we contemplate,
consider a simple example that will be elaborated later. 
Suppose we have a population of students, and we denote by 
$\mathsf{fan}(\logvar{x})$ the fact that student $\logvar{x}$ is a fan of say
a particular band. And we write $\mathsf{friends}(\logvar{x},\logvar{y})$
to indicate that $\logvar{x}$ is a friend of $\logvar{y}$. 
Now consider a Bayesian network with a node $\mathsf{fan}(\logvar{x})$
per student, and a node $\mathsf{friends}(\logvar{x},\logvar{y})$ 
per pair of students (see Figure \ref{figure:Friendship}). Suppose each node $\mathsf{fan}(\logvar{x})$
is associated with the assessment $\pr{\mathsf{fan}(\logvar{x})=\mathsf{true}}=0.2$. 
And finally suppose that a person is always a friend of herself,
and two people are friends if they are fans of the band;
that is, for each pair of students, $\mathsf{friends}(\logvar{x},\logvar{y})$ is
associated with the formula
\begin{equation}
\label{equation:FirstFriends}
\mathsf{friends}(\logvar{x},\logvar{y}) \leftrightarrow (\logvar{x}=\logvar{y}) \vee 
  (\mathsf{fan}(\logvar{x}) \wedge \mathsf{fan}(\logvar{y})).
\end{equation}
Now if we have data on some students, we may ask for the probability
that some two students are friends, or the probability that a student is a fan.
We may wish to consider more sophisticated formulas specifying friendship:
how would the complexity of our inferences change, say, if we allowed
quantifiers in our formula? Or if we allowed relations of arity higher than two? 
Such questions are the object of our discussion.

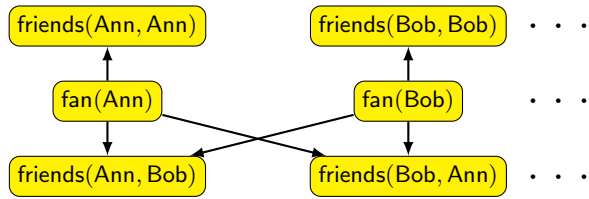
\begin{figure}
\begin{center}
\begin{tikzpicture}
\small
\node[rectangle,rounded corners,draw,fill=yellow] (fa) at (4,2.5) {$\mathsf{fan}(\mathsf{Ann})$};
\node[rectangle,rounded corners,draw,fill=yellow] (fb) at (8,2.5) {$\mathsf{fan}(\mathsf{Bob})$}; 
\node[rectangle,rounded corners,draw,fill=yellow] (fraa) at (4,3.5) {$\mathsf{friends}(\mathsf{Ann},\mathsf{Ann})$};
\node[rectangle,rounded corners,draw,fill=yellow] (frbb) at (8,3.5) {$\mathsf{friends}(\mathsf{Bob},\mathsf{Bob})$}; 
\node[rectangle,rounded corners,draw,fill=yellow] (frab) at (4,1.5) {$\mathsf{friends}(\mathsf{Ann},\mathsf{Bob})$}; 
\node[rectangle,rounded corners,draw,fill=yellow] (frba) at (8,1.5) {$\mathsf{friends}(\mathsf{Bob},\mathsf{Ann})$};  
\draw[->,>=latex,thick] (fa)--(fraa); 
\draw[->,>=latex,thick] (fa)--(frab);
\draw[->,>=latex,thick] (fb)--(frbb);
\draw[->,>=latex,thick] (fb)--(frba);   
\draw[->,>=latex,thick] (fa)--(frba);
\draw[->,>=latex,thick] (fb)--(frab);  
\node at (10,1.5) {\large\bfseries . \bfseries . \bfseries .};
\node at (10,2.5) {\large\bfseries . \bfseries . \bfseries .};
\node at (10,3.5) {\large\bfseries . \bfseries . \bfseries .};
\end{tikzpicture}
\end{center}
\caption{A Bayesian network with a repetitive pattern concerning friendship (only two
students are shown; a larger network is obtained for a larger number of students).}
\label{figure:Friendship}
\end{figure}

In this study,  we distinguish a few concepts. 
{\em Inferential complexity} is the complexity when the network, the query and 
the domain are given as input. When the specification vocabulary is fixed, inference 
complexity is akin to {\em combined complexity} as employed in database theory.
{\em Query complexity} is the complexity when the network is fixed and the
input consists of query  and domain. 
Query complexity has often been defined, in the contex of probabilistic 
databases, as {\em data complexity} \cite{Suciu2011}. 
Finally, {\em domain complexity} is the complexity when network and query 
are fixed, and only the domain is given as input.
Query and domain complexity are directly related respectively
to {\em dqe-liftability} and {\em domain liftability}, concepts that have been
used in lifted inference~\cite{Beame2015,Jaeger2012}.  
We make connections
with lifted inference and probabilistic databases whenever possible, and benefit
from deep results originated from those topics. One of the 
contributions of this paper is a framework that can unify these varied research
efforts with respect to the analysis of Bayesian networks. We show that many non-trivial 
complexity classes characterize the cost of inference as induced by various languages,
and we make an effort to relate our investigation to various 
knowledge representation formalisms, from probabilistic description logics to plates 
to probabilistic relational models.

The paper is organized as follows. 
Section \ref{section:Basics} reviews a few concepts concerning Bayesian networks 
and computational complexity.  Our contributions start in Section \ref{section:Propositional}, 
where we focus on propositional languages. 
In Section \ref{section:Relational} we extend our framework to relational languages,
and review relevant literature on probabilistic databases and lifted inference.
In Sections \ref{section:ComplexityResults} and \ref{section:DLLite} we study
a variety of relational Bayesian network specifications.
In Section \ref{section:PRMs}
we connect these specifications to other schemes proposed in the literature.
And in Section~\ref{section:Valiant} we relate our results, mostly presented 
for decision problems, to Valiant's counting classes and their extensions.
Section~\ref{section:Conclusion} summarizes our findings and proposes future work.

{\em All proofs are collected in  \ref{appendix:Proofs}.}

\section{A bit of notation and terminology}
\label{section:Basics}

We denote by $\pr{A}$ the probability of event $A$. 
In this paper, every random variable $X$ is a function from a finite sample space (usually
a space with finitely many truth assignments or  interpretations) to real numbers (usually to $\{0,1\}$).
We refer to an event $\{X=x\}$ as an {\em assignment}. 
Say that $\{X=1\}$ is a {\em positive}  assignment, and 
$\{X=0\}$ is a {\em negative} assignment. 
 
A {\em graph} consists of a set of {\em nodes} and a set of {\em edges}
(an edge is a pair of nodes), and we focus on graphs that are directed and 
acyclic \cite{Koller2009}.  The parents of a node $X$, for a given graph, are 
denoted $\parents{X}$. Suppose we have a directed acyclic graph $\mathbb{G}$ 
such that each node is a random variable, and we also have a joint probability distribution $\pro$ 
over these random variables. Say that $\mathbb{G}$ and $\pro$ satisfy the Markov 
condition iff each random variable $X$ is independent of its nondescendants 
(in the graph) given its parents (in the graph). 
 
A {\em Bayesian network} is a pair consisting of
a directed acyclic graph $\mathbb{G}$ whose nodes are random variables 
and a joint probability distribution $\pro$ over all variables in the graph, 
such that $\mathbb{G}$ and $\pro$ satisfy the Markov 
condition \cite{Neapolitan2003}. 
For a collection of measurable sets $A_1,\dots,A_n$, we then have
\[
\pr{X_1 \in A_1, \dots,X_n \in A_n} = 
	\prod_{i=1}^n \pr{X_i \in A_i|\parents{X_i} \in {\textstyle \bigcap_{j:X_j \in \parents{X_i}} A_j}}
\]
whenever the conditional probabilities exist.
If all random variables are discrete, then one can specify ``local''  conditional probabilities
$\pr{X_i=x_i|\parents{X_i}=\pi_i}$, and the joint probability distribution is  necessarily
the product of these local probabilities:
\begin{equation}
\label{equation:BayesianNetwork}
\pr{X_1=x_1, \dots,X_n=x_n} = \prod_{i=1}^n \pr{X_i=x_i|\parents{X_i}=\pi_i},
\end{equation}
where $\pi_i$ is the projection of $\{x_1,\dots,x_n\}$ on $\parents{X_i}$, with
the understanding that $\pr{X_i=x_i|\parents{X_i}=\pi_i}$ stands for $\pr{X_i=x_i}$
whenever $X_i$ has not parents.  

In this paper we only deal with finite objects, so we can assume that
a Bayesian network is fully specified by a finite graph and a local conditional probability
distribution per
random variable (the local distribution associated with  random variable $X$ specifies 
the probability of $X$ given the parents of $X$).  
Often probability values are given in tables (referred to as {\em conditional probability tables}).
Depending on how these tables
are encoded, the directed acyclic graph may be redundant; that is, all the information
to reconstruct the graph and the joint distribution is already in the tables. In fact we rarely
mention the graph $\mathbb{G}$ in our results; however graphs are visually useful and 
we often resort to drawing them in our examples.

A basic computational problem for Bayesian networks is:
Given a Bayesian network $\mathbb{B}$, a set of assignments $\mathbf{Q}$ and a set
of assignments $\mathbf{E}$, determine whether  $\pr{\mathbf{Q}|\mathbf{E}}>\gamma$
for some rational number $\gamma$.  
We assume that every probability value is specified as a rational number. 
Thus, $\pr{\mathbf{Q}|\mathbf{E}} = \pr{\mathbf{Q},\mathbf{E}}/\pr{\mathbf{E}}$ is a 
rational number, as $\pr{\mathbf{Q},\mathbf{E}}$ and $\pr{\mathbf{E}}$ are computed
by summing through products given by Expression (\ref{equation:BayesianNetwork}).

We adopt basic terminology and notation from computational
complexity~\cite{Papadimitriou94}. 
A {\em language}  is a set of strings. A language defines a {\em decision problem};
that is, the problem of deciding whether an input string is in the language.
A {\em complexity class} is a set of languages;
we use well-known complexity classes
$\mathsf{P}$, $\mathsf{NP}$, $\mathsf{PSPACE}$, $\mathsf{EXP}$, $\mathsf{ETIME}$,
$\mathsf{NETIME}$.
The complexity class $\mathsf{PP}$ consists of those languages $\mathcal{L}$
that satisfy the following property: there is a polynomial time 
nondeterministic Turing machine $M$ such that $\ell \in \mathcal{L}$ iff more than half
of the computations of $M$ on input $\ell$ end up accepting. Analogously, we 
have $\mathsf{PEXP}$, consisting of those languages $\mathcal{L}$ with the following property:
 there is 
an exponential time  nondeterministic Turing machine $M$ such that $\ell \in \mathcal{L}$ 
iff half of the computations of $M$ on input $\ell$ end up accepting~\cite{Buhrman98}. 

To proceed, we need to define oracles and related complexity classes.
An oracle Turing machine $M^\mathcal{L}$, where $\mathcal{L}$ is 
a language, is a Turing machine with additional tapes, 
such that it can write a string $\ell$ to a tape and obtain from the oracle, in unit time,
the decision as to whether $\ell \in \mathcal{L}$ or not. If a class of languages/functions 
$\mathsf{A}$ is defined by a set of Turing machines $\mathcal{M}$ (that is, the 
languages/functions are decided/computed by these machines), then define 
$\mathsf{A}^\mathcal{L}$ to be the set of languages/functions that 
are decided/computed by $\{ M^\mathcal{L} : M \in \mathcal{M} \}$. 
For a function $f$, an oracle Turing machine $M^f$ can be similarly defined,
and for any class $\mathsf{A}$ we have $\mathsf{A}^f$. 
If $\mathsf{A}$ 
and $\mathsf{B}$ are classes of languages/functions, 
$\mathsf{A}^\mathsf{B} = \cup_{x \in \mathsf{B}} \mathsf{A}^{x}$.  
For instance, the {\em polynomial hierarchy} consists of classes
$\Sigma^\mathsf{P}_i = \mathsf{NP}^{\Sigma^{\mathsf{P}}_{i-1}}$ and
$\Pi^\mathsf{P}_i = \mathsf{co}\Sigma^\mathsf{P}_i$, with
$\Sigma^\mathsf{P}_0 = \mathsf{P}$
(and $\mathsf{PH}$  is the  union 
$ \cup_i \Pi^\mathsf{P}_i = \cup_i \Sigma^\mathsf{P}_i$). 


We examine Valiant's approach to counting problems
in Section \ref{section:Valiant}; for now suffice to say that
$\#\mathsf{P}$ is the class of functions such that $f \in \#\mathsf{P}$
iff $f(\ell)$ is the number of computation paths that accept $\ell$ for some
polynomial time nondeterministic Turing machine \cite{Valiant79TCS}.
 It is as if we had a
special machine, called by Valiant a {\em counting} Turing machine, that on
input $\ell$ prints on a special tape the number of computations that
accept $\ell$.  

We will also use the class $\mathsf{PP}_1$, defined as the set of languages in $\mathsf{PP}$
that have a single symbol as input vocabulary. We can take this symbol to be $1$, so
the input is just a sequence of $1$s (one can interpret this input as a non-negative
integer written in unary notation). This is the counterpart of Valiant's class $\#\mathsf{P}_1$
that consists of the functions in $\#\mathsf{P}$ that have a single symbol as input vocabulary
\cite{Valiant79}.

We focus on many-one reductions: such
a reduction from $\mathcal{L}$ to $\mathcal{L}'$ is a
polynomial time algorithm  that takes the input to decision problem $\mathcal{L}$
and transforms it into the input to decision problem $\mathcal{L}'$ such that 
$\mathcal{L}'$ has the same output as $\mathcal{L}$. A {\em Turing reduction}
from $\mathcal{L}$ to $\mathcal{L}'$  is an polynomial time algorithm that decides 
$\mathcal{L}$ using $\mathcal{L}'$ as an oracle.  
For a complexity class $\mathsf{C}$, a decision problem
$\mathcal{L}$ is $\mathsf{C}$-hard with respect to many-one reductions 
if each decision problem in $\mathsf{C}$ can be reduced to $\mathcal{L}$ with 
many-one reductions.
A decision problem is then $\mathsf{C}$-complete with respect to many-one reductions
 if it is in $\mathsf{C}$
and it is $\mathsf{C}$-hard with respect to many-one reductions. 
Similar definitions of hardness and completeness are obtained when ``many-one reductions''
are replaced by ``Turing reductions''. 

An important $\mathsf{PP}$-complete (with respect to many-one reductions) 
decision problem is MAJSAT: the input is a propositional sentence $\phi$ and 
the decision is whether or not the majority of assignments to the propositions 
in $\phi$ make $\phi$ $\mathsf{true}$~\cite{Gill77}. Another $\mathsf{PP}$-complete
problem (with respect to many-one reductions) is deciding whether the number of 
satisfying assignments for $\phi$ is larger than an input integer $k$ \cite{Simon75PHD};
in fact this problem is still $\mathsf{PP}$-complete with respect to many-one reductions
even if $\phi$ is monotone \cite{Goldsmith2008}. Recall that a sentence is {\em monotone}
if it has no negation. 

A formula is in $k$CNF iff it is in Conjunctive Normal Form with $k$ literals
per clause (if there is no restriction on $k$, we just write CNF). 
MAJSAT is $\mathsf{PP}$-complete with respect to many-one reductions
even if the input is restricted to be in CNF; however, it is not known whether 
MAJSAT is still  $\mathsf{PP}$-complete with respect to
many-one reductions if the sentence $\phi$ is in 3CNF.
Hence we will resort in proofs to a slightly different decision problem, 
following results by Bailey et al.\  \cite{Bailey2007}.
The problem $\#3\mathsf{SAT}(>)$ gets as input a propositional
sentence $\phi$ in 3CNF and an integer $k$, and the decision is whether 
$\#\phi > k$; we use, here and later in proofs, $\#\phi$ to denote 
the number of satisfying assignments for a formula $\phi$. 
We will also use, in the proof of Theorem \ref{theorem:and-or}, the following
decision problem. Say that an assignment to the propositions in a sentence
in CNF {\em repects} the 1-in-3 rule if at most one literal per clause is assigned 
$\mathsf{true}$. Denote by $\#\mbox{(1-in-3)}\phi$ the number of satisfying 
assignments for $\phi$ that also respects the 1-in-3 rule. The decision problem 
$\#\mbox{(1-in-3)}\mathsf{SAT}(>)$ gets as input a propositional sentence 
$\phi$ in 3CNF and an integer $k$, and decides whether $\#\mbox{(1-in-3)}\phi > k$.
We have:

\begin{Proposition}\label{proposition:1-3CNF}
Both $\#3\mathsf{SAT}(>)$ and $\#\mbox{(1-in-3)}\mathsf{SAT}(>)$ are
$\mathsf{PP}$-complete with respect to many-one reductions. 
\end{Proposition}

\section{Propositional languages: Inferential and query complexity}\label{section:Propositional}

In this section we focus on propositional languages, so as to present our proposed 
framework in the most accessible manner. Recall that we wish to parameterize the
complexity of inferences by the language used in specifying local distributions. 

\subsection{A specification framework}

We are interested in specifying Bayesian networks over binary variables $X_1,\dots,X_n$,
where each random variable $X_i$ is the indicator function of a proposition $A_i$.
That is, consider the space $\Omega$ consisting of all truth assignments for these variables (there
are $2^n$ such truth assignments); then $X_i$ yields $1$ for a truth assignment that satisfies
$A_i$, and $X_i$ yields $0$ for a truth assignment that does not satisfy $A_i$.  

{\em We will often use the same letter to refer to a proposition 
and the random variable that is the indicator function of the proposition.}

We adopt a specification strategy that moves away from tables of probability
values, and that is inspired by probabilistic programming \cite{Poole93AI,Sato2001} and by 
structural models \cite{Pearl2000}. A {\em Bayesian network specification} 
associates with each proposition $X_i$ either
\begin{itemize}
\item a logical equivalence $X_i \leftrightarrow \ell_i$, or   
\item a probabilistic assessment $\pr{X_i=1}=\alpha$,
\end{itemize}
where $\ell_i$ is a formula in a propositional language $\mathcal{L}$, such that the only extralogical
symbols in $\ell_i$ are propositions in $\{X_1,\dots,X_n\}$, and $\alpha$ is a rational number in the
interval $[0,1]$.

We refer to each logical equivalence $X_i \leftrightarrow \ell_i$ as a {\em definition axiom}, borrowing
terminology from description logics \cite{Baader2003}. We refer to $\ell_i$ as the {\em body} of
the definition axiom. In order to avoid confusion between 
the leftmost symbol $\leftrightarrow$ and possible logical equivalences within $\ell_i$, we write
a definition axiom as in description logics:
\[
X_i \definitionaxiom \ell_i,
\]
and we emphasize that $\definitionaxiom$ is just syntactic sugar for logical equivalence $\leftrightarrow$.

A Bayesian network specification
induces a directed graph where the nodes are the random variables $X_1,\dotsc,X_n$, 
and $X_j$ is a parent of $X_i$ if and only if the definition axiom for $X_i$ contains $X_j$.
If this graph is acyclic, as we assume in this paper, then the Bayesian network specification
does define a Bayesian network. 
 
Figure \ref{fig:lognet} depicts a Bayesian network specified this way.

\begin{figure}
\begin{center}
\begin{tikzpicture}[thick,->, >=latex]
\node[ellipse,draw] (X1) at (3,1.5) {$Y$};
\node[ellipse,draw] (X2) at (3,0.5) {$X$}; 
\node[ellipse,draw] (Y1) at (1.5,0.5) {$Z_0$}; 
\node[ellipse,draw] (Y2) at (4.5,0.5) {$Z_1$}; 
\draw (X1)--(X2);
\draw (Y1)--(X2);
\draw (Y2)--(X2);
\node at (9.5,1.5) {$\pr{Y=1} = 1/3,$};
\node at (9.5,1) {$\pr{Z_0=1}=1/5, \quad \pr{Z_1=1}=7/10,$};
\node at (9.5,0.5) {$X \definitionaxiom (Y \wedge Z_1) \vee (\neg Y \wedge Z_0)$.};
\end{tikzpicture}
\end{center}
\caption{A Bayesian network specified with logical equivalences and unconditional probabilistic assessments.}
\label{fig:lognet}
\end{figure}

Note that we avoid direct assessments of conditional probability,
because  one can essentially create negation through
$\pr{X=1|Y=1}=\pr{X=0|Y=0}=0$. In our framework, the use of negation is a 
decision about the language. We will see that negation does make a difference when
complexity is analyzed. 

Any distribution over binary variables given by a Bayesian network can be 
equivalently defined using definition axioms,
as long as  definitions are allowed to contain negation and conjunction (and then 
disjunction is syntactic sugar).
To see that, consider a conditional distribution for $X$ given $Y_1$ and $Y_2$;
we can specify this distribution using the definition axiom
\begin{eqnarray*}
X & \definitionaxiom & \left( \neg Y_1 \wedge \neg Y_2 \wedge Z_{00} \right) \vee    
             \left( \neg Y_1 \wedge  Y_2 \wedge Z_{01} \right) \vee  \\
   &   &  \left( Y_1 \wedge \neg Y_2 \wedge Z_{10} \right) \vee  
            \left( Y_1 \wedge Y_2 \wedge Z_{11} \right),
\end{eqnarray*}
where $Z_{ab}$ are fresh binary variables (that do not appear anywhere else), 
associated with  assessments 
$\pr{Z_{ab}=1} = \pr{X=1|Y_1=a, Y_2=b}$.
This sort of encoding can be extended to any set $Y_1,\dots,Y_m$ of parents, 
demanding the same space as the corresponding conditional probability table.

\begin{Example}\label{example:Translation}
Consider a simple Bayesian network with random variables $X$ and $Y$, where
$Y$ is the sole parent of $X$, and where:
\[
\pr{Y=1}  = 1/3, \quad \pr{X=1|Y=0} = 1/5, \quad \pr{X=1|Y=1} = 7/10.
\]
Then Figure \ref{fig:lognet} presents an equivalent specification for this network, in the 
sense that both specifications have the same marginal distribution over $(X,Y)$.
$\Box$
\end{Example}

Note that definition axioms can 
exploit structures that conditional probability tables cannot; for instance,
to create a Noisy-Or gate \cite{Pearl88Book},
we simply say that $X \definitionaxiom (Y_1 \wedge W_1) \vee (Y_2 \wedge W_2)$,
where $W_1$ and $W_2$ are inhibitor variables.  

\subsection{The complexity of propositional languages}

Now consider a language $\mathsf{INF}[\mathcal{L}]$ that consists of the strings 
$(\mathbb{B}, \mathbf{Q}, \mathbf{E}, \gamma)$ for which 
$\pr{\mathbf{Q}|\mathbf{E}}>\gamma$, where
\begin{itemize}
\item $\pro$ is the distribution encoded by a Bayesian network specification
$\mathbb{B}$ with definition axioms whose bodies are formulas in $\mathcal{L}$,
\item $\mathbf{Q}$ and $\mathbf{E}$ are sets of assignments (the {\em query}),
\item and $\gamma$ is a rational number in $[0,1]$. 
\end{itemize}
 
For instance, denote by $\mathsf{Prop}(\wedge,\neg)$   the language of propositional formulas
containing conjunction and negation. Then $\mathsf{INF}[\mathsf{Prop}(\wedge,\neg)]$
is the language that decides the probability of a query for networks specified with
definition axioms containing
conjunction and negation. As every Bayesian network over binary variables can be
specified with such definition axioms,  $\mathsf{INF}[\mathsf{Prop}(\wedge,\neg)]$ 
is in fact a 
$\mathsf{PP}$-complete language \cite[Theorems 11.3 and 11.5]{Darwiche2009}.

There is obvious interest in finding  simple languages $\mathcal{L}$ such that deciding 
$\mathsf{INF}[\mathcal{L}]$ is a tractable problem, so as to facilitate elicitation, decision-making
and learning \cite{Darwiche96,Domingos2012,Jaeger2004,Poon2011,Sanner2005}. 
And there  are indeed propositional languages that generate tractable Bayesian networks:
for instance, it is well known that {\em Noisy-Or} networks display polynomial
inference when the query  consists of negative assignments 
\cite{Heckerman90}. Recall that a Noisy-Or network has a {\em bipartite} graph
with edges pointing from nodes in one set to nodes in the other set,
and the latter nodes are associated with Noisy-Or gates. 

One might think that
tractability can only be attained by imposing some structural conditions on graphs,
given results that connect complexity and graph properties \cite{Kwisthout2010}. 
However, it is possible to attain tractability without restrictions on graph topology.
Consider the following result, where we use $\mathsf{Prop}(\nu)$ to indicate
a propositional language with operators restricted to the ones in the list $\nu$:

\begin{Theorem}\label{theorem:Trivial}
$\mathsf{INF}[\mathsf{Prop}(\wedge)]$ is in  to $\mathsf{P}$ when the query 
$(\mathbf{Q},\mathbf{E})$ contains only positive assignments,  
and $\mathsf{INF}[\mathsf{Prop}(\vee)]$ is in to $\mathsf{P}$ when the query
contains only negative assignments.
\end{Theorem}

As the proof of this result shows (in \ref{appendix:Proofs}), only polynomial
effort is needed to compute probabilities for positive queries in networks 
specified with $\mathsf{Prop}(\wedge)$, {\em even} if one allows root nodes
to be negated (that is, the variables that appear in probabilistic assessments
can appear negated in the body of definition axioms).

Alas, even small movements away from the conditions 
in Theorem \ref{theorem:Trivial} takes us to $\mathsf{PP}$-completeness:

\begin{Theorem} \label{theorem:and-or}
$\mathsf{INF}[\mathsf{Prop}(\wedge)]$ and $\mathsf{INF}[\mathsf{Prop}(\vee)]$
are $\mathsf{PP}$-complete with respect to many-one reductions.
\end{Theorem}

Proofs for these results are somewhat delicate due to the restriction to many-one 
reductions. In \ref{appendix:Proofs} we show that much simpler
proofs for $\mathsf{PP}$-completeness of $\mathsf{INF}(\mathsf{Prop}(\wedge))$ 
and $\mathsf{INF}(\mathsf{Prop}(\vee))$ are possible if one uses Turing reductions.
A Turing reduction gives some valuable information: if a problem is
$\mathsf{PP}$-complete with Turing reductions, then it is unlikely to be polynomial
(for if it were polynomial, then $\mathsf{P}^\mathsf{PP}$ would equal $\mathsf{P}$, 
a highly unlikely result given current assumptions in complexity theory \cite{Toda91}).
However, Turing reductions tend to
blur some significant distinctions. For instance, for Turing reductions it does not matter
whether $\mathbf{Q}$ is a singleton or not:
one can ask for $\pr{Q_1|\mathbf{E}_1}$, $\pr{Q_2|\mathbf{E}_2}$, and so on,
and then obtain $\pr{Q_1,Q_2,\dots|\mathbf{E}}$ as the product of the intermediate computations.
However, it may be the case that for some languages such a distinction concerning $\mathbf{Q}$  
matters. Hence many-one reductions yield stronger results, so we emphasize them throughout 
this papper.
 
One might try to concoct additional languages by using specific logical forms
in the literature  \cite{Darwiche2002}. We leave this to future work; instead of pursuing 
various possible sub-Boolean languages, we wish to quickly examine the {\em query complexity}
of Bayesian networks, and then   move to relational languages in Section~\ref{section:Relational}.

\subsection{Query complexity}

We have so far considered that the input is a string encoding 
a Bayesian network specification $\mathbb{B}$, a query $(\mathbf{Q},\mathbf{E})$, 
and a rational number $\gamma$. However in practice one may face a situation
where the Bayesian network is fixed, and the input is a string consisting of the 
pair $(\mathbf{Q},\mathbf{E})$ and a rational number $\gamma$; the goal is to 
determine whether $\pr{\mathbf{Q}|\mathbf{E}}>\gamma$ with respect to the fixed
Bayesian network. 

Denote by $\mathsf{QINF}[\mathbb{B}]$, where $\mathbb{B}$ is a Bayesian network
specification, the language consisting of each
string $(\mathbf{Q},\mathbf{E},\gamma)$ for which $\pr{\mathbf{Q}|\mathbf{E}}>\gamma$
with respect to $\mathbb{B}$. And denote by $\mathsf{QINF}[\mathcal{L}]$ the set of
languages $\mathsf{QINF}[\mathbb{B}]$ where $\mathbb{B}$ is a Bayesian network
specification with definition axioms whose bodies are formulas in $\mathcal{L}$. 

\begin{Definition}
Let $\mathcal{L}$ be a propositional language and $\mathsf{C}$ be a complexity class.
The {\em query complexity} of $\mathcal{L}$ is $\mathsf{C}$ if and only if 
every language in $\mathsf{QINF}[\mathcal{L}]$ is in~$\mathsf{C}$.
\end{Definition}

The fact that   query complexity may differ from inferential complexity was initially
raised by Darwiche and Provan \cite{Darwiche96}, and has led to a number
of techniques emphasizing compilation of a fixed Bayesian network
\cite{Chavira2008,Darwiche2003}. Indeed the expression ``query complexity''
seems to have been coined by Darwiche \cite[Section 6.9]{Darwiche2009}, without
the formal definition presented here.

The original work by Darwiche
and Provan \cite{Darwiche96} shows how to transform a fixed Bayesian network into
a {\em Query-DAG} such that $\pr{\mathbf{Q}|\mathbf{E}}>\gamma$ can
be decided in linear time. That is:

\begin{Theorem}[Darwiche and Provan \cite{Darwiche96}]\label{theorem:PropositionalQuery}
$\mathsf{QINF}[\mathsf{Prop}(\wedge,\neg)]$ is in $\mathsf{P}$.
\end{Theorem}

Results on query complexity become more  interesting when we move
to relational languages.

\section{Relational Languages: Inferential, query, and domain complexity}
\label{section:Relational}

In this section we extend our specification framework so as to analyze the complexity of relational
languages. Such languages have been used in a variety of applications with repetitive
entities and relationships  \cite{Getoor2007Book,Raedt2008Book}. 

\subsection{Relational Bayesian network specifications}

We start by blending some
terminology and notation by Poole \cite{Poole2003} and by Milch et al.\ \cite{Milch2008}. 

A {\em parameterized random variable}, abbreviated {\em parvariable}, is a function 
that yields, for each combination of its input parameters, a random variable. For 
instance, parvariable $X$ yields a random variable $X(\logvar{x})$ for each $\logvar{x}$. 
In what 
follows, parvariables and their parameters will correspond to relations and their 
logical variables.

We use a {\em vocabulary} consisting of names of relations.  
Every relation $X$ is associated with a non-negative integer called its \emph{arity}. 
We also use logical variables; a logical variable is referred to as a {\em logvar}. 
A vector of logvars $[\logvar{x}_1,\dots,\logvar{x}_k]$ is denoted $\vetlogvar{x}$;
then $X(\vetlogvar{x})$ is an {\em atom}. 
A {\em domain} is a set; in this paper every domain is finite.  When the logvars in 
an atom are replaced
by elements of the domain, we obtain $X(a_1,\dots,a_k)$, a {\em ground atom}, 
often referred to as a {\em grounding} of relation $X$.
An {\em interpretation} $\mathbb{I}$
is a function that assigns to each relation $X$ of arity $k$ a relation on~$\mathcal{D}^k$.
An interpretation can be viewed as a function that assigns $\mathsf{true}$ or $\mathsf{false}$ 
to each grounding $X(\vec{a})$, where $\vec{a}$ is a tuple of elements of the domain.
Typically in logical languages there is a distinction between {\em constants} and 
elements of a domain, but we avoid constants altogether in our discussion (as argued by 
Bacchus, if constants are used within a probabilistic logic, some sort of additional {\em rigidity}
assumption must be used \cite{Bacchus90}). 

Given a domain $\mathcal{D}$,
we can associate with each grounding $X(\vec{a})$ a random variable $\hat{X}(\vec{a})$
over the set of all possible interpretations, such that $\hat{X}(\vec{a})(\mathbb{I})=1$
if interpretation $\mathbb{I}$ assigns $\mathsf{true}$ to $X(\vec{a})$, and
$\hat{X}(\vec{a})(\mathbb{I})=0$ otherwise.
Similarly, we can associate with a relation $X$ a parvariable $\hat{X}$ that yields,
once a domain is given, a random variable $\hat{X}(\vec{a})$ for each grounding
$X(\vec{a})$.  To simplify matters, we use the same symbol for a grounding $X(\vec{a})$ 
and its associated random variable $\hat{X}(\vec{a})$, much as we did with
propositions and their associated random variables. Similarly, we use the same symbol 
for a relation $X$ and its associated parvariable $\hat{X}$. We can then write down 
logical formulas over relations/parvariables, and we can assess probabilities
for relations/parvariables. The next example clarifies the dual use of symbols
for relations/parvariables.

\begin{Example}\label{example:Fitness}
Consider a model of friendship built on top of the example in Section~\ref{section:Introduction}.
Two people are friends if they are both fans of the same band, or if they are linked in some other
unmodeled way, and a person is always a friend of herself. Take relations $\mathsf{friends}$,
$\mathsf{fan}$, and $\mathsf{linked}$. Given a domain, say $\mathcal{D}=\{a,b\}$,
we have the grounding $\mathsf{friends}(a,b)$, whose intended interpretation
is that $a$ and $b$ are friends; we take friendship to be asymmetric so 
$\mathsf{friends}(a,b)$ may hold while $\mathsf{friends}(b,a)$ may not hold.
We also have groundings $\mathsf{fan}(a)$,
$\mathsf{linked}(b,a)$, and so on. Each one of these groundings corresponds to
a random variable that yields $1$ or $0$ when the grounding is respectively
$\mathsf{true}$ or $\mathsf{false}$ is an interpretation. 

The stated facts about friendship might be encoded by an extended
version of  Formula (\ref{equation:FirstFriends}), written here with the symbol
$\definitionaxiom$ standing for logical equivalence:
\begin{equation}
\label{equation:Friends}
\mathsf{friends}(\logvar{x},\logvar{y}) \; \definitionaxiom \; (\logvar{x} = \logvar{y}) \vee           
 (\mathsf{fan}(\logvar{x}) \wedge \mathsf{fan}(\logvar{y})) \vee \mathsf{linked}(\logvar{x},\logvar{y}).
\end{equation}

\begin{figure}
\begin{center}
\begin{tikzpicture}
\node[ellipse,draw,thick] (fan) at (1,1) {$\mathsf{fan}$};
\node[ellipse,draw,thick] (friend) at (4,1) {$\mathsf{friends}$};
\node[ellipse,draw,thick] (linked) at (7,1) {$\mathsf{linked}$};
\draw[->,>=latex,thick] (fan)--(friend);
\draw[->,>=latex,thick] (linked)--(friend);
\end{tikzpicture}
\end{center}
\vspace*{-2ex}
\caption{Representing dependences amongst relations in Example \ref{example:Fitness}.}
\label{figure:ParvariableGraph}
\end{figure}
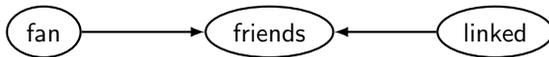

We can draw a directed graph indicating the dependence of $\mathsf{friends}$ on
the other relations, as in Figure \ref{figure:ParvariableGraph}.
Suppose we believe $0.2$ is the probability that an element of the domain is a fan,
and $0.1$ is the probability that two  people are linked for some 
other reason. To express these assesssments we might write
\begin{equation}
\label{equation:Assessments}
\pr{\mathsf{fan}(\logvar{x})=1} = 0.2  \qquad \mbox{ and } \qquad
\pr{\mathsf{linked}(\logvar{x},\logvar{y})=1} = 0.1,
\end{equation}
 with implicit outer universal quantification. 
$\Box$
\end{Example}

Given a formula and a domain, we can produce all groundings of the formula by replacing
its logvars by elements of the domain in every possible way
(as usual when grounding first-order formulas). 
We can similarly ground probabilistic assessments by grounding the
affected relations.

\begin{Example}\label{example:GroundedFitness}
In Example \ref{example:Fitness}, we can produce the following groundings
from domain $\mathcal{D}=\{a,b\}$ and Formula (\ref{equation:Friends}):
\begin{eqnarray*}
\mathsf{friends}(a,a) &\definitionaxiom & (a = a) \vee           
      (\mathsf{fan}(a) \wedge \mathsf{fan}(a)) \vee \mathsf{linked}(a,a), \\
\mathsf{friends}(a,b) &\definitionaxiom & (a = b) \vee           
      (\mathsf{fan}(a) \wedge \mathsf{fan}(b)) \vee \mathsf{linked}(a,b), \\
\mathsf{friends}(b,a) &\definitionaxiom & (b = a) \vee           
      (\mathsf{fan}(b) \wedge \mathsf{fan}(a)) \vee \mathsf{linked}(b,a), \\
\mathsf{friends}(b,b) &\definitionaxiom & (b = b) \vee           
      (\mathsf{fan}(b) \wedge \mathsf{fan}(b)) \vee \mathsf{linked}(b,b),
\end{eqnarray*}
Similarly, we obtain:
\[
\begin{array}{ccc}
\pr{\mathsf{fan}(a)=1} = 0.2,  & \qquad &  \pr{\mathsf{fan}(b)=1} = 0.2, \\
\pr{\mathsf{linked}(a,a)=1} = 0.1, & \qquad & \pr{\mathsf{linked}(a,b)=1} = 0.1, \\
\pr{\mathsf{linked}(b,a)=1} = 0.1, & \qquad & \pr{\mathsf{linked}(b,b)=1} = 0.1, \\
\end{array}
\]
by grounding assessments in Expression (\ref{equation:Assessments}).
$\Box$
\end{Example}

In short:
we wish to extend our propositional framework by specifying Bayesian networks using
both parameterized probabilistic assessments and first-order definitions. 
So, suppose we have a finite set of parvariables, each one of them
corresponding to a relation in a vocabulary. 
A {\em relational Bayesian network specification} associates, with each parvariable
$X_i$, either
\begin{itemize}
\item  a {\em definition axiom} $X_i(\vetlogvar{x})  \definitionaxiom \ell_i(\vetlogvar{x},  Y_1,\dots,Y_m)$, or
\item a {\em probabilistic assessment} $\pr{X(\vetlogvar{x}) =1}=\alpha$, 
\end{itemize}
where 
\begin{itemize}
\item $\ell_i$ is a well-formed formula in a language $\mathcal{L}$, 
containing relations $Y_1,\dots,Y_m$ and free logvars $\vetlogvar{x}$
(and possibly additional logvars bound to quantifiers),
\item and $\alpha$ is a rational number in $[0,1]$.  
\end{itemize}

The formula $\ell_i$ is the {\em body} of the corresponding definition axiom. 
The parvariables that appear in $\ell_i$ are the {\em parents} of
parvariable $X_i$, and are denoted by $\parents{X_i}$. Clearly the definition
axioms induce a directed graph where the nodes are the parvariables and the
parents of a parvariable (in the graph) are exactly $\parents{X_i}$. This is
the {\em parvariable graph} of the relational Bayesian network specification
(this sort of graph is called a {\em template dependency graph} by Koller and
Friedman \cite[Definition 6.13]{Koller2009}). 
For instance, Figure \ref{figure:ParvariableGraph} depicts the
parvariable graph for Example \ref{example:Fitness}.

When the parvariable graph of a relational Bayesian network specification is
acyclic, we say the specification itself is acyclic. {\em In this paper we assume that
  relational Bayesian network specifications are acyclic,} and we do not even
mention this anymore. 

The {\em grounding} of a relational Bayesian network specification $\mathbb{S}$
on a domain $\mathcal{D}$ is defined as follows.
First, produce all groundings of all definition axioms. 
Then, for each parameterized probabilistic assessment $\pr{X(\vetlogvar{x})=1}=\alpha$,
produce its ground probabilistic assessments
\[
\pr{X(\vec{a_1})=1}=\alpha, \quad
\pr{X(\vec{a_2})=1}=\alpha, \quad \mbox{ and so on,}
\]
for all appropriate tuples $\vec{a_j}$ built from the domain. 
The grounded relations, definitions and assessments specify a propositional
Bayesian network that is then the semantics of $\mathbb{S}$ with respect to domain $\mathcal{D}$.

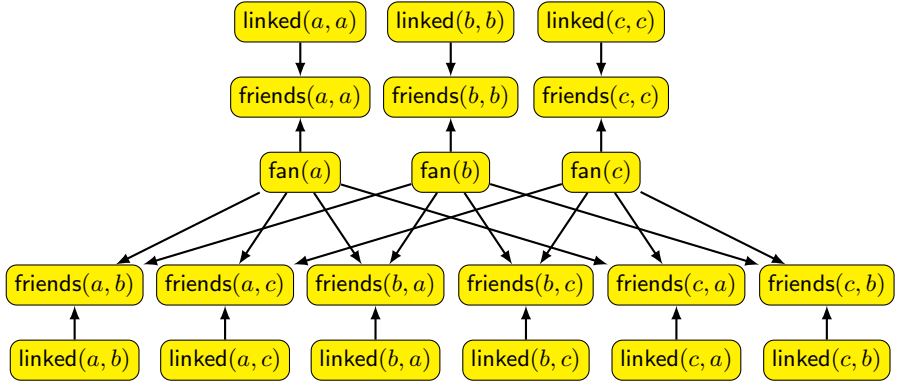
\begin{figure}
\begin{center}
\begin{tikzpicture}
\small
\node[rectangle,rounded corners,draw,fill=yellow] (fa) at (4,2.5) {$\mathsf{fan}(a)$};
\node[rectangle,rounded corners,draw,fill=yellow] (fb) at (6,2.5) {$\mathsf{fan}(b)$};
\node[rectangle,rounded corners,draw,fill=yellow] (fc) at (8,2.5) {$\mathsf{fan}(c)$};
\node[rectangle,rounded corners,draw,fill=yellow] (fraa) at (4,3.5) {$\mathsf{friends}(a,a)$};
\node[rectangle,rounded corners,draw,fill=yellow] (frbb) at (6,3.5) {$\mathsf{friends}(b,b)$};
\node[rectangle,rounded corners,draw,fill=yellow] (frcc) at (8,3.5) {$\mathsf{friends}(c,c)$};
\node[rectangle,rounded corners,draw,fill=yellow] (oaa) at (4,4.5) {$\mathsf{linked}(a,a)$};
\node[rectangle,rounded corners,draw,fill=yellow] (obb) at (6,4.5) {$\mathsf{linked}(b,b)$};
\node[rectangle,rounded corners,draw,fill=yellow] (occ) at (8,4.5) {$\mathsf{linked}(c,c)$};
\node[rectangle,rounded corners,draw,fill=yellow] (frab) at (1,1) {$\mathsf{friends}(a,b)$};
\node[rectangle,rounded corners,draw,fill=yellow] (frac) at (3,1) {$\mathsf{friends}(a,c)$};
\node[rectangle,rounded corners,draw,fill=yellow] (frba) at (5,1) {$\mathsf{friends}(b,a)$};
\node[rectangle,rounded corners,draw,fill=yellow] (frbc) at (7,1) {$\mathsf{friends}(b,c)$};
\node[rectangle,rounded corners,draw,fill=yellow] (frca) at (9,1) {$\mathsf{friends}(c,a)$};
\node[rectangle,rounded corners,draw,fill=yellow] (frcb) at (11,1) {$\mathsf{friends}(c,b)$};
\node[rectangle,rounded corners,draw,fill=yellow] (oab) at (1,0) {$\mathsf{linked}(a,b)$};
\node[rectangle,rounded corners,draw,fill=yellow] (oac) at (3,0) {$\mathsf{linked}(a,c)$};
\node[rectangle,rounded corners,draw,fill=yellow] (oba) at (5,0) {$\mathsf{linked}(b,a)$};
\node[rectangle,rounded corners,draw,fill=yellow] (obc) at (7,0) {$\mathsf{linked}(b,c)$};
\node[rectangle,rounded corners,draw,fill=yellow] (oca) at (9,0) {$\mathsf{linked}(c,a)$};
\node[rectangle,rounded corners,draw,fill=yellow] (ocb) at (11,0) {$\mathsf{linked}(c,b)$};
\draw[->,>=latex,thick] (fa)--(frab);
\draw[->,>=latex,thick] (fa)--(frac);
\draw[->,>=latex,thick] (fa)--(frba);
\draw[->,>=latex,thick] (fa)--(frca);
\draw[->,>=latex,thick] (fb)--(frba);
\draw[->,>=latex,thick] (fb)--(frbc);
\draw[->,>=latex,thick] (fb)--(frab);
\draw[->,>=latex,thick] (fb)--(frcb);
\draw[->,>=latex,thick] (fc)--(frca);
\draw[->,>=latex,thick] (fc)--(frcb);
\draw[->,>=latex,thick] (fc)--(frac);
\draw[->,>=latex,thick] (fc)--(frbc);
\draw[->,>=latex,thick] (oab)--(frab);
\draw[->,>=latex,thick] (oac)--(frac);
\draw[->,>=latex,thick] (oba)--(frba);
\draw[->,>=latex,thick] (obc)--(frbc);
\draw[->,>=latex,thick] (oca)--(frca);
\draw[->,>=latex,thick] (ocb)--(frcb);
\draw[->,>=latex,thick] (fa)--(fraa);
\draw[->,>=latex,thick] (fb)--(frbb);
\draw[->,>=latex,thick] (fc)--(frcc);
\draw[->,>=latex,thick] (oaa)--(fraa);
\draw[->,>=latex,thick] (obb)--(frbb);
\draw[->,>=latex,thick] (occ)--(frcc);
\end{tikzpicture}
\end{center}
\caption{The grounding (on domain $\{a,b,c\}$)
 of the relational Bayesian network specification in Example \ref{example:Fitness}.}
\label{figure:Fitness}
\end{figure}

\begin{Example}\label{example:CompleteFitness}
Consider Example \ref{example:Fitness}. For a domain $\{a,b\}$,
the relational Bayesian network specification given by Expressions  (\ref{equation:Friends}) 
and  (\ref{equation:Assessments})
is grounded into the sentences and assessments in Example \ref{example:GroundedFitness}.
By repeating this process for a larger domain $\{a,b,c\}$, we obtain a larger Bayesian network 
whose graph is depicted in Figure \ref{figure:Fitness}. 
$\Box$
\end{Example}

Note that logical inference might be used to simplify grounded definitions; for instance,
in the previous example, one might note that $\mathsf{friends}(a,a)$ is simply $\mathsf{true}$. 
Note also that the grounding of an formula with quantifiers turns, as usual, an existential quantifier
into a disjunction, and a universal quantifier into a conjunction. Consider the example:

\begin{Example}\label{example:OtherRelationalSpecification}
Take the following relational Bayesian network specification (with no 
particular meaning, just to illustrate a few possibilities):
\[
\begin{array}{c}
\pr{X_1(\logvar{x})=1}=2/3, \qquad \pr{X_2(\logvar{x})=1}=1/10, \\
\pr{X_3(\logvar{x})=1}=4/5,   \qquad   \pr{X_4(\logvar{x},\logvar{y})=1}=1/2, \\
X_5(\logvar{x}) \; \definitionaxiom \; \exists \logvar{y}: \forall \logvar{z}: 
	\neg X_1(\logvar{x}) \vee X_2(\logvar{y}) \vee X_3(\logvar{z}), \\
X_6(\logvar{x}) \; \definitionaxiom \; X_5(\logvar{x}) \wedge \exists \logvar{y} : 
         X_4(\logvar{x},\logvar{y}) \wedge X_1(\logvar{y}), 
\end{array}
\]
Take a domain $\mathcal{D}=\{\mathtt{1},\mathtt{2}\}$; the grounded definition
of $X_5(\mathtt{1})$ is
\begin{eqnarray*}
X_5(\mathtt{1}) & \definitionaxiom &
\left(
(\neg X_1(\mathtt{1}) \vee X_2(\mathtt{1}) \vee X_3(\mathtt{1})) \wedge
(\neg X_1(\mathtt{1}) \vee X_2(\mathtt{1}) \vee X_3(\mathtt{2}))  
\right)
\vee \\
& & 
\left(
(\neg X_1(\mathtt{1}) \vee X_2(\mathtt{2}) \vee X_3(\mathtt{1}))  \wedge
(\neg X_1(\mathtt{1}) \vee X_2(\mathtt{2}) \vee X_3(\mathtt{2}))
\right).
\end{eqnarray*}
Figure \ref{figure:SimpleExample} depicts the parvariable graph and the grounding 
of this relational Bayesian network specification.
$\Box$
\end{Example}

In order to study complexity
questions we must decide   how to encode any given domain. Note that there is no need
to find special names for the elements of the domain, so we take that the domain is always the set
of numbers $\{\mathtt{1}, \mathtt{2}, \dots, \mathtt{N}\}$.
Now if this list is explicitly given as input, then the size of the input is of order $N$.
However, if only the number $N$ is given as input,
then the size of the input is either of order $N$ when $N$ is encoded in unary notation,
or of order $\log N$ when $N$ is encoded in binary notation.
The distinction between unary and binary notation for input numbers is often used
in description logics~\cite{Baader2003}. 

\begin{figure}
\small
\centering
\begin{tikzpicture}
\node[ellipse,draw,thick] (s2) at (0.5,3.5) {$X_2$};
\node[ellipse,draw,thick] (s3) at (2,3.5) {$X_3$};
\node[ellipse,draw,thick] (s1) at (0.5,2) {$X_1$};
\node[ellipse,draw,thick] (s) at (2,2) {$X_5$};
\node[ellipse,draw,thick] (r) at (0.5,0.5) {$X_4$};
\node[ellipse,draw,thick] (t) at (2,0.5) {$X_6$};
\draw[thick,->, >=latex] (s1)--(s);
\draw[thick,->, >=latex] (s2)--(s);
\draw[thick,->, >=latex] (s3)--(s);
\draw[thick,->, >=latex] (r)--(t);
\draw[thick,->, >=latex] (s)--(t);
\draw[thick,->, >=latex] (s1)--(t);

\node[rectangle,rounded corners,draw,fill=yellow] (s2a1) at (5,3.5) {$X_2(\mathtt 1)$};
\node[rectangle,rounded corners,draw,fill=yellow] (s3a1) at (6.5,3.5) {$X_3(\mathtt 1)$};
\node[rectangle,rounded corners,draw,fill=yellow] (s1a1) at (5,2) {$X_1(\mathtt 1)$};
\node[rectangle,rounded corners,draw,fill=yellow] (sa1) at (6.5,2) {$X_5(\mathtt 1)$};
\node[rectangle,rounded corners,draw,fill=yellow] (ra1a1) at (5,0.5) {$X_4(\mathtt 1,\mathtt 1)$};
\node[rectangle,rounded corners,draw,fill=yellow] (ra1a2) at (3.9,1.4) {$X_4(\mathtt 1,\mathtt 2)$};
\node[rectangle,rounded corners,draw,fill=yellow] (ta1) at (6.5,0.5) {$X_6(\mathtt 1)$};

\node[rectangle,rounded corners,draw,fill=yellow] (s2a2) at (8,3.5) {$X_2(\mathtt 2)$};
\node[rectangle,rounded corners,draw,fill=yellow] (s3a2) at (9.5,3.5) {$X_3(\mathtt 2)$};
\node[rectangle,rounded corners,draw,fill=yellow] (s1a2) at (8,2) {$X_1(\mathtt 2)$};
\node[rectangle,rounded corners,draw,fill=yellow] (sa2) at (9.5,2) {$X_5(\mathtt 2)$};
\node[rectangle,rounded corners,draw,fill=yellow] (ra2a1) at (8,0.5) {$X_4(\mathtt 2,\mathtt 1)$};
\node[rectangle,rounded corners,draw,fill=yellow] (ra2a2) at (10.6,1.4) {$X_4(\mathtt 2,\mathtt 2)$};
\node[rectangle,rounded corners,draw,fill=yellow] (ta2) at (9.5,0.5) {$X_6(\mathtt 2)$};

\draw[thick,->, >=latex] (s1a1)--(sa1);
\draw[thick,->, >=latex] (s2a1)--(sa1);
\draw[thick,->, >=latex] (s3a1)--(sa1);
\draw[thick,->, >=latex] (ra1a1)--(ta1);
\draw[thick,->, >=latex] (ra1a2)--(ta1);
\draw[thick,->, >=latex] (sa1)--(ta1);
\draw[thick,->, >=latex] (s1a1)--(ta1);
\draw[thick,->, >=latex] (s2a1)--(sa2);
\draw[thick,->, >=latex] (s3a1)--(sa2);
\draw[thick,->, >=latex] (s1a1)--(ta2);

\draw[thick,->, >=latex] (s1a2)--(sa2);
\draw[thick,->, >=latex] (s2a2)--(sa2);
\draw[thick,->, >=latex] (s3a2)--(sa2);
\draw[thick,->, >=latex] (ra2a1)--(ta2);
\draw[thick,->, >=latex] (ra2a2)--(ta2);
\draw[thick,->, >=latex] (sa2)--(ta2);
\draw[thick,->, >=latex] (s1a2)--(ta2);
\draw[thick,->, >=latex] (s2a2)--(sa1);
\draw[thick,->, >=latex] (s3a2)--(sa1);
\draw[thick,->, >=latex] (s1a2)--(ta1);

\end{tikzpicture}
\caption{The parvariable graph of the
 relational Bayesian network specification in Example~\ref{example:OtherRelationalSpecification}, 
and its grounding on domain $\mathcal{D}=\{\mathtt{1},\mathtt{2}\}$.}
\label{figure:SimpleExample}
\end{figure}
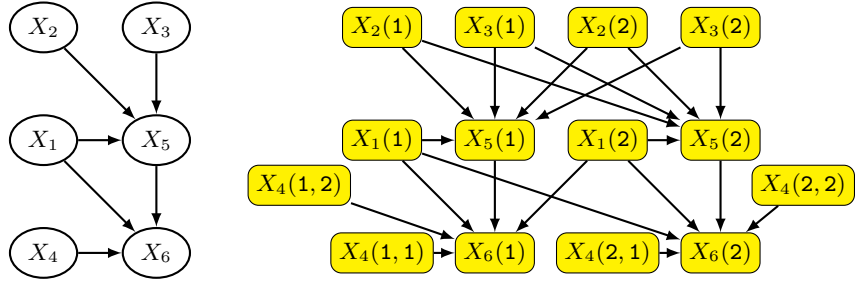

The conceptual difference between unary and binary encodings of domain size
can be captured by the following analogy. Suppose we are interested in the inhabitants 
of a  city: the probabilities that they study, that they marry, that they vote, and so on. 
Suppose the behavior of these inhabitants is modeled by a relational Bayesian network 
specification, and we observe evidence on  a few people. If we then take our input $N$ 
to be in unary notation, we are implicitly assuming that we have a directory, say a mailing 
list, with the names of all inhabitants; even if we do  not care about their specific names, 
each one of them exists concretely in our modeled reality. But if we take our input $N$ 
to be in binary notation, we are just focusing on the impact of city size on probabilities, 
without any regard for the 
actual inhabitants; we may say that $N$ is a thousand, or maybe fifty million
(and perhaps  neither of these numbers is remotely accurate).

\subsection{Inferential, combined, query and domain complexity}

To repeat,
we are interested in the relationship between the language $\mathcal{L}$ that is 
employed in the body of definition axioms and the complexity of inferences. While in the propositional
setting we distinguished between inferential and query complexity, here we have an 
additional distinction to make. Consider the following definitions, where 
$\mathbb{S}$ is a relational Bayesian network specification, 
$N$ is the domain size,
$\mathbf{Q}$ and $\mathbf{E}$ are sets of assignments for ground atoms,   
$\gamma$ is a rational number in $[0,1]$, and
$\mathsf{C}$ is a complexity class:

\begin{Definition}\label{definition:Inferential}
Denote by $\mathsf{INF}[\mathcal{L}]$ the language consisting of  
strings $(\mathbb{S},\!N,\mathbf{Q},\mathbf{E},\gamma)$
for which $\pr{\mathbf{Q}|\mathbf{E}}>\gamma$ with respect
to the grounding of $\mathbb{S}$ on domain of size $N$, where $\mathbb{S}$
contains definition axioms whose bodies are formulas in $\mathcal{L}$.
The {\em inferential complexity} of $\mathcal{L}$ is $\mathsf{C}$
iff $\mathsf{INF}[\mathcal{L}]$ is in $\mathsf{C}$; moreover, the inferential
complexity is $\mathsf{C}$-hard with respect to a reduction iff $\mathsf{INF}[\mathcal{L}]$ is
$\mathsf{C}$-hard with respect to the reduction, and it is $\mathsf{C}$-complete with
respect to a reduction iff it is in $\mathsf{C}$ and it is $\mathsf{C}$-hard with respect
to the reduction.
\end{Definition}

\begin{Definition}\label{definition:Query}
Denote by $\mathsf{QINF}[\mathbb{S}]$ the language consisting of  
strings $(N,\mathbf{Q},\mathbf{E},\gamma)$
for which $\pr{\mathbf{Q}|\mathbf{E}}>\gamma$ with respect
to the grounding of $\mathbb{S}$ on domain of size $N$.
Denote by $\mathsf{QINF}[\mathcal{L}]$ the set of languages 
$\mathsf{QINF}[\mathbb{S}]$ for $\mathbb{S}$ where the bodies
of definition axioms in $\mathbb{S}$  are formulas in $\mathcal{L}$. 
The {\em query complexity} of $\mathcal{L}$ is in $\mathsf{C}$
iff every language in $\mathsf{QINF}[\mathcal{L}]$ is in $\mathsf{C}$;
moreover, the query complexity is $\mathsf{C}$-hard with respect to a reduction
iff some language in $\mathsf{QINF}[\mathcal{L}]$ is $\mathcal{C}$-hard
with respect to the reduction, and it is $\mathsf{C}$-complete with respect
to a redution iff it is in $\mathsf{C}$ and it is $\mathsf{C}$-hard with
respect to the reduction.
\end{Definition}

\begin{Definition}\label{definition:Domain}
Denote by $\mathsf{DINF}[\mathbb{S},\mathbf{Q},\mathbf{E}]$ the language consisting of  
strings $(N,\gamma)$
for which $\pr{\mathbf{Q}|\mathbf{E}}>\gamma$ with respect
to the grounding of $\mathbb{S}$ on domain of size $N$.
Denote by $\mathsf{DINF}[\mathcal{L}]$ the set of languages 
$\mathsf{DINF}[\mathbb{S},\mathbf{Q},\mathbf{E}]$ for $\mathbb{S}$ where the bodies
of definition axioms in $\mathbb{S}$  are formulas in $\mathcal{L}$, and 
 where $\mathbf{Q}$ and $\mathbf{E}$ are sets of assignments.
The {\em domain complexity} of $\mathcal{L}$ is in $\mathsf{C}$
iff every language in $\mathsf{DINF}[\mathcal{L}]$ is in $\mathsf{C}$;
moreover, the domain complexity is $\mathsf{C}$-hard with respect to a reduction
iff  some language in $\mathsf{DINF}[\mathcal{L}]$ is $\mathcal{C}$-hard
with respect to the reduction, and it is $\mathsf{C}$-complete with respect
to a redution iff it is in $\mathsf{C}$ and it is  $\mathsf{C}$-hard with
respect to the reduction.
\end{Definition}

We conclude this section with a number of observations.

\paragraph{Combined complexity}
The definition of inferential complexity imposes no restriction on the vocabulary; 
later we will impose bounds on relation arity. We might instead assume that the 
vocabulary is fixed; in this case we might use the term {\em combined complexity}, 
as this is the term employed in finite model theory and database theory
 to refer to the complexity of model 
checking when both the formula and the model are given as input, but the
vocabulary  is fixed \cite{Libkin2004}. 

\paragraph{Lifted inference}
We note that query and domain complexities are related 
respectively to {\em dqe-liftability} and  {\em domain-liftability}, 
as defined in the study of lifted inference~\cite{Jaeger2012,Jaeger2014}. 

The term ``lifted inference''  is usually attached to algorithms that try to compute 
inferences involving
parvariables without actually producing groundings \cite{Kersting2012,Milch2008,Poole2003}.
A formal definition of lifted inference has been proposed by Van den Broeck
\cite{Broeck2011}:  an algorithm is {\em domain lifted} iff
inference runs in polynomial time with respect to $N$, for fixed model and
query. This definition assumes that $N$ is given in unary notation; if $N$
is given in binary notation, the input is of size $\log N$, and 
a domain lifted algorithm may take exponential time. Domain liftability
has been extended to {\em dqe-liftability}, where the inference must run in 
polynomial time with respect to $N$ and the query, for fixed model \cite{Jaeger2012}.

In short, dqe-liftability means that query complexity is
polynomial, while domain-liftability means that domain complexity is
polynomial. Deep results have been obtained both on the limits of
liftability \cite{Jaeger2012,Jaeger2014},  and on algorithms that attain  liftability 
\cite{Beame2015,Broeck2014,Kazemi2016,Niepert2014,Taghipour2013AISTATS}. 
We will use several of these results in our later proofs. 

We feel that dqe-liftability and domain-liftability are important
concepts but they focus only on a binary choice (polynomial versus non-polynomial);
our goal here is to map languages and complexities in more detail. 
As we have mentioned in Section \ref{section:Introduction}, our main goal
is to grasp the complexity, however high, of language features. 

\paragraph{Probabilistic databases}
Highly relevant material has been produced in the study of probabilistic databases;
that is, databases where   data may be associated with probabilities. 
There exist several probabilistic database systems  
\cite{Dalvi2007,Koch2009,Singh2008,Wang2008,Widom2009}; for instance,
the Trio system lets the user
indicate that $\mathsf{Amy}$ drives an $\mathsf{Acura}$ with probability $0.8$
\cite{Benjelloun2007}. As another example,
the NELL system scans text from the web and builds a database of facts, 
each associated with a number between zero and one~\cite{Mitchell2015}. 

To provide some focus to this overview, we adopt the framework described by
Suciu et al.~\cite{Suciu2011}. Consider  a set of relations, each implemented 
as a table. Each tuple in a table may be associated with a probability. These 
probabilistic tuples are assumed independent (as dependent tuples can be 
modeled from independent ones \cite[Section 2.7.1]{Suciu2011}). 
A probabilistic database management system receives a logical formula 
$\phi(\vetlogvar{x})$ and must  determine, using data and probabilities 
in the tables, the probability $\pr{\phi(\vec{a})}$ for tuples $\vec{a}$. 
The logical formula $\phi(\vetlogvar{x})$ is referred to as the {\em query};
for example, $\phi$ may be a {\em Union of Conjunctive Queries} (a first-order 
formula with equality, conjunction, disjunction and existential quantification). 
Note that the word ``query''  is not used with the meaning usually adopted 
in the context of Bayesian networks; in probabilistic databases, 
a query is a formula whose probability is to be computed.
 
Suppose that all tuples in the table for relation $X(\vetlogvar{x})$ 
are associated with identical probability value $\alpha$. 
This table can be viewed as the grounding of a parvariable $X(\vetlogvar{x})$ 
that is associated with the assessment $\pr{X(\vetlogvar{x})=1}=\alpha$.
Beame et al.\ say that a probabilistic database is {\em symmetric} iff 
each table in the database can be thus associated with a parvariable
and a single probabilistic assessment  \cite{Beame2015}.

Now suppose we have a symmetric probabilistic database and a query 
$\phi$. Because the query is itself a logical formula, results on the 
complexity of computing its probability can be directly mapped to our study
of relational Bayesian network specifications.
This is pleasant because several deep results have been derived on the complexity
probabilistic databases. We later transfer some of those results to obtain
the combined and query complexity of specifications based on first-order logic
and on fragments of first-order logic with bounded number of logvars. 
In Section~\ref{section:ComplexityResults} we also comment on {\em safe queries} 
and associated dichotomy theorems from the literature on probabilistic databases.

A distinguishing characteristic of research on probabilistic databases is
 the intricate search for languages that lead to tractable inferences. 
Some of the main differences
between our goals and the goals of research on probabilistic databases are already
captured by Suciu et al.\  when they compare probabilistic databases and probabilistic
graphical models \cite[Section 1.2.7]{Suciu2011}: while probabilistic databases
deal with simple probabilistic modeling and possibly large volumes of data,
probabilistic graphical models encode complex probability models
whose purpose is to yield conditional probabilities. 
As we have already indicated in our previous discussion of liftability,
our main goal is to  understand the connection between features of a 
knowledge representation formalism and the resulting complexity.
We are not so focused on finding tractable cases, 
even though we are obviously looking for them; in fact, later we present
tractability results for the $\mathsf{DLLite^{nf}}$ language, results that we take
to be are one of the main contributions of this paper.
 
\paragraph{Query or data complexity?}
The definition of query complexity  (Definition \ref{definition:Query})
reminds one of {\em data complexity}
as adopted in finite model theory and in database theory \cite{Libkin2004}. It is thus
not surprising that research on probabilistic databases has used the term ``data
complexity'' to mean the complexity when the database is the only input \cite{Suciu2011}.

In the context of Bayesian networks, usually a ``query'' is a pair $(\mathbf{Q},\mathbf{E})$ 
of assignments. We have adopted such a terminology in this paper.
Now if $\mathbf{Q}$ and $\mathbf{E}$ contain all available data,
there is no real difference between ``query'' and ``data''. Thus
we might have adopted the term ``data complexity'' in this paper as
we only discuss queries that contain all available data.\footnote{In fact 
we have used the term {\em data complexity} in previous work \cite{Cozman2015AAAI}.} 
However we feel  that there are situations where the ``query'' is not equal to the  
``data''. For instance, in {\em probabilistic relational models} one often uses auxiliary 
grounded relations to indicate which groundings are parents of a given grounding  
(we return to this in Section~\ref{section:PRMs}). 
And in  {\em probabilistic logic programming} one can 
use {\em probabilistic facts} to associate probabilities 
with specific groundings \cite{Fierens2015,Poole93AI,Sato2001}. 
In these cases there is a distinction between the ``query'' $(\mathbf{Q},\mathbf{E})$ and 
the ``data''  that regulate parts of the grounded Bayesian network.

Consider another  possible difference between ``query'' and ``data''.
Suppose we have a relational  Bayesian network specification, a formula $\phi$ whose 
probability $\pr{\phi}$  is to be computed, and a table with the probabilities for 
various groundings. Here $\phi$ is the ``query'' and the table is the ``data''
(this sort of arrangement has been used in description logics \cite{Calvanese2006}).
One might then either fix the specification and vary the query 
and the data  (``query'' complexity), 
or fix the specification and the query and vary the data (``data'' complexity).  

It is possible that such distinctions between ``query'' and ``data'' are not found to be
of practical value in future work. For now we prefer to keep open the possibility 
of a fine-grained analysis of complexity, so we use the term ``query complexity'' 
even though our queries are simply sets of assignments containing all available data. 

\section{The complexity of relational Bayesian network specifications}
\label{section:ComplexityResults}

We start with {\em function-free first-order logic with equality}, a language
we denote by $\mathsf{FFFO}$. One might guess that such a powerful language leads
to exponentially hard inference problems. 
Indeed:
\begin{Theorem}\label{theorem:INF-FFFO}
$\mathsf{INF}[\mathsf{FFFO}]$ is $\mathsf{PEXP}$-complete  with respect 
to many-one reductions, regardless of whether the domain is specified in 
unary or binary notation. 
\end{Theorem}

We note that Grove, Halpern and Koller have already argued that counting the number of 
suitably defined  distinct interpretations of monadic first-order logic is hard 
for the class of languages decided by exponential-time counting Turing machines
\cite[Theorem 4.14]{Grove96}. As they do not present a proof of their counting
result (and no similar proof seems to be available in the literature), and as we
need some of the reasoning to address query complexity later, we present a detailed
proof of Theorem~\ref{theorem:INF-FFFO} in \ref{appendix:Proofs}.

We emphasize that when the domain is specified in binary notation
the proof of Theorem \ref{theorem:INF-FFFO} only requires relations 
of arity one. One might hope to find lower complexity classes for 
fragments of $\mathsf{FFFO}$  that go beyond monadic logic but restrict 
quantification. For instance, the popular description logic $\mathcal{ALC}$ restricts
quantification to obtain $\mathsf{PSPACE}$-completeness of 
satisfiability \cite{Baader2003}.  Inspired by this result, we might consider
the following specification language:

\begin{Definition}\label{definition:ALC}
The language $\mathsf{ALC}$ consists of all formulas recursively 
defined so that $X(\logvar{x})$ is a formula where $X$ is a unary relation,
$\neg \phi$ is a formula when $\phi$ is a formula,
$\phi \wedge \varphi$ is a formula when both $\phi$ and $\varphi$ are formulas,
and $\exists \logvar{y}: X(\logvar{x},\logvar{y}) \wedge Y(\logvar{y})$ is a formula 
when $X$ is a binary relation and $Y$ is a unary relation. 
\end{Definition}

However, $\mathsf{ALC}$ does not move us below 
$\mathsf{PEXP}$ when domain size is given in binary notation:

\begin{Theorem}\label{theorem:INF-ALC}
$\mathsf{INF}[\mathsf{ALC}]$ is $\mathsf{PEXP}$-complete with respect
to many-one reductions, when domain
size is given in binary notation.
\end{Theorem} 

Now returning to full $\mathsf{FFFO}$, consider its query complexity.
We divide the analysis in two parts, as the related proofs are quite 
different:\footnote{The query complexity of  {\em monadic} $\mathsf{FFFO}$ seems 
to be open, both for domain in binary and in unary notation; proofs
of Theorems \ref{theorem:QINF-FFFO-binary} and \ref{theorem:QINF-FFFO-unary} 
need relations of arity two.}

\begin{Theorem}\label{theorem:QINF-FFFO-binary}
$\mathsf{QINF}[\mathsf{FFFO}]$ is $\mathsf{PEXP}$-complete with respect 
to many-one reductions, when the domain is specified in binary notation.
\end{Theorem}

\begin{Theorem}\label{theorem:QINF-FFFO-unary}
$\mathsf{QINF}[\mathsf{FFFO}]$ is $\mathsf{PP}$-complete with respect 
to many-one reductions when the domain is specified in unary notation.
\end{Theorem}

As far as domain complexity is concerned, it seems very hard to establish a completeness
result for $\mathsf{FFFO}$ when domain size is given in binary notation.\footnote{One 
might think that, when domain size is given in binary notation,  some small change 
in the proof of Theorem \ref{theorem:DINF-FFFO-unary}  would show that 
$\mathsf{DINF}[\mathsf{FFFO}]$ is complete for a suitable subset of $\mathsf{PEXP}$. 
Alas, it does not seem easy to define a complexity class that can convey the complexity 
of $\mathsf{DINF}[\mathsf{FFFO}]$ when domain size is in binary notation. Finding the precise
complexity class of $\mathsf{DINF}[\mathsf{FFFO}]$ is an open problem.}
We simply rephrase
an ingenious argument by Jaeger \cite{Jaeger2014} to establish:

\begin{Theorem}\label{theorem:DINF-FFFO}
Suppose $\mathsf{NETIME} \neq \mathsf{ETIME}$.
Then $\mathsf{DINF}[\mathsf{FFFO}]$ is not solved in deterministic exponential time,
when the domain size is given in binary notation.
\end{Theorem} 
 
And for domain size in unary notation:

\begin{Theorem}\label{theorem:DINF-FFFO-unary}
$\mathsf{DINF}[\mathsf{FFFO}]$ is $\mathsf{PP}_1$-complete with respect
to many-one reductions, when the domain is given in unary notation.
\end{Theorem}

Theorem \ref{theorem:DINF-FFFO-unary} is in essence implied by a major result 
by Beame et al.~\cite[Lemma 3.9]{Beame2015}: they show that counting the number 
of interpretations for formulas in the {\em three-variable} fragment $\mathsf{FFFO}^3$ 
is $\#\mathsf{P}_1$-complete. The fragment $\mathsf{FFFO}^k$ consists of the formulas 
in $\mathsf{FFFO}$ that employ at most $k$ logvars (note that logvar symbols may be reused 
within a formula, but there is a bounded supply of such symbols) \cite[Chapter 111]{Libkin2004}. 
The proof by Beame et al.\ is rather involved because they are restricted to three logvars;
in \ref{appendix:Proofs} we show that a relatively simple proof of Theorem \ref{theorem:DINF-FFFO-unary} 
is possible when there is no bound on the number of logvars, a small contribution that
may be useful to researchers. 

It is apparent from Theorems \ref{theorem:INF-FFFO}, \ref{theorem:INF-ALC},
\ref{theorem:QINF-FFFO-binary} and \ref{theorem:DINF-FFFO} that we are 
bound to obtain exponential complexity when domain size is given in binary notation.
Hence, from now on we work with domain sizes in unary notation, unless explicitly indicated.

Of course, a domain size in unary notation cannot by itself avoid exponential behavior, 
as an exponentially large number of groundings can be simulated by increasing arity.
For instance,  a domain with two individuals leads to $2^k$ groundings for a relation 
with arity $k$. Hence, we often assume that our relations have bounded arity. We might 
instead assume that the vocabulary is fixed, as done in finite model theory when 
studying combined complexity. We prefer the more general strategy where we bound arity;  
clearly a fixed vocabulary implies a fixed  maximum  arity.

With such additional assumptions, we obtain $\mathsf{PSPACE}$-completeness of 
inferential complexity. 
With a few differences, this result is implied by results by Beame et al.\ in their
important paper \cite[Theorem 4.1]{Beame2015}: they show that counting interpretations
with a fixed vocabulary is $\mathsf{PSPACE}$-complete (that is, they focus on combined
complexity and avoid conditioning assignments). We present a short proof  of
Theorem \ref{theorem:INF-FFFO-PSPACE} within our framework in \ref{appendix:Proofs}. 

\begin{Theorem}\label{theorem:INF-FFFO-PSPACE}
$\mathsf{INF}[\mathsf{FFFO}]$ is $\mathsf{PSPACE}$-complete with respect to 
many-one reductions, when relations have
bounded arity and the domain size is given in unary notation.
\end{Theorem}
 
Note that the proof of Theorem  \ref{theorem:QINF-FFFO-unary} 
is already restricted to arity $2$,  hence $\mathsf{QINF}[\mathsf{FFFO}]$ is 
$\mathsf{PP}$-complete with respect to many-one
reductions when relations have bounded arity (larger than one)
and the domain is given in unary notation.

We now turn to $\mathsf{FFFO}^k$. As we have already noted, this 
sort of language has been studied already, again by Beame et al.,
who have derived their domain and combined complexity \cite{Beame2015}.
In \ref{appendix:Proofs} we present a short proof of the next result, 
to  emphasize that it follows by a simple adaptation of the proof 
of Theorem \ref{theorem:QINF-FFFO-unary}:

\begin{Theorem}\label{theorem:INF-FFFO-k}
$\mathsf{INF}[\mathsf{FFFO}^k]$ is $\mathsf{PP}$-complete with respect to 
many-one reductions, for  all $k \geq 0$, when  
the domain size is given in unary notation.
\end{Theorem}

Query complexity also follows directly from arguments in the proofs of
previous results, as is clear from the proof of the next theorem in \ref{appendix:Proofs}:\footnote{The 
case $k=1$ seems to be open; when $k=1$, query complexity is polynomial when inference is solely 
on unary relations \cite{Broeck2012AAAI,Broeck2014}. When $k=0$ we obtain propositional
networks and then query complexity is polynomial by Theorem \ref{theorem:PropositionalQuery}.}

\begin{Theorem}\label{theorem:QINF-FFFO-k}
$\mathsf{QINF}[\mathsf{FFFO}^k]$ is $\mathsf{PP}$-complete with respect 
to many-one reductions, for all $k \geq 2$, when
domain size is given in unary notation.
\end{Theorem}

Now consider domain complexity for the bounded variable fragment;
previous results in the literature establish this complexity
\cite{Beame2015,Broeck2011,Broeck2014}. In fact, the case $k > 2$ 
is based on a result by Beame et al.\ that we have already alluded to;
in \ref{appendix:Proofs} we present a simplified argument for this result.

\begin{Theorem}\label{theorem:DINF-FFFO-k}
$\mathsf{DINF}[\mathsf{FFFO}^k]$ is $\mathsf{PP}_1$-complete with respect to 
many-one reductions, for $k>2$, and is in $\mathsf{P}$ for $k \leq 2$, when  
the domain size is given in unary notation.
\end{Theorem}

There are important knowledge representation formalisms within bounded-variable fragments of $\mathsf{FFFO}$.
An example is the description logic $\mathcal{ALC}$ that we have discussed before: every sentence 
in this  description logic can be translated to a formula in $\mathsf{FFFO}^2$~\cite{Baader2003}.
Hence we obtain:

\begin{Theorem} \label{theorem:ALC}
Suppose the domain size is specified in unary notation.
Then $\mathsf{INF}[\mathsf{ALC}]$ and $\mathsf{QINF}[\mathsf{ALC}]$
are $\mathsf{PP}$-complete, and $\mathsf{DINF}[\mathsf{ALC}]$ is in $\mathsf{P}$.
\end{Theorem}
 
As a different exercise, we now consider the quantifier-free fragment of
 $\mathsf{FFFO}$. In such a language, every logvar in the body of a definition
axiom must appear in the defined relation, as no logvar is bound to any quantifier.
 Denote this language by $\mathsf{QF}$;
in Section \ref{section:PRMs} we show the close connection between $\mathsf{QF}$
and {\em plate models}. We have:

\begin{Theorem}\label{theorem:QF} 
Suppose relations have bounded arity.
$\mathsf{INF}[\mathsf{QF}]$ and $\mathsf{QINF}[\mathsf{QF}]$ are
$\mathsf{PP}$-complete with respect to many-one  reductions, 
and $\mathsf{DINF}[\mathsf{QF}]$ requires constant computational effort. 
These results hold {\em even if domain size is given in binary notation}.
\end{Theorem} 

As we have discussed at the end of Section \ref{section:Relational}, the literature on 
lifted inference and on probabilistic databases has produced deep results on
query and domain complexity.  One example is the definition of {\em safe queries},
a large class of formulas with tractable query complexity  \cite{Dalvi2012}. Similar classes 
of formulas have been studied for symmetric probabilistic databases \cite{Gribkoff2014}. 
Based on such results in the literature, one might define the language $\mathsf{SAFE}$ 
consisting of safe queries, or look for similar languages with favorable query complexity.
We prefer to move to description logics in the next section, leaving safe queries and
related languages to future work; we prefer to focus on languages whose complexity
can be determined directly from their constructs (note that a sentence can be decided to
be safe in polynomial time, but such a decision requires computational support).

\section{Specifications based on description logics: the $\mathsf{DLLite}$ language}
\label{section:DLLite}

The term ``description logic'' encompasses a rich family of formal languages with
the ability to encode terminologies and assertions about individuals. Those languages 
are now fundamental  knowledge representation tools, as they have
solid semantics and computational guarantees concerning reasoning tasks~\cite{Baader2003}. 
Given the favorable properties of description logics, much effort has been spent in
mixing them with probabilities \cite{Lukasiewicz2008Review}.  

In this section we examine relational Bayesian network specifications based on 
description logics. Such specifications can benefit from well tested tools and offer a natural 
path to encode probabilistic ontologies. Recall that we have already examined the description 
logic $\mathcal{ALC}$ in the previous section.

Typically a description logic  deals with {\em individuals}, {\em concepts},
and {\em roles}. An individual like $\mathsf{John}$ corresponds to a constant
in first-order logic; a concept like $\mathsf{researcher}$ corresponds to a
unary relation in first-order logic; and a role like $\mathsf{buysFrom}$ 
corresponds to a binary relation in first-order logic. 
 A vocabulary
contains a set of individuals plus some {\em primitve} concepts and some
{\em primitive roles}. From these primitive concepts and roles
one can define other concepts and roles using a set of operators. 
For instance, one may allow for concept {\em intersection}: then
$C \sqcap D$ is the intersection of concepts $C$ and $D$. 
Likewise, $C \sqcup D$ is the {\em union}  of $C$ and $D$, and
$\neg C$ is the {\em complement} of $C$. For a role $r$ and a concept $C$,
a common construct is $\forall r.C$, called a {\em value restriction}.
Another common construct is $\exists r.C$, an {\em existential restriction}. 
Description logics often define composition of roles, inverses of roles, 
and even intersection/union/complement of roles.
For instance, usually $r^-$ denotes the {\em inverse} of role $r$. 

The semantics of description logics typically resorts to domains and 
interpretations. A {\em domain} $\mathcal{D}$ is a set. An {\em interpretation} 
$\mathbb{I}$ maps
each individual to an element of the domain, each primitive concept to
a subset of the domain, and each role to a set of pairs of elements
of the domain. And then the semantics of $C \sqcap D$ is fixed by 
$\mathbb{I}(C \sqcap D) = \mathbb{I}(C) \cap \mathbb{I}(D)$. 
Similarly,
$\mathbb{I}(C \sqcup D) = \mathbb{I}(C) \cup \mathbb{I}(D)$
and 
$\mathbb{I}(\neg C) = \mathcal{D} \backslash \mathbb{I}(C)$. 
And for the restricted quantifiers, we have
$\mathbb{I}(\forall r.C) = \{x \in \mathcal{D} : 
                    \forall y : (x,y)\in\mathbb{I}(r) \rightarrow y\in\mathbb{I}(C)\}$
and
$\mathbb{I}(\exists r.C) =
  \{x \in \mathcal{D} : \exists y : (x,y) \in \mathbb{I}(r) \wedge \logvar{y} \in \mathbb{I}(C) \}$.
The semantics of the inverse role $r^-$ is, unsurprisingly, given by
$\mathbb{I}(r^-) = \{(x,y) \in \mathcal{D} \times \mathcal{D} : (y,x) \in \mathbb{I}(r)\}$.

We can translate this syntax and semantics to their counterparts in first-order logic.
Thus $C \sqcap D$ can be read as $C(x) \wedge D(x)$, $C \sqcup D$ as $C(x) \vee D(x)$,
and $\neg C$ as $\neg C(x)$. Moreover, $\forall r.C$ translates into
$\forall y : r(x,y) \rightarrow C(y)$ and $\exists r.C$ translates into
$\exists y : r(x,y) \wedge C(y)$. 

Definition \ref{definition:ALC} introduced the language
$\mathsf{ALC}$ by adopting intersection,  complement,
 and existential restriction (union and value restrictions are then obtained
from the other constructs).  We can go much further than
$\mathsf{ALC}$ in expressivity and still be within the two-variable fragment 
of $\mathsf{FFFO}$; for instance, we can allow for role composition, role
inverses, and Boolean operations on roles. The complexity of such languages
is obtained from results discussed in the previous section.

%
%
%

Clearly we  can also contemplate description logics that are less expressive than $\mathcal{ALC}$
in an attempt to obtain tractability. Indeed, some description logics combine
selected Boolean operators with restricted quantification to obtain polynomial complexity
of logical inferences. Two notable such description logics are $\mathcal{EL}$ and
DL-Lite; due to their favorable balance between expressivity and complexity, they
are the basis of existing standards for knowledge representation.\footnote{Both
$\mathcal{EL}$ and DL-Lite define standard profiles of the OWL knowledge
representation language, as explained at http://www.w3.org/TR/owl2-profiles/.}

Consider first the description logic $\mathcal{EL}$, where the only allowed 
operators are intersection and existential restrictions, and where the {\em top}
concept is available, interpreted as the whole domain \cite{Baader2003IJCAI}. 
Note that we can translate every sentence of $\mathcal{EL}$ into the negation-free
fragment of $\mathsf{ALC}$, and we can simulate the top concept with the 
assessment $\pr{\top=1}=1$. Thus we take the language $\mathsf{EL}$ as the negation-free
fragment of $\mathsf{ALC}$. Because $\mathsf{EL}$ contains conjunction, we easily
have that $\mathsf{INF}[\mathsf{EL}]$ is $\mathsf{PP}$-hard by Theorem \ref{theorem:and-or}.
And domain complexity is polynomial as implied by   $\mathsf{DINF}[\mathsf{ALC}]$.
Query complexity requires some additional work as discussed in 
\ref{appendix:Proofs}; altogether, we have:\footnote{The proof of 
Theorem \ref{theorem:EL} uses queries with negative assignments;  both the 
inferential/query complexity of $\mathsf{EL}$ 
are open when the query is restricted to positive assignments.}

\begin{Theorem} \label{theorem:EL}
Suppose the domain size is specified in unary notation.
Then $\mathsf{INF}[\mathsf{EL}]$ and 
$\mathsf{QINF}[\mathsf{EL}]$ are $\mathsf{PP}$-complete with respect to
many-one reductions, even if the query contains only positive
assignments, and $\mathsf{DINF}[\mathsf{EL}]$ is in $\mathsf{P}$.
\end{Theorem}

We can present more substantial results when we focus on the negation-free 
fragment of the popular description logic  DL-Lite \cite{Calvanese2005}. 
DL-Lite is particularly interesting because it captures central
features of ER or UML diagrams, and yet common inference services 
have polynomial complexity~\cite{Artale2009}.

The simplicity and computational efficiency of the DL-Lite
language have led many researchers to mix them with probabilities. 
For instance, D'Amato et~al.~\cite{Damato2008} propose a variant
of DL-Lite where the interpretation of each sentence is conditional on a context
that is specified by a Bayesian network.
A similar approach was taken by Ceylan and Pe\~nalosa  
\cite{Ceylan2014}, with minor semantic differences. A different approach 
is to extend the syntax of DL-Lite sentences with probabilistic subsumption
connectives, as in the {\em Probabilistic DL-Lite} \cite{Ramachandran2012}.
Differently from our focus here, none of those proposals employ DL-Lite to
specify Bayesian networks.

In DL-Lite one has primitive concepts as before, and also {\em basic concepts}:
a basic concept is either a primitive concept, or $\exists r$ for a role $r$, or 
$\exists r^-$ for a role $r$. Again, $r^-$ denotes the inverse of $r$. 
And then a concept in DL-Lite is either a basic concept, or
$\neg C$ when $C$ is a basic concept, or $C \sqcap D$ when $C$ and $D$ are
concepts. The semantics of $r^-$, $\neg C$ and $C \sqcap D$ are as before,
and the semantics of $\exists r$ is, unsurprisingly, given by
$\mathbb{I}(\exists r) = \{x \in \mathcal{D} : \exists y : (x,y) \in \mathbb{I}(r)\}$.

We focus on the negation-free fragment of DL-Lite; that is, we consider:

\begin{Definition}
The language $\mathsf{DLLite^{nf}}$ consists of all formulas recursively 
defined so that $X(\logvar{x})$ are formulas when $X$ is a unary relation,
$\phi \wedge \varphi$ is a formula when both $\phi$ and $\varphi$ are 
formulas, and
$\exists \logvar{y}: X(\logvar{x},\logvar{y})$ and 
$\exists \logvar{y}: X(\logvar{y},\logvar{x})$  
are formulas when $X$ is a binary relation. 
\end{Definition}

\begin{Example} \label{example1}
The following definition axioms express a few
 facts about families:
\[
\begin{array}{c}
\mathsf{female}(\logvar{x}) \definitionaxiom \neg \mathsf{male}(\logvar{x}), \\
\mathsf{father}(\logvar{x}) \definitionaxiom \mathsf{male}(\logvar{x}) \wedge 
  \exists \logvar{y} : \mathsf{parentOf}(\logvar{x},\logvar{y}),  \\
\mathsf{mother}(\logvar{x}) \definitionaxiom \mathsf{female}(\logvar{x}) \wedge 
   \exists \logvar{y} : \mathsf{parentOf}(\logvar{x},\logvar{y}), \\  
\mathsf{son}(\logvar{x}) \definitionaxiom \mathsf{male}(\logvar{x}) \wedge 
\exists \logvar{y} : \mathsf{parentOf}(\logvar{y},\logvar{x}), \\
\mathsf{daughter}(\logvar{x}) \definitionaxiom \mathsf{female}(\logvar{x}) \wedge 
  \exists \logvar{y} : \mathsf{parentOf}(\logvar{y},\logvar{x}).
\end{array}
\]
For domain $\mathcal{D}=\{\mathtt{1},\mathtt{2}\}$, this relational Bayesian 
network is grounded into the Bayesian network in Figure~\ref{fig:bn}.
$\Box$
\end{Example}

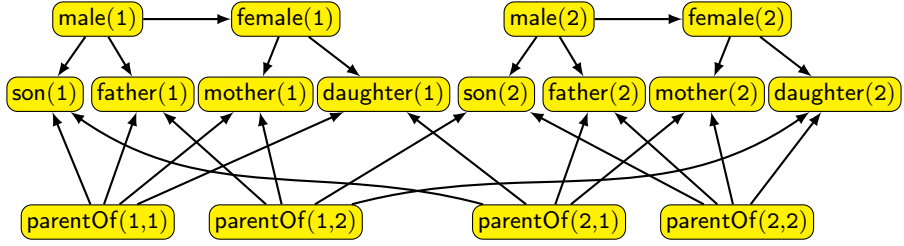
\begin{figure}[t]
  \centering
\small
  \begin{tikzpicture}
      \tikzstyle{every node}=[draw,rectangle,rounded corners,inner sep=2pt,fill=yellow]
      \begin{scope}
      \node (malea) at (-1.9,0) {\textsf{male}($\mathtt{1}$)};
      \node (femalea) at (0.6,0) {\textsf{female}($\mathtt{1}$)};
      \node (fathera) at (-1.3,-1) {\textsf{father}($\mathtt{1}$)};
      \node (mothera) at (0.2,-1) {\textsf{mother}($\mathtt{1}$)};
      \node (sona) at (-2.6,-1) {\textsf{son}($\mathtt{1}$)};
      \node (daughtera) at (1.9,-1) {\textsf{daughter}($\mathtt{1}$)};
      \node (parentOfaa) at (-1.9,-2.7) {\textsf{parentOf}($\mathtt{1}$,$\mathtt{1}$)};
      \node (parentOfab) at (0.6,-2.7) {\textsf{parentOf}($\mathtt{1}$,$\mathtt{2}$)};
      \draw[thick,->, >=latex] (malea) -- (sona);
      \draw[thick,->, >=latex] (malea) -- (fathera);
      \draw[thick,->, >=latex] (femalea) -- (mothera);
      \draw[thick,->, >=latex] (femalea) -- (daughtera);
      \draw[thick,->, >=latex] (malea) -- (femalea);
      \end{scope}

      \begin{scope}[xshift=6cm]
      \node (maleb) at (-1.9,0) {\textsf{male}($\mathtt{2}$)};
      \node (femaleb) at (0.6,0) {\textsf{female}($\mathtt{2}$)};
      \node (fatherb) at (-1.3,-1) {\textsf{father}($\mathtt{2}$)};
      \node (motherb) at (0.2,-1) {\textsf{mother}($\mathtt{2}$)};
      \node (sonb) at (-2.6,-1) {\textsf{son}($\mathtt{2}$)};
      \node (daughterb) at (1.9,-1) {\textsf{daughter}($\mathtt{2}$)};
      \node (parentOfba) at (-1.9,-2.7) {\textsf{parentOf}($\mathtt{2}$,$\mathtt{1}$)};
      \node (parentOfbb) at (0.6,-2.7) {\textsf{parentOf}($\mathtt{2}$,$\mathtt{2}$)};
      \draw[thick,->, >=latex] (maleb) -- (sonb);
      \draw[thick,->, >=latex] (maleb) -- (fatherb);
      \draw[thick,->, >=latex] (femaleb) -- (motherb);
      \draw[thick,->, >=latex] (femaleb) -- (daughterb);
      \draw[thick,->, >=latex] (maleb) -- (femaleb);
      \end{scope}

      \draw[thick,->, >=latex] (parentOfaa) -- (fathera);
      \draw[thick,->, >=latex] (parentOfab) -- (fathera);
      \draw[thick,->, >=latex] (parentOfba) -- (fatherb);
      \draw[thick,->, >=latex] (parentOfbb) -- (fatherb);

      \draw[thick,->, >=latex] (parentOfaa) -- (mothera);
      \draw[thick,->, >=latex] (parentOfab) -- (mothera);
      \draw[thick,->, >=latex] (parentOfba) -- (motherb);
      \draw[thick,->, >=latex] (parentOfbb) -- (motherb);

      \draw[thick,->, >=latex] (parentOfaa) -- (sona);
      \draw[thick,->, >=latex] (parentOfba) edge[out=165,in=-35] (sona);
      \draw[thick,->, >=latex] (parentOfab) -- (sonb);
      \draw[thick,->, >=latex] (parentOfbb) -- (sonb);

      \draw[thick,->, >=latex] (parentOfaa) -- (daughtera);
      \draw[thick,->, >=latex] (parentOfba) -- (daughtera);
      \draw[thick,->, >=latex] (parentOfab) edge[out=15,in=-145] (daughterb);
      \draw[thick,->, >=latex] (parentOfbb) -- (daughterb);
  \end{tikzpicture}
\caption{Grounding of relational Bayesian network specification from  
Example \ref{example1} on domain $\mathcal{D}=\{\mathtt{1},\mathtt{2}\}$.}
  \label{fig:bn}
\end{figure}

We again have that $\mathsf{INF}[\mathsf{DLLite^{nf}}]$ is $\mathsf{PP}$-hard
by Theorem \ref{theorem:and-or}. However, inferential complexity becomes
polynomial when the query is positive:
\begin{Theorem}\label{theorem:DLLITE}
Suppose the domain size is specified in unary notation. Then
$\mathsf{DINF}[\mathsf{DLLite^{nf}}]$ is in $\mathsf{P}$; also, 
$\mathsf{INF}[\mathsf{DLLite^{nf}}]$ and $\mathsf{QINF}[\mathsf{DLLite^{nf}}]$
are in $\mathsf{P}$ when the query $(\mathbf{Q},\mathbf{E})$ 
contains only positive assignments. 
\end{Theorem}

In proving this result (in \ref{appendix:Proofs}) we show that an inference with a
positive query can be reduced to a particular tractable model counting problem. 
The analysis of this model counting problem is a result of independent interest.

Using the model counting techniques we just mentioned, we can also show that a related
problem, namely finding the most probable explanation, is polynomial
for relational Bayesian network specifications based on $\mathsf{DLLite^{nf}}$. 
To understand this,
consider a relational Bayesian network $\mathbb{S}$ based on $\mathsf{DLLite^{nf}}$, 
a set of assignments $\mathbf{E}$ for ground atoms,
and a domain size $N$. Denote by $\mathbf{X}$ the set of random variables
that correspond to groundings of relations in $\mathbb{S}$. 
Now there is at least one set of assignments $\mathbf{M}$ such that:
(i) $\mathbf{M}$ contains assignments to all random variables in $\mathbf{X}$
that are not mentioned in $\mathbf{E}$; and 
(ii) $\pr{\mathbf{M},\mathbf{E}}$ is maximum over all such sets of assignments.
Denote by $\mathsf{MLE}(\mathbb{S},\mathbf{E},N)$ the function problem that consists
in generating such a set of assignments $\mathbf{M}$. 

\begin{Theorem} \label{mpe}
Given a relational Bayesian network $\mathbb{S}$ based on $\mathsf{DLLite^{nf}}$,
a set of positive assignments to grounded relations $\mathbf{E}$,
and a domain size $N$ in unary notation,  
$\mathsf{MLE}(\mathbb{S},\mathbf{E},N)$ can be solved in polynomial time.
\end{Theorem}
 
These results on $\mathsf{DLLite^{nf}}$ can be directly extended in some other
important ways. For example, suppose we allow negative groundings of
roles in the query. Then most of  the proof of
Theorem \ref{theorem:DLLITE} follows (the difference is that the intersection
graphs used in the proof do not satisfy the same symmetries); we can then
resort to approximations for weighted edge cover counting
\cite{Liu2014}, so as to develop a fully polynomial-time approximation
scheme (FPTAS)  for inference. Moreover, the $\mathsf{MLE}(\mathbb{S},\mathbf{E},N)$ 
problem remains polynomial.
Similarly, we could allow for different groundings of the same relation to
be associated with different probabilities; the proofs given in \ref{appendix:Proofs} 
can be modified to develop a FPTAS for inference. 
%

We have so far presented results for a number of languages.
Table \ref{table:Results} summarizes most of our findings;
as noted previously, most of the results on $\mathsf{FFFO}$ with
bound on relation arity  and on $\mathsf{FFFO}^k$ have been in 
essence derived by Beame et al.\ \cite{Beame2015}. 

\begin{table}[t]
\begin{center}
\begin{tabular}{|c|c|c|c|} \hline
Language ($N$ in unary notation) & Inferential & Query & Domain \\ \hline \hline
$\mathsf{Prop}(\wedge)$, positive query & $\mathsf{P}$ & $\mathsf{P}$ & --- \\ \hline
$\mathsf{Prop}(\wedge,\neg)$, $\mathsf{Prop}(\wedge)$, $\mathsf{Prop}(\vee)$ & 
	$\mathsf{PP}$ & $\mathsf{P}$ & --- \\ \hline
$\mathsf{FFFO}$ & $\mathsf{PEXP}$ & $\mathsf{PP}$ & $\mathsf{PP}_1$ \\ \hline
$\mathsf{FFFO}$ with bound on relation arity & $\mathsf{PSPACE}$ & $\mathsf{PP}$ 
& $\mathsf{PP}_1$ \\ \hline
$\mathsf{FFFO}^k$ with $k \geq 3$ & $\mathsf{PP}$ & $\mathsf{PP}$ & $\mathsf{PP}_1$ \\ \hline
$\mathsf{QF}$ with bound on arity & $\mathsf{PP}$ & $\mathsf{PP}$ & $\mathsf{P}$ \\ \hline
$\mathsf{ALC}$ & $\mathsf{PP}$ & $\mathsf{PP}$ & $\mathsf{P}$ \\ \hline
$\mathsf{EL}$ & $\mathsf{PP}$ & $\mathsf{PP}$ & $\mathsf{P}$ \\ \hline
$\mathsf{DLLite^{nf}}$  & $\mathsf{PP}$ & $\mathsf{PP}$ & $\mathsf{P}$ \\ \hline
$\mathsf{DLLite^{nf}}$, positive query & $\mathsf{P}$ & $\mathsf{P}$ & $\mathsf{P}$ \\ \hline
\end{tabular}
\end{center}
\caption{Inferential, query and domain complexity for relational Bayesian networks
based on various logical languages with domain size given in unary notation. All cells 
indicate completeness with respect to many-one reductions.
On top of these results, note that when domain size is given in 
binary notation we have, with respect to many-one reductions: 
$\mathsf{INF}[\mathsf{FFFO}]$ is $\mathsf{PEXP}$-complete (even when 
restricted to relations of arity $1$), $\mathsf{QINF}[\mathsf{FFFO}]$ is
$\mathsf{PEXP}$-complete (even when restricted to relations of arity $2$),
and $\mathsf{INF}[\mathsf{ALC}]$ is $\mathsf{PEXP}$-complete.}
\label{table:Results}
\end{table}

\section{Plates, probabilistic relational models, and related specification languages}
\label{section:PRMs}

In this paper we have followed a strategy that has long been cherished in the study 
of formal languages; that is, we have focused on languages that are based on minimal 
sets of constructs borrowed from logic. Clearly this plan succeeds only to the extent 
that results can be transferred to practical specification languages. In this section we 
examine some important cases where our strategy pays off.

Consider, for instance, {\em plate models}, a rather popular specification
formalims. Plate models have been extensively used in statistical practice
\cite{Lunn2012} since they were introduced in the BUGS project
\cite{Gilks93,Lunn2009}. In machine learning, they have been used  
to convey several models since their first appearance~\cite{Buntine94}. 

There seems to be no standard formalization for plate models, so we adapt 
some of our previous concepts as needed. A plate model consists of a set of 
parvariables, a directed acyclic graph where each node is a parvariable, and 
a set of {\em template conditional probability distributions}. Parvariables are 
{\em typed}: each parameter of a parvariable is associated with a set, the 
{\em domain} of the parameter.
All parvariables that share a domain are said to belong to a {\em plate}.
The central constraint on ``standard'' plate models is that the domains
that appear in the parents of a parvariable must appear in the parvariable. 
For a given parvariable $X$, its corresponding template conditional probability 
distribution associates a probability value to each value of $X$ given each 
configuration of parents of $X$. 
 
To make things simple, here we focus on parvariables that correspond to relations,
thus every random variable has values $\mathsf{true}$ and $\mathsf{false}$
(plate models in the literature often specify discrete and even 
continuous random variables \cite{Lunn2012,Sontag2011}). Our next
complexity results do not really change if one allows parvariables to have
a finite number of values.

We can use the same semantics as before to interpret plate models,
with a small change: now the groundings of a relation are produced by running
only over the domains of its associated logvars. 

\begin{Example} \label{ex:university1}
Suppose we are interested in a ``University World'' containing a population 
of students and a population of courses \cite{Getoor2007}. 
Parvariable  $\mathsf{Failed?}(\logvar{x},\logvar{y})$ yields the  final status of
student $\logvar{y}$ in course $\logvar{x}$; $\mathsf{Difficult?}(\logvar{x})$ is 
a parvariable indicating the difficulty of a course $\logvar{x}$, and
$\mathsf{Committed?}(\logvar{y})$ is a parvariable indicating the commitment 
of student  $\logvar{y}$.

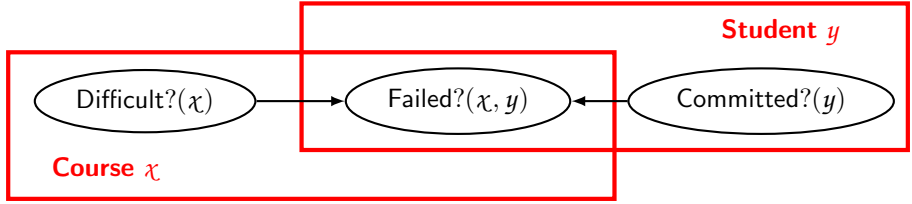
\begin{figure*}[t]
\begin{center}
\begin{tikzpicture}[thick,->, >=latex,scale=1.3]
\draw[red,ultra thick] (5,0) rectangle (11.2,1.5);
\draw[red,ultra thick] (8,0.5) rectangle (14.2,2);
\node[red] at (6,0.3) {\textsf{\textbf{Course}} $\logvar{x}$};
\node[red] at (12.9,1.7) {\textsf{\textbf{Student}} $\logvar{y}$};
\node[ellipse,rounded corners,draw] (G) at (9.6,1) {$\mathsf{Failed?}(\logvar{x},\logvar{y})$};
\node[ellipse,rounded corners,draw] (D) at (6.4,1) {$\mathsf{Difficult?}(\logvar{x})$};
\node[ellipse,rounded corners,draw] (I) at (12.7,1) {$\mathsf{Committed?}(\logvar{y})$};
\draw (D)--(G);
\draw (I)--(G);
\end{tikzpicture}
\end{center}
\vspace*{-3ex}
\caption{Plate model for the University World. We show logvars explicitly, even though they
are not always depicted in plate models.}
\label{figure:plates}
\end{figure*}

A plate model is drawn in Figure \ref{figure:plates}, where plates are rectangles. 
Each parvariable is  associated with a template conditional probability
distribution:
\[ 
\pr{\mathsf{Difficult?}(\logvar{x})=1}=0.3, \qquad \pr{\mathsf{Committed?}(\logvar{y})=1}=0.7, 
\]
\[
\pr{\mathsf{Failed?}(\logvar{x},\logvar{y})=1 \middle|
\begin{array}{c}\mathsf{Difficult?}(\logvar{x})=d, \\ \mathsf{Committed?}(\logvar{y})=c \end{array}} = \left\{
\begin{array}{ll}
0.4 & \mbox{ if } d = 0, c = 0; \\
0.2 & \mbox{ if } d = 0, c = 1; \\
0.9 & \mbox{ if } d = 1, c = 0; \\
0.8 & \mbox{ if } d = 1, c = 1. \quad \Box
\end{array} \right. 
\]
\end{Example}

Note that plate models can always be specified using definition axioms in the
quantifier-free fragment of $\mathsf{FFFO}$, given that the logvars of a relation
appear in its parent relations. 
For instance, the table in Example \ref{ex:university1} can be encoded as follows:
\begin{equation}
\label{equation:Failed}
\mathsf{Failed?}(\logvar{x},\logvar{y})   \definitionaxiom  
\left(\begin{array}{c}
\left(\neg \mathsf{Difficult?}(\logvar{x}) \wedge \neg \mathsf{Committed?}(\logvar{y}) 
\wedge \mathsf{A_1}(\logvar{x},\logvar{y}) \right) \vee \\
\left( \neg \mathsf{Difficult?}(\logvar{x}) \wedge  \mathsf{Committed?}(\logvar{y}) 
\wedge \mathsf{A_2}(\logvar{x},\logvar{y}) \right) \vee \\
\left(  \mathsf{Difficult?}(\logvar{x}) \wedge \neg \mathsf{Committed?}(\logvar{y}) 
 \wedge \mathsf{A_3}(\logvar{x},\logvar{y}) \right) \vee \\
\left(  \mathsf{Difficult?}(\logvar{x}) \wedge  \mathsf{Committed?}(\logvar{y}) 
 \wedge \mathsf{A_4}(\logvar{x},\logvar{y}) \right)
\end{array}\right),
\end{equation}
where we introduced four auxiliary parvariables with associated assessments
$\pr{\mathsf{A_1}(\logvar{x},\logvar{y})=1}=0.4$, 
$\pr{\mathsf{A_2}(\logvar{x},\logvar{y})=1}=0.2$, 
$\pr{\mathsf{A_3}(\logvar{x},\logvar{y})=1}=0.9$, and
$\pr{\mathsf{A_4}(\logvar{x},\logvar{y})=1}=0.8$.

Denote by $\mathsf{INF}[\mathsf{PLATE}]$ the language consisting of inference
problems as in Definition \ref{definition:Inferential}, where relational Bayesian
network specifications are restricted to satisfy the constraints of plate models. 
Adopt $\mathsf{QINF}[\mathsf{PLATE}]$ and $\mathsf{DINF}[\mathsf{PLATE}$ similarly.
We can reuse arguments in the proof of
Theorem \ref{theorem:QF} to show that:

\begin{Theorem}\label{theorem:Plates}
$\mathsf{INF}[\mathsf{PLATE}]$ and $\mathsf{QINF}[\mathsf{PLATE}]$ are 
$\mathsf{PP}$-complete with respect
to many-one reductions, and $\mathsf{DINF}[\mathsf{PLATE}]$ requires constant
computational effort. These results hold {\em even if the domain size is given in
binary notation}. 
\end{Theorem}

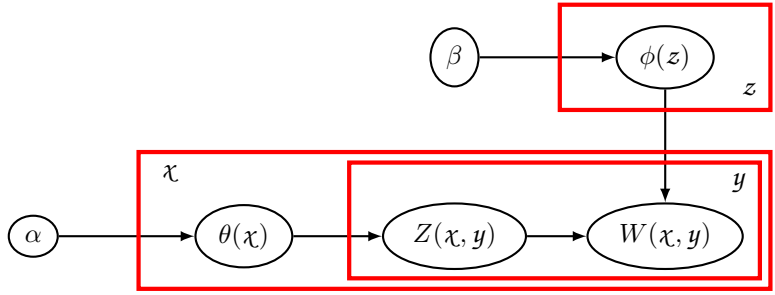
\begin{figure*}[t]
\begin{center}
\begin{tikzpicture}[thick,->, >=latex,scale=1.4]
\node[ellipse,rounded corners,draw] (alpha) at (1,1) {$\alpha$};
\node[ellipse,rounded corners,draw] (theta) at (3,1) {$\theta(\logvar{x})$};
\node[ellipse,rounded corners,draw] (Z) at (5,1) {$Z(\logvar{x},\logvar{y})$};
\node[ellipse,rounded corners,draw] (W) at (7,1) {$W(\logvar{x},\logvar{y})$};
\node[ellipse,rounded corners,draw] (beta) at (5,2.7) {$\beta$};
\node[ellipse,rounded corners,draw] (phi) at (7,2.7) {$\phi(\logvar{z})$};

\draw (alpha)--(theta);
\draw (theta)--(Z);
\draw (Z)--(W);
\draw (beta)--(phi);
\draw (phi)--(W);

\draw[red,ultra thick] (6,2.2) rectangle (8,3.2);
\draw[red,ultra thick] (2,0.5) rectangle (8,1.8);
\draw[red,ultra thick] (4,0.6) rectangle (7.9,1.7);

\node at (2.3,1.6) {$\logvar{x}$};
\node at (7.7,1.5) {$\logvar{y}$};
\node at (7.8,2.4) {$\logvar{z}$};
\end{tikzpicture}
\end{center}
\vspace*{-3ex}
\caption{Smoothed Latent Dirichlet Allocation.}
\label{figure:LDA}
\end{figure*}
 
One can find  extended versions of plate models in the literature, where
a node can have children in other plates (for instance the 
{\em smoothed Latent Dirichlet Allocation} (sLDA) model \cite{Blei2003}
depicted   in Figure~\ref{figure:LDA}). In such extended plates a 
template conditional probability distribution can refer to logvars from plates
that are not enclosing the parvariable; if definition axioms are then allowed
to specify  template distributions, one obtains as before 
$\mathsf{INF}[\mathsf{FFFO}]$, $\mathsf{QINF}[\mathsf{FFFO}]$, etc;  
that is, results obtained in previous sections apply.

Besides plates, several other languages can encode repetitive Bayesian networks.
Early proposals resorted
to object orientation \cite{Koller97,Laskey96}, to frames~\cite{Koller98}, and
to rule-based statements \cite{Bacchus93,Glesner95,Haddawy94UAI}, all
inspired by knowledge-based model construction \cite{Horsch90,Goldman90,Wellman92}. 
Some of these proposals coalesced into a family of models loosely grouped under the
name of {\em Probabilistic Relational Models} (PRMs)~\cite{Friedman99}.
We adopt PRMs as defined by Getoor et al. \cite{Getoor2007}; again, to simplify
matters, we focus on parvariables that correspond to relations.

Similarly to a plate model, a PRM contains typed parvariables and domains. 
A domain is now called a {\em class}; each class appears as a box containing parvariables.
For instance, Figure \ref{figure:PRM} depicts a PRM for the University World:
edges between parvariables indicate probabilistic dependence, and 
dashed edges between  classes indicate {\em associations} between elements of the classes. 
In Figure~\ref{figure:PRM} we 
have classes $\textbf{Course}$, $\textbf{Student}$, and $\textbf{Registration}$,
with associations between them. Consider association $\textbf{studentOf}$: the idea is 
that $\textbf{studentOf}(\logvar{x},\logvar{z})$ holds when $\logvar{x}$ is the student
in registration $\logvar{z}$.
Following terminology by Koller and Friedman \cite{Koller2009}, we say
that relations that encode classes and associations, such as
$\mathbf{Course}$ and $\mathbf{studentOf}$, are {\em guard parvariables}.

A {\em relational skeleton} for a PRM is an explicit specification of elements in each class, 
plus  the explicit specification of pairs of objects that are associated.
That is, the relational skeleton specifies the groundings of the guard parvariables.

\begin{figure}
\begin{center}
\begin{tikzpicture}[scale=0.85]
\node[blue,ultra thick,text depth=13mm,minimum width=32mm,rectangle,draw] (R) at (5.5,1) {\textsf{\textbf{Registration}} $\logvar{z}$};
\node[ellipse,rounded corners,draw,thick] (G) at ([yshift=4ex]R.south) {$\mathsf{Failed?}(\logvar{z})$};
\node[blue,ultra thick,text depth=13mm,minimum width=37mm,rectangle,draw] (C) at (1,1) {\textsf{\textbf{Course}} $\logvar{x}$};
\node[ellipse,rounded corners,draw,thick] (DI) at ([yshift=4ex]C.south) {$\mathsf{Difficult?}(\logvar{x})$};
\node[blue,ultra thick,text depth=13mm,minimum width=37mm,rectangle,draw] (S) at (10,1) 
	{\textsf{\textbf{Student}} $\logvar{y}$};
\node[ellipse,rounded corners,draw,thick] (DE) at ([yshift=4ex]S.south) {$\mathsf{Committed?}(\logvar{y})$};
\draw[thick,->, >=latex] (DI)--(G);
\draw[thick,->, >=latex] (DE)--(G);
\draw[blue,ultra thick,dashed] (R)  edge[out=50,in=130]  node[above] {\small $\textbf{studentOf}$}  (S);  
\draw[blue,ultra thick,dashed] (R)  edge[out=130,in=50]  node[above] {\small $\textbf{courseOf}$} (C);
\end{tikzpicture}
\end{center}
\vspace*{-2ex}
\caption{PRM for the University World. We show logvars explicitly, even though they
are not always depicted in PRMs. Associations appear as dashed edges \cite{Getoor2007,Sommestad2010}.}
\label{figure:PRM}
\end{figure}
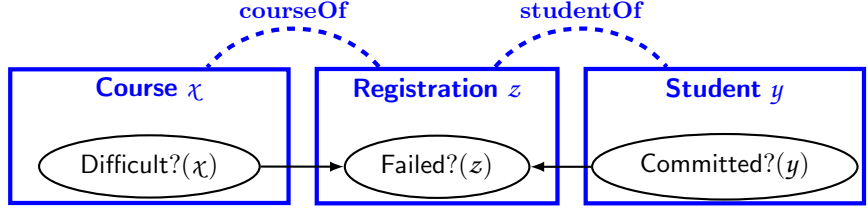

Each parvariable $X$ in a PRM is then associated with a
{\em template probability distribution} that specifies probabilities for the parvariable 
$X$ given a selected set of other parvariables. The latter are the {\em parents} of $X$, again
denoted by $\parents{X}$. In the University World of Figure \ref{figure:PRM}, we 
must associate with $\mathsf{Failed?}$ the template probabilities for
$\pr{\mathsf{Failed?}(\logvar{z})|\mathsf{Difficult?}(\logvar{x}), \mathsf{Committed?}(\logvar{y})}$. 
But differently from plate models, here the parents of a particular grounding are determined
by going through associations: for instance, to find the
parents of $\mathsf{Failed?}(r)$, we must find the course $c$ and the student
$s$ such that $\textbf{courseOf}(c,r)$ and $\textbf{studentOf}(s,r)$ hold, and then
we have parents $\mathsf{Difficult?}(c)$ and $\mathsf{Committed?}(s)$.

All of these types and associations can, of course, be encoded using first-order logic,
as long as all parvariables correspond to relations.
For instance, here is a definition axiom that captures the PRM for the University World:
\begin{eqnarray*}
\mathsf{Failed?}(\logvar{z}) & \definitionaxiom & 
\forall \logvar{x} : \forall \logvar{y} : 
\left( \begin{array}{c}
\mathbf{Course}(\logvar{x}) \wedge \mathbf{Student}(\logvar{y}) \wedge \\
\mathbf{courseOf}(\logvar{x},\logvar{z}) \wedge \mathbf{studentOf}(\logvar{y},\logvar{z}) 
\end{array}\right) \rightarrow \\
& & \hspace*{1cm} 
\left(\begin{array}{c}
\left(\neg \mathsf{Difficult?}(\logvar{x}) \wedge \neg \mathsf{Committed?}(\logvar{y}) 
\wedge \mathsf{A_1}(\logvar{x},\logvar{y}) \right) \vee \\
\left( \neg \mathsf{Difficult?}(\logvar{x}) \wedge  \mathsf{Committed?}(\logvar{y}) 
\wedge \mathsf{A_2}(\logvar{x},\logvar{y}) \right) \vee \\
\left(  \mathsf{Difficult?}(\logvar{x}) \wedge \neg \mathsf{Committed?}(\logvar{y}) 
 \wedge \mathsf{A_3}(\logvar{x},\logvar{y}) \right) \vee \\
\left(  \mathsf{Difficult?}(\logvar{x}) \wedge  \mathsf{Committed?}(\logvar{y}) 
 \wedge \mathsf{A_4}(\logvar{x},\logvar{y}) \right)
\end{array}\right),
\end{eqnarray*}
using the same   auxiliary parvariables employed in 
Expression~(\ref{equation:Failed}). The parvariable graph for the resulting
specification is depicted in Figure \ref{figure:FailedGraph}.

\begin{figure}
\begin{center}
\begin{tikzpicture}
\node[ellipse,draw,thick] (difficult) at (1,1) {$\mathsf{Difficult?}$};
\node[ellipse,draw,thick] (failed) at (5,1) {$\mathsf{Failed?}$};
\node[ellipse,draw,thick] (committed) at (9,1) {$\mathsf{Committed?}$};
\node[ellipse,draw,thick] (course) at (0.5,2) {$\mathbf{Course}$};
\node[ellipse,draw,thick] (student) at (3.5,2) {$\mathbf{Student}$};
\node[ellipse,draw,thick] (courseOf) at (7,2) {$\mathbf{courseOf}$};
\node[ellipse,draw,thick] (studentOf) at (10,2) {$\mathbf{studentOf}$};
\node[ellipse,draw,thick] (a1) at (2,0) {$\mathsf{A_1}$};
\node[ellipse,draw,thick] (a2) at (4,0) {$\mathsf{A_2}$};
\node[ellipse,draw,thick] (a3) at (6,0) {$\mathsf{A_3}$};
\node[ellipse,draw,thick] (a4) at (8,0) {$\mathsf{A_4}$};
\draw[thick,->, >=latex] (course)--(failed);
\draw[thick,->, >=latex] (student)--(failed);
\draw[thick,->, >=latex] (courseOf)--(failed);
\draw[thick,->, >=latex] (studentOf)--(failed);
\draw[thick,->, >=latex] (difficult)--(failed);
\draw[thick,->, >=latex] (committed)--(failed);
\draw[thick,->, >=latex] (a1)--(failed);
\draw[thick,->, >=latex] (a2)--(failed);
\draw[thick,->, >=latex] (a3)--(failed);
\draw[thick,->, >=latex] (a4)--(failed); 
\end{tikzpicture}
\vspace*{-2ex}
\end{center}
\caption{Parvariable graph for the University World PRM.}
\label{figure:FailedGraph}
\end{figure}
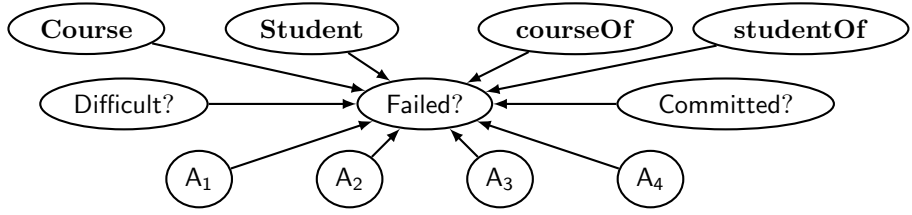

Thus we can take a PRM and translate it into a relational Bayesian network specification
$\mathbb{S}$.  As long as the parvariable graph is acyclic, results in the previous sections
apply. To see this, note that a skeleton is simply an assignment for all groundings of the
guard parvariables. Thus a skeleton can be encoded into a set of
assignments $\mathbf{S}$, and our inferences should focus on deciding
 $\pr{\mathbf{Q}|\mathbf{E},\mathbf{S}}>\gamma$
with respect to $\mathbb{S}$ and a domain that is the union of all classes of the PRM.

For instance, suppose we have a fixed PRM and we receive as input a skeleton and
a query  $(\mathbf{Q},\mathbf{E})$, and we wish
to decide whether  $\pr{\mathbf{Q}|\mathbf{E}}>\gamma$. If the
template probability distributions are specified with $\mathsf{FFFO}$, and the
parvariable graph is acyclic, then
this decision problem is a $\mathsf{PP}$-complete problem. We can replay our 
previous results on inferential and query complexity this way. 
The concept of domain complexity seems less meaningful when PRMs are considered:
the larger the domain, the more data on guard parvariables are needed, so we cannot
really fix the domain in isolation.

We conclude this section with a few observations.

\paragraph{Cyclic parvariable graphs}
Our results assume acyclicity of parvariable graphs,
 but this is not a condition that is typically imposed on PRMs. 
A cyclic parvariable graph may still produce an acyclic grounding,
depending on the given skeleton. For instance,
one might want to model blood-type inheritance, where a $\mathbf{Person}$
inherits a genetic predisposition from another $\mathbf{Person}$. This creates
a loop around the class $\mathbf{Person}$, even though we do not expect a cycle 
in any valid grounding of the PRM. The literature has proposed languages that allow 
cycles  \cite{Getoor2007,Heckerman2007}; one example is shown in 
Figure \ref{figure:BloodType}.
The challenge then is to guarantee that a given skeleton will lead to an acyclic grounded
Bayesian network; future work on cyclic parvariable graphs must deal with such a 
consistency problem. 

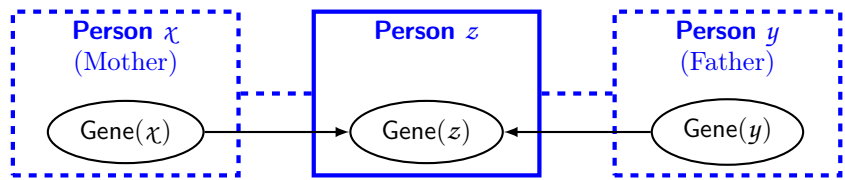
\begin{figure}[t]
\begin{center}
\begin{tikzpicture}
\node[blue,ultra thick,text depth=17mm,minimum width=30mm,rectangle,draw] (R) at (5,1) {\textsf{\textbf{Person}} $\logvar{z}$};
\node[ellipse,rounded corners,draw,thick] (G) at ([yshift=4ex]R.south) {$\mathsf{Gene}(\logvar{z})$};
\node[dashed,blue,ultra thick,text depth=17mm,minimum width=30mm,rectangle,draw] (C) at (1,1) {\textsf{\textbf{Person}} $\logvar{x}$};
\node[ellipse,rounded corners,draw,thick] (DI) at ([yshift=4ex]C.south) {$\mathsf{Gene}(\logvar{x})$};
\node[blue] at ([yshift=10ex]C.south) {(Mother)};
\node[dashed,blue,ultra thick,text depth=17mm,minimum width=30mm,rectangle,draw] (S) at (9,1) 
	{\textsf{\textbf{Person}} $\logvar{y}$};
\node[ellipse,rounded corners,draw,thick] (DE) at ([yshift=4ex]S.south) {$\mathsf{Gene}(\logvar{y})$};
\node[blue] at  ([yshift=10ex]S.south) {(Father)};
\draw[thick,->, >=latex] (DI) -- (G);
\draw[thick,->, >=latex] (DE) -- (G); 
\draw[blue,ultra thick,dashed] (R) -- (S);
\draw[blue,ultra thick,dashed] (R) -- (C); 
\end{tikzpicture}
\vspace*{-1ex}
\end{center}
\caption{A PRM for the blood-type model, adapted from a proposal by
Getoor et al. \cite{Getoor2007}. Dashed boxes stand for repeated classes;
Getoor et al.\ suggest that some associations may be constrained to be ``guaranteed
acyclic'' so that the whole model is consistent for any skeleton that satisfies the
constraints.}
\label{figure:BloodType}
\end{figure}

\paragraph{Other specification languages}
There are several other languages that specify PRMs and related formalisms;
such languages can be subjected to the same analysis we have explored in this paper.
A notable formalism is the Probabilistic Relational Language (PRL) \cite{Getoor2006},
where  logic program are used to specify PRMs; the specification 
is divided into logical background (that is, guard parvariables), 
probabilistic background, and probabilistic dependencies.
Two other examples of textual formalisms that can be used to encode PRMs 
are Logical Bayesian Networks (LBNs) \cite{Fierens2004,Fierens2005} and Bayesian Logic Programs 
(BLPs) \cite{Kersting2000,Raedt2004}.
 Both distinguish  between {\em logical}  predicates that constrain groundings
(that is, guard parvariables), and 
{\em probabilistic} or {\em Bayesian} predicates  that encode probabilistic assessments \cite{Ngo97}.

A more visual language, based on Entity-Relationship Diagrams, is DAPER  \cite{Heckerman2007}.
Figure \ref{figure:DAPERBloodType} shows a DAPER diagram for the University World and a DAPER 
diagram for the blood-type model. 
Another diagrammatic language is given by 
Multi-Entity Bayesian Networks (MEBNs), a graphical 
representation for arbitrary first-order sentences \cite{Laskey2008}. 
Several other graphical languages mix probabilities with description logics
 \cite{Carvalho2013,Costa2006,Ding2006,Koller97PClassic}, as we have 
mentioned in Section~\ref{section:DLLite}.

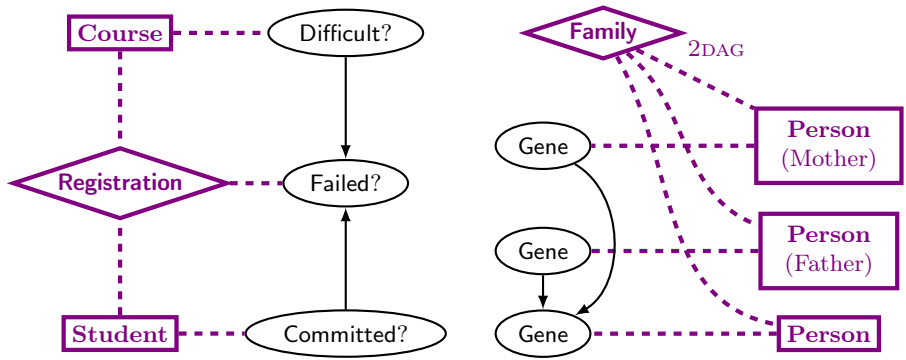
\begin{figure}
\begin{center}
\begin{tikzpicture} \small
\node[violet,ultra thick,draw] (C) at (1,5) {$\textbf{Course}$};
\node[violet,ultra thick,draw] (S) at (1,1) {$\textbf{Student}$};
\node[ellipse,rounded corners,draw,thick] (G) at (4,3) {$\mathsf{Failed?}$};
\node[ellipse,rounded corners,draw,thick] (D) at (4,5) {$\mathsf{Difficult?}$};
\node[ellipse,rounded corners,draw,thick] (I) at (4,1) {$\mathsf{Committed?}$};
\node[violet,diamond,aspect=3,ultra thick,draw] (R) at (1,3) 
	{\hspace*{-2ex}{\textsf{\textbf{Registration}}}\hspace*{-2ex}};
\draw[violet,dashed,ultra thick,draw] (C)--(D);
\draw[violet,dashed,ultra thick,draw] (S)--(I);
\draw[violet,dashed,ultra thick,draw] (C)--(R);
\draw[violet,dashed,ultra thick,draw] (S)--(R);
\draw[violet,dashed,ultra thick,draw] (R)--(G);
\draw[thick,->, >=latex] (D)--(G);
\draw[thick,->, >=latex] (I)--(G);
\end{tikzpicture}
\hspace*{4mm}
\begin{tikzpicture} \small
\node[violet,ultra thick,draw] (M) at (5,2) {\begin{tabular}{c} $\textbf{Person}$ \\  (Mother) \end{tabular}};
\node[violet,ultra thick,draw] (F) at (5,0.6) {\begin{tabular}{c} $\textbf{Person}$ \\  (Father) \end{tabular}};
\node[violet,ultra thick,draw] (P) at (5,-0.5) {$\textbf{Person}$};
\node[violet,diamond,aspect=3,ultra thick,draw] (R) at (2,3.5) {\hspace*{-2ex}{\textsf{\textbf{Family}}}\hspace*{-2ex}};
\node[violet] at (3.5,3.3) {\small \sc 2dag};

\node[ellipse,rounded corners,draw,thick] (GM) at (1.2,-0.5) {$\mathsf{Gene}$};
\node[ellipse,rounded corners,draw,thick] (GF) at (1.2,2) {$\mathsf{Gene}$}; 
\node[ellipse,rounded corners,draw,thick] (GP) at (1.2,0.6) {$\mathsf{Gene}$};

\draw[violet,dashed,ultra thick,draw] (R)--(M);
\draw[violet,dashed,ultra thick,draw] (R) edge[out=-40,in=160] (F);
\draw[violet,dashed,ultra thick,draw] (R) edge[out=-60,in=170] (P);

\draw[violet,dashed,ultra thick,draw] (M) -- (GF);  
\draw[violet,dashed,ultra thick,draw] (P)--(GM);
\draw[violet,dashed,ultra thick,draw] (F)--(GP);

\draw[thick,->, >=latex] (GF) edge[out=-30,in=30] (GM);
\draw[thick,->, >=latex] (GP)--(GM);
\end{tikzpicture}
\vspace*{-1ex}
\end{center}
\caption{Left: A DAPER diagram for the University World.
Right: A DAPER diagram for the blood-type model, as proposed by
Heckerman et al.\ \cite{Heckerman2007}; note the
constraint {\sc 2dag}, meaning that each child of the node has at most two parents 
and cannot be his or her own ancestor.}
\label{figure:DAPERBloodType}
\end{figure}
  
There are other formalisms in the literature that are somewhat removed from
our framework. For instance, Jaeger's {\em Relational Bayesian Networks} 
\cite{Jaeger97,Jaeger2001} offer
a solid representation language where the user can directly specify and manipulate 
probability values, for instance specifying that a probability value is the average
of other probability values. We have examined the complexity of Relational Bayesian
Networks elsewhere \cite{Maua2016PGM}; some results and proofs, but not all of them, 
are similar  to the ones presented here.
There are also languages that encode repetitive Bayesian networks using
functional programming \cite{Mansinghka2014,Milch2005,Tenorth2013,Pfeffer2001}
or logic programming
\cite{Costa2003,Fierens2015,Poole93AI,Poole97,Riguzzi2015,Sato95}, 
We have examined the complexity of the latter formalisms elsewhere
\cite{Cozman2016PGM,Cozman2016WPLP,Cozman2016ENIAC}; again, 
some results and proofs, but not all of them, are similar to the ones presented here.

\section{A detour into Valiant's counting hierarchy}
\label{section:Valiant}

We have so far focused on inferences that compare a conditional probability with a given
rational number. However one might argue that the real purpose of a probabilistic inference
is to compute a probability value. We can look at the complexity of calculating such numbers 
using Valiant's counting classes and their extensions. Indeed, most work on probabilistic databases
and lifted inference has used Valiant's classes.  In this section we justify our focus on decision
problems, and adapt our results to Valiant's approach.
 
Valiant defines, for complexity class $\mathsf{A}$, the class $\#\mathsf{A}$
to be $\cup_{\mathcal{L} \in \mathsf{A}} (\#\mathsf{P})^{\mathcal{L}}$, where
$(\#\mathsf{P})^{\mathcal{L}}$ is the class of functions counting the accepting paths
of nondeterministic polynomial time Turing machines with $\mathcal{L}$ as oracle
\cite{Valiant79TCS}.
Valiant declares function $f$ to be $\#\mathsf{P}$-hard when $\#\mathsf{P} \subseteq \mathsf{FP}^f$;
that is,  $f$ is $\#\mathsf{P}$-hard if any function $f'$ in $\#\mathsf{P}$ can be reduced to 
$f$ by the analogue of a polynomial time Turing reduction
(recall that $\mathsf{FP}$ is the
class of functions that can be computed in polynomial time by a
deterministic Turing machine).
Valiant's  is a very loose notion of hardness; as shown by Toda and Watanabe~\cite{Toda92},
any function in $\#\mathsf{PH}$ can be reduced to a function in $\#\mathsf{P}$
via a one-Turing reduction (where $\#\mathsf{PH}$ is a counting class with the whole
polynomial hierarchy as oracle). Thus a Turing
reduction is too weak to distinguish classes of functions within $\#\mathsf{PH}$.
For this reason, other
 reductions have been considered for counting problems \cite{Bauland2010,Durand2005}.

A somewhat stringent strategy is to say that $f$ is $\#\mathsf{P}$-hard if any 
function $f'$ in $\#\mathsf{P}$ can be produced from $f$ by a {\em parsimonious} 
reduction; that is, $f'(\ell)$ is computed by applying a polynomial time function 
$g$ to $x$ and then computing $f(g(\ell))$ \cite{Simon75PHD}. 
However, such a strategy is inadequate for our purposes:
counting classes such as $\#\mathsf{P}$ produce integers, and we
cannot produce integers by computing probabilities.

A sensible strategy is to adopt a reduction that allows for multiplication by a 
polynomial function. This has been done both in the context of probabilistic 
inference with ``reductions modulo normalization'' \cite{Kwisthout2011} and 
in the context of probabilistic databases \cite{Suciu2011}.  We adopt the 
reductions proposed by Bulatov et al.\ in their study of weighted constraint 
satisfaction problems \cite{Bulatov2012}. They define {\em weighted reductions}
as follows (Bulatov et al.\ consider functions into the algebraic numbers, but
for our purposes we can restrict the weighted reductions to rational numbers):

\begin{Definition}
\label{definition:WeightedReduction}
Consider functions $f_1$ and $f_2$ from an input language $\mathcal{L}$ to rational 
numbers $\mathbb{Q}$. A {\em weighted reduction} from $f_1$ to $f_2$ a pair of polynomial time 
functions $g_1:\mathcal{L}\rightarrow\mathbb{Q}$ and $g_2:\mathcal{L}\rightarrow\mathcal{L}$
such that  $f_1(\ell) = g_1(\ell) f_2(g_2(\ell))$ for all $\ell$.
\end{Definition}

We say a function $f$ is $\#\mathsf{P}$-hard with respect to weighted reductions
if any function in $\#\mathsf{P}$ can be reduced to $f$ via a weighted reduction.

Having decided how to define hardness, we must look at membership. 
As counting problems generate integers, we cannot really say that probabilistic
inference problems belong to any class in Valiant's counting hierarchy.
In fact, in his seminal work on the complexity of Bayesian
networks, Roth notes that ``strictly speaking the problem of computing the degree
of belief is not in $\#\mathsf{P}$, but easily seem equivalent to a problem in this 
class'' \cite{Roth96}. The challenge is to formalize such an equivalence. 

Grove, Halpern, and Koller quantify the complexity of probabilistic inference 
by allowing polynomial computations to occur after counting \cite[Definition 4.12]{Grove96}.
Their strategy is to say that $f$ is $\#\mathsf{P}$-{\em easy} if
there exists $f' \in \#\mathsf{P}$ and $f'' \in \mathsf{FP}$ such that for all
$\ell$ we have $f(\ell) = f''(f'(\ell))$. 
Similarly, Campos, Stamoulis and Weyland take $f$ to be $\#\mathsf{P}[1]$-{\em equivalent}
if $f$ is $\#\mathsf{P}$-hard (in Valiant's sense) and belongs to $\mathsf{FP}^{\#\mathsf{P}[1]}$. 
Here the superscript $\#\mathsf{P}[1]$ means that the oracle $\#\mathsf{P}$ can be called only
once.
It is certainly a good idea to resort to a new term (``equivalence'')  in this context; however one
must feel that membership to $\mathsf{FP}^{\#\mathsf{P}[1]}$ is too weak a requirement given 
Toda and Watanabe's theorem \cite{Toda92}:
any function in  $\#\mathsf{PH}$ can be produced within $\mathsf{FP}^{\#\mathsf{P}[1]}$.

We adopt a stronger notion of equivalence:
a function $f$ is {\em $\#\mathsf{P}$-equivalent} if it is $\#\mathsf{P}$-hard with
respect to weighted reductions and 
$g \cdot f$ is in $\#\mathsf{P}$ for some polynomial-time  function~$g$ from the
input language to rational numbers.

Also, we need to define a class of functions that corresponds to the complexity 
class $\mathsf{PEXP}$. We might extend Valiant's definitions and take
$\#\mathsf{EXP}$ to be $\cup_{\mathcal{L} \in \mathsf{EXP}} \mathsf{FP}^{\mathcal{L}}$. 
However functions in such a class produce numbers
whose size is at most polynomial on the size of the input, as the number of
accepting paths of a nondeterministic Turing machine on input $\ell$ is bounded by
$2^{p(|\ell|)}$ where $p$ is polynomial and $|\ell|$ is the length of $\ell$.
This is not appropriate for our purposes, as even
simple specifications may lead to exponentially long output
(for instance, take $\pr{X(\logvar{x})=1}=1/2$ 
and compute $\pr{\exists \logvar{x}: X(\logvar{x})}$: we must be able
to write the answer $1-1/2^N$ using $N$ bits, and this number of bits
is exponential on the input if the domain size is given in binary notation). 
For this reason, we take $\#\mathsf{EXP}$ to be the class of functions that can be
computed by counting machines of exponential time complexity \cite{Papadimitriou86}.
We say that a function $f$ is $\#\mathsf{EXP}$-equivalent
if $f$ is $\#\mathsf{EXP}$-hard with respect to reductions that follow
exactly Definition \ref{definition:WeightedReduction}, except for the fact that
$g_1$ may be an exponential time function, 
and $g f$ is in $\#\mathsf{EXP}$ for some exponential time function $g$ from the
input language to the rational numbers.

Now consider polynomial bounds on space. 
We will use the class $\natural\mathsf{PSPACE}$ class defined by Ladner \cite{Ladner89},
consisting of those functions that can be computed by counting Turing machines
with a polynomial space bound {\em and} a polynomial bound on the number of
nondeterministic moves. This class is actually equal to 
$\mathsf{FPSPACE[poly]}$, the class of functions computable in polynomial
space whose outputs are strings encoding numbers in binary notation,
and bounded in length by a polynomial \cite[Theorem 2]{Ladner89}. 
We say that a function is $\natural\mathsf{PSPACE}$-equivalent if
$f$ is $\natural\mathsf{PSPACE}$-hard with respect to weighted reductions
(as in Definition \ref{definition:WeightedReduction}),  and
$g \cdot f$ is in $\#\mathsf{PSPACE}$ for some polynomial space function~$g$
from the input language to the rational numbers.
Of course we might have used ``$\mathsf{FPSPACE[poly]}$-equivalent'' instead, but
we have decided to follow Ladner's original notation.

There is one more word of caution when it comes to adopting Valiant's
counting Turing machines. 
It is actually likely that functions that are proportional to conditional
probabilities $\pr{\mathbf{Q}|\mathbf{E}}$ cannot be produced by counting
Turing machines, as classes in Valiant's counting hierarchy
are not likely to be closed under division even by polynomial-time computable functions
\cite{Ogiwara93}.  Thus we must focus on inferences of the form $\pr{\mathbf{Q}}$;
indeed this is the sort of computation that is analyzed in probabilistic databases
\cite{Suciu2011}. 

The drawback of Valiant's hierarchy is, therefore, that a significant amount of
adaptation is needed before that hierarchy can be applied to probabilistic inference.
But after all this preliminary work, we can   convert our previous results
accordingly. For instance, we have:
\begin{Theorem} \label{theorem:SharpEXP}
Consider the class of functions that gets as input a relational Bayesian network
specification based on $\mathsf{FFFO}$, a domain size $N$ (in binary or unary
notation), and a set of assignments $\mathbf{Q}$, and returns 
$\pr{\mathbf{Q}}$. This class of functions is $\#\mathsf{EXP}$-equivalent.
\end{Theorem}

\begin{Theorem} \label{theorem:SharpPSPACE}
Consider the class of functions that gets as input a relational Bayesian network
specification based on $\mathsf{FFFO}$  with relations of bounded arity, 
a domain size $N$ in unary
notation, and a set of assignments $\mathbf{Q}$, and returns 
$\pr{\mathbf{Q}}$. This class of functions is $\natural\mathsf{PSPACE}$-equivalent.
\end{Theorem}

\begin{Theorem} \label{theorem:SharpFFFOk}
Consider the class of functions that gets as input a relational Bayesian network
specification based on $\mathsf{FFFO}^k$  for $k \geq 2$, a domain size $N$ in unary
notation, and a set of assignments $\mathbf{Q}$, and returns 
$\pr{\mathbf{Q}}$. This class of functions is $\#\mathsf{P}$-equivalent.
\end{Theorem}
 
\begin{Theorem} \label{theorem:SharpPlates}
Consider the class of functions that get as input a plate model
 based on $\mathsf{FFFO}$, a domain size $N$ (either in binary or unary
notation), and a set of assignments $\mathbf{Q}$, and returns 
$\pr{\mathbf{Q}}$. This class of functions is $\#\mathsf{P}$-equivalent.
\end{Theorem}

\section{Conclusion}
\label{section:Conclusion}

We have presented a framework for specification and analysis of 
Bayesian networks, particularly networks containing repetitive patterns
that can be captured using function-free first-order logic. Our specification
framework is based on previous work on probabilistic programming and
structural equation models; our analysis is based on notions of complexity
(inferential, combined, query, data, domain) that are similar to concepts
used in lifted inference and in probabilistic databases.

Our emphasis was on knowledge representation; in particular we wanted to
understand how features of the specification language affect the complexity
of inferences. Thus we devoted some effort to connect probabilistic modeling
with knowledge representation formalisms, particularly description logics.
We hope that we have produced here a sensible framework that unifies
several disparate efforts, a contribution that may lead to further insight 
into probabilistic modeling. 

Another contribution of this work is a collection of results on complexity of
inferences, as summarized by Table \ref{table:Results} and related commentary.
We have also introduced relational Bayesian network specifications based on the
$\mathsf{DLLite^{nf}}$ logic, a language that can be used to specify
 probabilistic ontologies and
a sizeable class of  probabilistic entity-relationship diagrams. 
In proving results about  $\mathsf{DLLite^{nf}}$, we have identified a class
of model counting problems with tractability guarantees. 
Finally, we have shown how to transfer our results into plate models and
PRMs, and in doing so we have presented a much needed analysis of these
popular specification formalisms.

There are several avenues open for future work. Ultimately, we must reach
an understanding of the relationship between expressivity and complexity 
of Bayesian networks specifications that is as rich as the understanding 
we now have about the expressivity and complexity of logical languages. 
We must consider Bayesian
networks specified by operators from various description and modal logics, 
or look at languages that allow variables to have more than two values.
In a different direction, we must look at parameterized counting 
classes~\cite{Flum2004}, so as to refine the analysis even further.
There are also several problems that go beyond the inferences discussed
in this paper: for example, the computation of Maximum a Posteriori (MAP)
configurations, and the verification that a possibly cyclic PRM is consistent
for every possible skeleton. There are also models that encode structural
uncertainty, say about the presence of edges~\cite{Koller2009}, and novel
techniques must be developed to investigate the complexity of such 
models. 

\section*{Acknowledgements}
 
The first author is partially supported by CNPq (grant 308433/2014-9)
and the second author is supported by FAPESP (grant 2013/23197-4).
We thank Cassio Polpo de Campos for valuable insights concerning complexity 
proofs, and Johan Kwisthout for discussion concerning the MAJSAT problem. 

\appendix

\section{Proofs}
\label{appendix:Proofs}

\setcounter{Proposition}{0}  

\begin{Proposition} 
Both $\#3\mathsf{SAT}(>)$ and $\#\mbox{(1-in-3)}\mathsf{SAT}(>)$ 
are $\mathsf{PP}$-complete with respect to many-one reductions.
\end{Proposition}
\begin{proof}
Consider first $\#3\mathsf{SAT}(>)$. It belongs to $\mathsf{PP}$ because 
deciding $\#\phi>k$, for propositional sentence $\phi$, is in $\mathsf{PP}$ 
\cite[Theorem 4.4]{Simon75PHD}. And it is $\mathsf{PP}$-hard because 
it is already $\mathsf{PP}$-complete when the input is $k=2^{n/2}-1$ 
\cite[Proposition~1]{Bailey2007}.

Now consider $\#\mbox{(1-in-3)}\mathsf{SAT}(>)$.
Suppose the input is a propositional sentence $\phi$ in 3CNF with
propositions $A_1,\dots,A_n$ and $m$ clauses. Turn $\phi$ into another sentence 
$\varphi$ in 3CNF by turning each clause $L_1 \vee L_2 \vee L_3$ in $\phi$ 
into a set of four clauses:
\[
\neg L_1 \vee B_1 \vee B_2, \quad L_2 \vee B_2 \vee B_3, \quad
\neg L_3 \vee B_3 \vee B_4, \quad B_1 \vee B_3 \vee B_5,
\]
where the $B_j$ are fresh propositions not in $\phi$. 
We claim that $\#\varphi = \#\mbox{(1-in-3)}\phi$;
that is, 
$\#\mbox{(1-in-3)}\phi > k$ is equivalent to $\#\phi > k$,
proving the desired hardness.

To prove this claim in the previous sentence, reason as follows.
Define  $\theta(L_1,L_2,L_3) = 
      (L_1 \wedge \neg L_2 \wedge \neg L_3) \vee (\neg L_1 \wedge L_2 \wedge \neg L_3) 
           \vee (\neg L_1 \wedge \neg L_2 \wedge L_3)$;
that is, $\theta(L_1,L_2,L_3)$ holds if exactly one of its input literals is $\mathsf{true}$.
And for a clause $\rho = (L_1 \vee L_2 \vee L_3)$,  define
\[
\nu(\rho) = \theta(\neg L_1, B_1, B_2) \wedge \theta(L_2, B_2, B_3) \wedge
                   \theta(\neg L_3, B_3, B_4) \wedge \theta(B_1, B_3, B_5),
\]
where each $B_j$ is a fresh proposition not in $\phi$. 
Note that for each assignment to $(L_1,L_2,L_3)$ that satisfies $\rho$ there is
only one assignment to $(B_1,B_2,B_3,B_4,B_5)$ that satisfies $\nu(\rho)$.
To prove this, Table \ref{table:Satisfying} presents the set of {\em all} assignments 
that satisfy $\nu(C)$. Consequently, $\#\rho = \#\nu(\rho)$. 
By repeating this argument for each clause in $\phi$, we obtain our claim.
\end{proof}

\begin{table}
\begin{center}
  \begin{tabular}{|c|c|c|c|c|c|c|c|} \hline
    $L_1$ & $L_2$ & $L_3$ & $B_1$ & $B_2$ & $B_3$ & $B_4$ & $B_5$ \\ \hline\hline
    0 & 0 & 1 & 0 & 0 & 1 & 0 & 0 \\ \hline
    0 & 1 & 0 & 0 & 0 & 0 & 0 & 1 \\ \hline
    0 & 1 & 1 & 0 & 0 & 0 & 1 & 1 \\ \hline
    1 & 0 & 0 & 0 & 1 & 0 & 0 & 1 \\ \hline
    1 & 0 & 1 & 0 & 1 & 0 & 1 & 1 \\ \hline
    1 & 1 & 0 & 1 & 0 & 0 & 0 & 0 \\ \hline 
    1 & 1 & 1 & 1 & 0 & 0 & 1 & 0 \\ \hline
  \end{tabular}
\end{center}
\caption{Assignments that satisfy $\nu(C)$.}
\label{table:Satisfying}
\end{table}

\setcounter{Theorem}{0}

\begin{Theorem}
$\mathsf{INF}[\mathsf{Prop}(\wedge)]$ is in  to $\mathsf{P}$ when the query 
$(\mathbf{Q},\mathbf{E})$ contains only positive assignments,  
and $\mathsf{INF}[\mathsf{Prop}(\vee)]$ is in to $\mathsf{P}$ when the query
contains only negative assignments.
\end{Theorem}
\begin{proof}  
Consider first $\mathsf{INF}[\mathsf{Prop}(\wedge)]$.
To run inference with positive assignments $(\mathbf{Q},\mathbf{E})$, just 
run d-separation to collect the set of root variables that must be true given the 
assignments
(note that as soon as  a node is set to {\sf true}, its parents must be {\sf true},
and so on recursively). Then the probability of the conjunction of
assignments in $\mathbf{Q}$ and in $\mathbf{E}$ is
just the product of probabilities for these latter atomic propositions to
be true, and these probabilities are given in the network specification. 
Thus we obtain $\pr{\mathbf{Q},\mathbf{E}}$. Now
repeat the same polynomial computation only using assignments in $\mathbf{E}$,
to obtain $\pr{\mathbf{E}}$, and determine whether 
$\pr{\mathbf{Q},\mathbf{E}}/\pr{\mathbf{E}}>\gamma$ or not.

Now consider  $\mathsf{INF}[\mathsf{Prop}(\vee)]$.
For any input network specification, we can easily build a network specification
in $\mathsf{INF}[\mathsf{Prop}(\wedge)]$ by turning every variable $X$ into
a new variable $X'$ such that $X = \neg X'$. Then the root node associated
with assessment $\pr{X=1}=\alpha$ is turned into a root node associated
with $\pr{X'=1}=1-\alpha$, and a definition axiom $X \definitionaxiom \vee_i Y_i$ 
is turned into a definition axiom $X' \definitionaxiom \wedge_i Y'_i$. Any negative
evidence is then turned into positive evidence, and the reasoning in the previous
paragraph applies.
\end{proof}

\begin{Theorem}  
$\mathsf{INF}[\mathsf{Prop}(\wedge)]$ and $\mathsf{INF}[\mathsf{Prop}(\vee)]$
are $\mathsf{PP}$-complete with respect to many-one reductions.
\end{Theorem}
\begin{proof}
Membership follows from the fact that $\mathsf{INF}[\mathsf{Prop}(\wedge,\neg)] \in \mathsf{PP}$.
We therefore focus on hardness.

Consider first $\mathsf{INF}[\mathsf{Prop}(\wedge)]$.
We present a parsimonious reduction from $\#\mbox{(1-in-3)}\mathsf{SAT}(>)$
to $\mathsf{INF}[\mathsf{Prop}(\wedge)]$, following a strategy by Mau\'a et al.\ \cite{Maua2015IJCAI}.

Take a sentence $\phi$ in 3CNF with propositions $A_1,\dots,A_n$ and $m$ clauses. 
If there is a clause with a repeated literal (for instance, $(A_1 \vee A_1 \vee A_2)$
or $(\neg A_1 \vee A_2 \vee \neg A_1)$), then there is no assignment respecting
the 1-in-3 rule, so the count can be immediately assigned zero. So we
assume that no clause contains a repeated literal in the remainder of this proof.

For each literal in $\phi$, introduce a random variable $X_{ij}$, where $i$ refers to the
$i$th clause, and $j$ refers to the $j$th literal (note: $j\in\{1,2,3\}$). The set
of all such random variables is $\mathbf{L}$. 

For instance, suppose we have the sentence $(A_1 \vee A_2 \vee A_3) \wedge (A_4 \vee \neg A_1 \vee A_3)$.
We then make the correspondences:
$X_{11} \leadsto A_1$,
$X_{12} \leadsto A_2$,
$X_{13} \leadsto A_3$,
$X_{21} \leadsto A_4$,
$X_{22} \leadsto \neg A_1$,
$X_{23} \leadsto A_3$.

Note that $\{X_{ij}=1\}$ indicates an assignment of $\mathsf{true}$ to the corresponding
literal. Say that a configuration of $\mathbf{L}$ is {\em gratifying} if
$X_{i1}+X_{i2}+X_{i3}\geq 1$ for every clause (without necessarily respecting the 1-in-3 rule). 
Say that a configuration is {\em respectful} if is respects the 1-in-3 rule; that is,
if $X_{i1}+X_{i2}+X_{i3} \leq 1$ for every clause.
And say that a configuration is {\em sensible} if 
two variables that correspond to the same literal have the same value, 
and 
two variables that correspond to a literal and its negation have distinct values 
(in the example in the last paragraph, both  $\{X_{11}=1,X_{22}=1\}$  
and $\{X_{13}=1,X_{23}=0\}$ fail to produce a sensible configuration).
 
For each random variable $X_{ij}$,
introduce the assessment $\pr{X_{ij}=1} = 1-\varepsilon$,
where $\varepsilon$ is a rational number determined later. Our strategy
is to introduce definition axioms so that only the 
gratifying-respectful-sensible
configurations of $\mathbf{L}$  get high
probability, while the remaining configurations have low probability. 
The main challenge is to do so without negation.

Let $\mathbf{Q}$ be an initially empty set of assignments.
We first eliminate the configurations that do not respect the 1-in-3
rule. To do so, for $i=1,\dots,m$ include definition axioms
\begin{equation}
\label{equation:Respectful}
  Y_{i12}  \definitionaxiom X_{i1} \wedge X_{i2}, \qquad
  Y_{i13}  \definitionaxiom X_{i1} \wedge X_{i3}, \qquad
  Y_{i32}  \definitionaxiom X_{i2} \wedge X_{i3},
\end{equation}
and add $\{Y_{i12} = 0, Y_{i13} = 0, Y_{i23} = 0 \}$ to $\mathbf{Q}$. This
guarantees that configurations of $\mathbf{L}$ that fail to be respectful
are incompatible with $\mathbf{Q}$.
 
We now eliminate gratifying-respectful configurations that are not sensible. 
We focus on gratifying and respectful configurations because, as we
show later, ungratifying configurations compatible with $\mathbf{Q}$ are 
assigned low probability. 
\begin{itemize}
\item Suppose first that we have clause where the same literal appears twice.
For instance, suppose we have
$(A_i \vee \neg A_i \vee L)$, where $L$ is a literal. Assume the 
literals of this clause correspond to variables $X_{i1}$, $X_{i2}$, and $X_{i3})$.
Then impose $\{X_{i3}=0\}$. 
All other cases where a clause contains a literal and its negation must be treated similarly.

\item Now suppose we have two clauses 
$(A \vee L_{i2} \vee L_{i3})$ and $(\neg A \vee L_{j2} \vee L_{j3})$,
where $A$ is a proposition and the $L_{uv}$ are literals (possibly referring more 
than once to the same propositions). Suppose the six literals in these two clauses 
correspond to variables $(X_{i1},X_{i2},X_{i3})$ and $(X_{j1},X_{j2},X_{j3})$, in this order.
We must have $X_{i1}=1-X_{j1}$. To encode this relationship, we take two steps.
First, introduce the definition axiom
\[
  Y_{i1j1} \definitionaxiom X_{i1} \wedge X_{j1},
\]
and add $\{ Y_{i1j1} = 0 \}$ to $\mathbf{Q}$: at most one of $X_{i1}$ and $X_{j1}$ is
equal to $1$,  but there may still be gratifying-respectful configurations where $X_{i1}=X_{j1}=0$. 
Thus the second step is enforce the sentence $\theta = \neg (L_{i2} \vee L_{i3}) \vee \neg (L_{j2} \vee L_{j3})$,
as this forbids $X_{i1}=X_{j1}=0$. Note that $\theta$ is equivalent to 
$\neg(L_{i2} \wedge L_{j2}) \wedge \neg(L_{i2} \wedge L_{j3}) \wedge \neg(L_{i3} \wedge L_{j2}) \wedge \neg(L_{i3} \wedge L_{j3})$,
so introduce the definition axiom
\[
  Y_{iujv} \definitionaxiom X_{iu} \wedge X_{jv} 
\]
and add $\{Y_{iujv}=0\}$ to $\mathbf{Q}$, for each $u \in \{2,3\}$ and $v \in \{2,3\}$.
Proceed similarly if the literals of interest appear in other positions in the clauses.

\item We must now deal with cases where the same literal appears in different positions;
recall that no clause contains a repeated literal. So we focus on two clauses that share
a literal. Say we have $(A \vee L_{i2} \vee L_{i3})$ and $(A \vee L_{j2} \vee L_{j3})$
where the symbols are as in the previous bullet, and where the literals are again paired
with variables $(X_{i1},X_{i2},X_{i3})$ and $(X_{j1},X_{j2},X_{j3})$. 
If $X_{i1}=1$, then we must have $X_{j1}=1$, and to guarantee this in a gratifying-respectful
configuration we   introduce 
\[
 Y_{i1j2} \definitionaxiom X_{i1} \wedge X_{j2}, \qquad Y_{i1j3} \definitionaxiom X_{i1} \wedge X_{j3},
\]
and add $\{Y_{i1j2}=0, Y_{i1j3}=0\}$ to $\mathbf{Q}$. 
Similarly, if $X_{j1}=1$, we must have $X_{i1}=1$, so introduce
\[
 Y_{i2j1} \definitionaxiom X_{i2} \wedge X_{j1}, \qquad Y_{i3j1} \definitionaxiom X_{i3} \wedge X_{j1},
\]
and add $\{Y_{i2j1}=0, Y_{i3j1}=0\}$ to $\mathbf{Q}$. 
Again, proceed similarly if the literals of interest appear in other positions in the clauses.
\end{itemize}

Consider a configuration $x_{11},\dots,x_{m3}$ of $\mathbf{L}$. If this is a
gratifying-respectful-sensible configuration, we have that
\[
  \pr{X_{11}=x_{11},\dots,X_{m3}=x_{m3}} = (1-\varepsilon)^m\varepsilon^{2m} = \alpha \, .
\]
If the configuration is respectful but \emph{not} gratifying, then
\[
  \pr{X_{11}=x_{11},\dots,X_{m3}=x_{m3}} \leq (1-\varepsilon)^{m-1}\varepsilon^{2m+1} = \beta \, .
\]
The number of  respectful configurations is at most
$4^m$, since for each $i$ there are $4$ ways to assign values to
$(X_{i1}, X_{i2}, X_{i3})$ such that $X_{i1}+X_{i2}+X_{i3} \leq 1$. 

The whole reasoning is illustrated in the decision tree in Figure \ref{fig:proofInfAnd}.

If the number of solutions to the original
problem is strictly greater than $k$ then
$\pr{\mathbf{Q}} \geq (k+1)\alpha$. 
And if the number of solutions is
smaller or equal than $k$ then
$\pr{\mathbf{Q}} \leq k\alpha + 4^m \beta$. 
Now we must choose $\varepsilon$ so that $(k+1) \alpha > k\alpha+4^m\beta$,
so that we can differentiate between the two cases. 
We do so by choosing $\varepsilon < 1/(1+4^m)$.
Then $(\varphi,k)$ is in the language $\#\mbox{(1-in-3)}\mathsf{SAT}(>)$
iff $\pr{\mathbf{Q}} > k \alpha$.

\begin{figure}
\begin{center}
  \begin{tikzpicture}[thick]
    \node (A) at (0,0) {Is $\mathbf{L}$ respectful?};
    \node[rectangle,draw] (B) at (-3,-1) {$\pr{\mathbf{L},\mathbf{Q}}=0$};
    \node (C) at (3,-1) {Is $\mathbf{L}$ gratifying?};
    \node[rectangle,draw] (D) at (0,-2) {$\pr{\mathbf{L},\mathbf{Q}} \leq \beta$};
    \node (E) at (5,-2) {Is $\mathbf{L}$ sensible?};
    \node[rectangle,draw] (F) at (3,-3) {$\pr{\mathbf{L},\mathbf{Q}} = 0$};
    \node[rectangle,draw] (G) at (7,-3) {$\pr{\mathbf{L},\mathbf{Q}}=\alpha$};
    \foreach \x/\y in {A/B,A/C,C/D,C/E,E/F,E/G} { \draw (\x) -- (\y); }
  \end{tikzpicture}
\end{center}
  \caption{Decision tree of the probability assigned to configurations of the network constructed in the Proof.}
  \label{fig:proofInfAnd}
\end{figure}
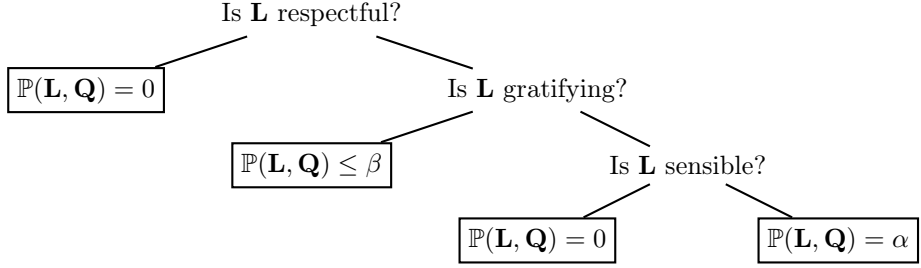

The whole construction is polynomial: the number of definition axioms is at most quadratic
in the number of literals of $\varphi$, and $\varepsilon$ can be encoded with $\mathcal{O}(m+n)$ bits.

Because the construction just described is somewhat complicated, 
we present an example. Consider again the sentence
$(A_1 \vee A_2 \vee A_3) \wedge (A_4 \vee \neg A_1 \vee A_3)$
and the related variables $X_{ij}$. 
We introduce definitions enforcing the 1-in-3 rule:
\begin{align*}
  Y_{112} &\definitionaxiom X_{11} \wedge X_{12} & Y_{113} &\definitionaxiom X_{11} \wedge X_{13} &  
          Y_{123} &\definitionaxiom X_{12} \wedge X_{13} \, ,\\
  Y_{212} &\definitionaxiom X_{21} \wedge X_{22} & Y_{213} &\definitionaxiom X_{21} \wedge X_{23} &  
         Y_{223} &\definitionaxiom X_{22} \wedge X_{23} \, .\,
\end{align*}
and appropriate assignments in $\mathbf{Q}$.
We then guarantee that at most one of
$A_1$ and $\neg A_1$ is true, by introducing
$Y_{1122} \definitionaxiom X_{11} \wedge X_{22}$, and by adding $\{ Y_{1122} = 0 \}$ to $\mathbf{Q}$.
Note that are configurations that are not sensible but that satisfy the previous constraints:
for instance, $\{ L_{13}=L_{23}=1, L_{11}=L_{12}=L_{21}=L_{22}=0 \}$ is not sensible
and has probability $\alpha=(1-\varepsilon)^2\varepsilon^{4}$.
To remove those configurations that
are not sensible but that have ``high'' probability, we introduce:
\begin{align*}
  Y_{1221} &\definitionaxiom X_{12} \wedge X_{21} \,,& Y_{1223} &\definitionaxiom X_{12} \wedge X_{23} \, ,\\
  Y_{1321} &\definitionaxiom X_{13} \wedge X_{21} \,,& Y_{1323} &\definitionaxiom X_{13} \wedge X_{23} \, , \\
  Y_{1321} &\definitionaxiom X_{13} \wedge X_{21} \,,& Y_{1322} &\definitionaxiom X_{13} \wedge X_{22} \, ,\\
  Y_{1123} &\definitionaxiom X_{11} \wedge X_{23} \,,& Y_{1223} &\definitionaxiom X_{12} \wedge X_{23} \, ,  
\end{align*}
and we add $\{ E_{1221}=0, E_{1223}=0, E_{1321}=0, E_{1323}=0,
 E_{1321}=0, E_{1322}=0, E_{1123}=0, E_{1223}=0\}$ to $\mathbf{Q}$. 
There are
$2^6=64$ configurations of $X_{11},\dots,X_{23}$, and $15$ of them have
$X_{i1}=X_{i2}=X_{i3}=0$ for $i=1$ or $i=2$ (or both). Among these
ungratifying configurations, 8 do not respect the 1-in-3 rule; the
remaining 7 that respect the 1-in-3 rule are assigned at most
probability $\beta$. Among the $49$ gratifying configurations (i.e.,
those that assign $X_{ij}=1$ for some $j$ for $i=1,2$), $40$ do not
respect the 1-in-3 rule. Of the remaining $9$ configurations, $7$ are 
not sensible. The last $2$ configurations are assigned probability $\alpha$. 
We thus have that
\[
  \pr{\mathbf{Q}} = \sum_{x_{11},\dots,x_{23}} \pr{X_{11}=x_{11},\dots,X_{23}
                           =x_{23},\mathbf{Q}} \leq 2\alpha+7\beta, 
\]
which implies that $(\varphi,3)$ is not in $\#\mbox{(1-in-3)}\mathsf{SAT}(>)$; 
indeed, there are $2 < 3$ assignments to $A_1,A_2,A_3,A_4$ that satisfy $\varphi$ 
and respect the 1-in-3 rule.

This concludes our discussion of $\mathsf{INF}[\mathsf{Prop}(\wedge)]$, so we
move to $\mathsf{INF}[\mathsf{Prop}(\vee)]$. To prove its $\mathsf{PP}$-completeness,
we must do almost exactly the same construction described before, with a few changes
that we enumerate.

First, we associate each literal with a random variable $X_{ij}$ as before, but
now $X_{ij}$ stands for a {\em negated} literal. That is, if the literal corresponding
to $X_{ij}$ is $A$ and $A$ is $\mathsf{true}$, then $\{X_{ij}=0\}$. Thus we must
associate  each $X_{ij}$ with the assessment $\pr{X_{ij}=1}=\varepsilon$.
Definitions must change accordingly: a configuration is now gratifying
if $X_{i1}+X_{i2}+X_{i3}<3$. 

Second, the previous construction used a number of definition axioms of the form 
\[
Y \definitionaxiom X \wedge X',
\]
with associated assignment $\{Y=0\}$. We must replace each such pair by a definition axiom
\[
Y \definitionaxiom X \vee X'
\]
and an assignment $\{Y=1\}$; recall that $X$ is just the negation of the variable
used in the previous construction, so the overall effect of the constraints is the same.

All other arguments carry, so we obtain the desired hardness.
\end{proof}

It is instructive to look at a proof of Theorem \ref{theorem:and-or} that uses
Turing reductions, as it is much shorter than the previous proof:

\begin{proof}
To prove hardness of $\mathsf{INF}[\mathsf{Prop}(\vee)]$, 
we use the fact that the function $\#\mathsf{MON2SAT}$ is $\#\mathsf{P}$-complete
with respect to Turing reductions \cite[Theorem~1]{Valiant79}. Recall that  
  $\#\mathsf{MON2SAT}$ is the function that counts the number of satisfying 
assignments of a monotone sentence in 2CNF.

So, we can take any MAJSAT problem where the input is sentence $\phi$
and produce (with polynomial effort) another sentence $\phi'$ in 2CNF such that 
$\#\phi$ is obtained from $\#\phi'$ (again with polynomial effort).  
And we can compute $\#\phi'$ using $\mathsf{INF}[\mathsf{Prop}(\vee)]$, as follows. 
Write $\phi'$ as $\bigwedge_{i=1}^m (A_{i_1} \vee A_{i_2})$, where each $A_{i_j}$ 
is a proposition in $A_1,\dots,A_n$. Introduce fresh propositions/variables $C_i$ and 
definition axioms $C_i \definitionaxiom A_{i_1} \vee A_{i_2}$. Also introduce 
$\pr{A_i=1}=1/2$ for each $A_i$, and consider the query
$\mathbf{Q}=\{C_1,\dots,C_m\}$. Note that  $\pr{\mathbf{Q}}>\gamma$
if and only if $\#\phi' = 2^n\pr{\mathbf{Q}} > 2^n\gamma$, so we can bracket $\#\phi'$. 
From $\#\phi'$ we obtain $\#\phi$ and we can decide whether $\#\phi > 2^{n-1}$, thus
solving the original MAJSAT problem.

To prove hardness of $\mathsf{INF}[\mathsf{Prop}(\wedge)]$,  
note that the number of satisfying assignments of $\phi'$ in 2CNF is equal to the number 
of satisfying assignments of  $\bigwedge_{i=1}^m (\neg A_{i_1} \vee \neg A_{i_2})$,
because one can take each satisfying assignment for the latter sentence and create a satisfying 
assignment for $\phi'$ by interchanging $\mathsf{true}$ and $\mathsf{false}$, and 
likewise for the unsatisfying assignments. 
Introduce fresh propositions/variables $C_i$ and definition axioms
$C_i \definitionaxiom A_{i_1} \wedge A_{i_2}$. Also introduce $\pr{A_i=1}=1/2$ for each $A_i$,
and consider the query where $\mathbf{Q}=\{\neg C_1,\dots,\neg C_m\}$. 
Again we can bracket the number of assignments that satisfy $\phi'$, and 
thus we can solve any MAJSAT problem by using $\mathsf{INF}[\mathsf{Prop}(\wedge)]$
and appropriate auxiliary polynomial computations.
\end{proof}

\addtocounter{Theorem}{1}

\begin{Theorem} 
$\mathsf{INF}[\mathsf{FFFO}]$ is $\mathsf{PEXP}$-complete with
respect to many-one reductions, regardless of
whether the domain is specified in unary or binary notation. 
\end{Theorem}
\begin{proof}
To prove membership, note that a relational Bayesian network specification
based on $\mathsf{FFFO}$ can be grounded into an exponentially large Bayesian 
network, and inference can be carried out in that network using a counting Turing machine
with an exponential bound on time. This is true even if we have  unbounded arity 
of relations: even if we have domain size $2^N$ and maximum arity $k$, grounding 
each relation generates up to $2^{kN}$ nodes, still an exponential quantity in the input. 

To prove hardness, we focus on binary domain size $N$ as this simplifies the notation. 
Clearly if $N$ is given in unary, then an exponential number of groundings can be produced 
by increasing the arity of relations (even if the domain is of size $2$, an arity $k$ leads to 
$2^k$ groundings). 

Given the lack of $\mathsf{PEXP}$-complete problems in the literature, we have to work 
directly from Turing machines.  Start by taking any language $\mathcal{L}$ such that 
$\ell \in \mathcal{L}$ if and only  $\ell$ is accepted by more than half of the computation
paths of a nondeterministic Turing machine $\mathbb{M}$ within time $2^{p(|\ell|)}$ where 
$p$ is a polynomial and $|\ell|$ denotes the size of $\ell$. To simplify matters, 
denote $p(|\ell|)$ by  $n$. The Turing machine is defined by its alphabet, its states, 
and its transition function.

Denote by $\sigma$ a symbol in $\mathbb{M}$'s alphabet, and by $q$ a state of $\mathbb{M}$.
A configuration of $\mathbb{M}$ can be described by a string 
$\sigma^0\sigma^1\dots\sigma^{i-1}(q\sigma^i)\sigma^{i+1}\dots\sigma^{2^n-1}$,
where each $\sigma^j$ is a symbol in the tape, $(q\sigma^i)$ indicates both the state
$q$ and the position of the head at cell $i$ with symbol $\sigma^i$. 
The initial configuration is $(q_0\sigma^0)\sigma_*^1\dots\sigma_*^{m-1}$
followed by $2^n-m$ blanks, where $q_0$ is the initial state. There are also 
states $q_a$ and $q_r$  that respectively indicate acceptance or
rejection of the input string $\sigma_*^0\dots\sigma_*^{m-1}$. 
We assume that if $q_a$ or $q_r$ appear in some configuration, then 
the configuration is not modified anymore (that is, the transition moves from this configuration
to itself). This is necessary to guarantee that the number of accepting computations
is equal to the number of ways in which we can fill in a matrix of computation.
For instance, a 
particular accepting computation could be depicted as a $2^n \times 2^n$ matrix as 
in Figure \ref{figure:Computation}, 
where $\llcorner\!\lrcorner$ denotes the blank, and where we complete the rows of the
matrix after the acceptance by repeating the accepting row.

The transition function $\delta$ of $\mathbb{M}$ takes a pair $(q,\sigma)$ consisting of a state and a symbol 
in the machine's tape, and returns a triple $(q',\sigma',m)$: the next state $q'$, the symbol $\sigma'$
to be written in the tape (we assume that a blank is never written by the machine), 
and an integer $m$ in $\{-1,0,1\}$. Here $-1$ means that the head is to move left, $0$ means
that the head is to stay in the current cell, and $1$ means that the head is to move right. 

\begin{figure}
\[
\begin{array}{|c|c|c|c|c|c|c|c|c|c} \cline{1-9} 
\sigma^0 & \dots & \dots & \dots & \dots & \sigma^{i-1} & (q_a\sigma^i) & \sigma^{i+1} & \dots & \mbox{row } 2^n-1 \\ 
\cline{1-9}  
\vdots & \vdots & \vdots & \vdots & \vdots & \vdots & \vdots & \vdots & \vdots & \mbox{repeat} \\
\cline{1-9}  
\sigma^0 & \dots & \dots & \dots & \dots & \sigma^{i-1} & (q_a\sigma^i) & \sigma^{i+1} & \dots & \mbox{repeat} \\ 
\cline{1-9}  
\sigma^0 & \dots & \dots & \dots & \dots & \sigma^{i-1} & (q_a\sigma^i) & \sigma^{i+1} & \dots & \mbox{acceptance} \\ 
\cline{1-9}  
\vdots & \vdots & \vdots & \vdots & \vdots & \vdots & \vdots & \vdots & \vdots & \mbox{computations} \\ 
\cline{1-9}  
(q_0\sigma_*^0)  & \sigma_*^1 & \dots & \sigma_*^{m-1} &  \llcorner\!\lrcorner  & 
	\dots & \dots & \dots & \llcorner\!\lrcorner & \mbox{row } 0 \mbox{ (input)} \\ 
\cline{1-9}  
\end{array}
\]
\caption{An accepting computation.}
\label{figure:Computation}
\end{figure}

We now encode this Turing machine using monadic logic, mixing some ideas 
by Lewis \cite{Lewis80} and  by Tobies \cite{Tobies2000}.
 
Take a domain of size $2^{2n}$.
The idea is that each $\logvar{x}$ is a cell in the computation matrix.
From now on, a ``point'' is a cell in that matrix. 
Introduce parvariables $X_0(\logvar{x}),\dots,X_{n-1}(\logvar{x})$ and
$Y_0(\logvar{x}),\dots,Y_{n-1}(\logvar{x})$ to encode the index of
the  column and the row of point $\logvar{x}$. 
Impose, for $0 \leq i \leq n-1$,  the assessments 
$\pr{X_i(\logvar{x})=1}  = \pr{Y_i(\logvar{x})=1} =   1/2$.

We need to specify the concept of adjacent points in the computation matrix. 
To this end we introduce two macros,
$\mathsf{EAST}(\logvar{x},\logvar{y})$ and 
$\mathsf{NORTH}(\logvar{x},\logvar{y})$
 (note that we do not actually need binary 
relations here; these expressions are just syntactic sugar). The meaning of 
$\mathsf{EAST}(\logvar{x},\logvar{y})$ is that for   point $\logvar{x}$  there is a point 
$\logvar{y}$  that is immediately to the right of $\logvar{x}$.  And the 
meaning of $\mathsf{NORTH}(\logvar{x},\logvar{y})$ is that for   point $\logvar{x}$  there 
is a point $\logvar{y}$ that is immediately on top of $\logvar{x}$. 
\begin{eqnarray*}
\mathsf{EAST}(\logvar{x},\logvar{y}) & :=  &  
\bigwedge_{k=0}^{n-1} (\wedge_{j=0}^{k-1} X_j(\logvar{x}))\rightarrow(X_k(\logvar{x}) \leftrightarrow \neg X_k(\logvar{y}))  \\
&  &    \wedge \bigwedge_{k=0}^{n-1} (\vee_{j=0}^{k-1} \neg X_j(\logvar{x}))\rightarrow(X_k(\logvar{x}) \leftrightarrow X_k(\logvar{y})) \\
& &     \wedge \bigwedge_{k=0}^{n-1} ( Y_k(\logvar{x}) \leftrightarrow Y_k(\logvar{y}) ). 
\end{eqnarray*}
\begin{eqnarray*}
\mathsf{NORTH}(\logvar{x},\logvar{y}) & := &  
\bigwedge_{k=0}^{n-1} (\wedge_{j=0}^{k-1} Y_j(\logvar{x}))\rightarrow(Y_k(\logvar{x}) \leftrightarrow \neg Y_k(\logvar{y})) \\
&  & 
  \wedge \bigwedge_{k=0}^{n-1} (\vee_{j=0}^{k-1} \neg Y_j(\logvar{x}))\rightarrow(Y_k(\logvar{x}) \leftrightarrow Y_k(\logvar{y})) \\
& &
    \wedge \bigwedge_{k=0}^{n-1} ( X_k(\logvar{x}) \leftrightarrow X_k(\logvar{y}) ).
\end{eqnarray*}
Now introduce
\[
Z_1 \definitionaxiom 
\left( \forall \logvar{x} : \exists \logvar{y} : \mathsf{EAST}(\logvar{x},\logvar{y}) \right)
\wedge
\left( \forall \logvar{x} : \exists \logvar{y} : \mathsf{NORTH}(\logvar{x},\logvar{y}) \right),
\]
\[
Z_2  \definitionaxiom \exists \logvar{x} :  
	\bigwedge_{k=0}^{n-1} (\neg X_k(\logvar{x}) \wedge \neg Y_k(\logvar{x})).
\]
Now if $Z_1 \wedge Z_2$ is true,   we ``build'' a square ``board'' of
size $2^n \times 2^n$ (in fact this is a torus as the top row is followed
by the bottom row, and the rightmost column is followed by the leftmost column).

Introduce a relation $C_j$ for each triplet $(\alpha, \beta, \gamma)$ 
where each element of the triplet is either a symbol $\sigma$ or a symbol 
of the form $(q\sigma)$ for our machine $\mathbb{M}$, and with an additional condition:
if $(\alpha, \beta, \gamma)$ has $\beta$
equal to a blank, then $\gamma$ is a blank. Furthermore, introduce
a relation $C_j$ for each triple $(\diamond, \beta, \gamma)$,
where $\beta$ and $\gamma$ are as before, and $\diamond$ is a new special symbol
(these relations are needed later to encode the ``left border'' of the board).
We refer to each $C_j$ as a {\em tile}, as we are in effect encoding a 
domino system \cite{Lewis80}. For each tile, impose $\pr{C_j(\logvar{x})=1}=1/2$. 

Now each point must have one and only one tile:
\[
Z_3 \definitionaxiom \forall \logvar{x}: \left( \bigvee_{j = 0}^{c-1} C_j(\logvar{x}) \right) \wedge
                        \left( \bigwedge_{0 \leq j \leq c-1, 0 \leq k \leq c-1,  j \neq k} 
                                 \neg (C_j(\logvar{x}) \wedge C_k(\logvar{x})) \right).
\]

Having defined the tiles, we now define a pair of relations encoding the 
``horizontal'' and ``vertical'' constraints on tiles, so as to encode the transition
function of the Turing machine. 
Denote by $H$ the relation consisting of pairs of tiles that satisfy the horizontal constraints
and by $V$ the relation consisting of pairs of tiles that satisfy the vertical constraints. 

The horizontal constraints must enforce the fact that, in a fixed row,
a tile $(\alpha,\beta,\gamma)$ at column $i$ for $0 \leq i \leq 2^n-1$ overlaps
the  tile $(\alpha',\beta',\gamma')$ at column $i+1$ by satisfying
\[
((\alpha,\beta,\gamma),(\alpha',\beta',\gamma')) : \beta = \alpha', \gamma = \beta'.
\]

The vertical constraints must encode the possible computations. To do so, consider
a tile $t = (\alpha,\beta,\gamma)$ at row $j$, for $0 \leq j \leq 2^n-1$,
and tile $t' = (\alpha',\beta',\gamma')$ at row $j+1$, both at the same column.
The pair $(t,t')$ is in $V$ if and only if 
(a) $t'$ can be reached from $t$ given the states in the Turing machine; and
(b) if  $t=(\diamond, \beta, \gamma)$, then $t'=(\diamond, \beta', \gamma')$
for $\beta'$ and $\gamma'$ that follow from the behavior of~$\mathbb{M}$.

We distinguish the last row and the last column, as the transition function
does not apply to them:
\[
D_X(\logvar{x}) \definitionaxiom  \bigwedge_{k=0}^{n-1} X_k(\logvar{x}), 
\qquad
D_Y(\logvar{x}) \definitionaxiom \bigwedge_{k=0}^{n-1}  Y_k(\logvar{x}).
\]
We can now encode the transition function:
\[
Z_4 \definitionaxiom\forall \logvar{x}: \neg D_X(\logvar{x}) \rightarrow \left( \bigwedge_{j = 0}^{c-1} C_j(\logvar{x}) \rightarrow 
(\forall \logvar{y}: \mathsf{EAST}(\logvar{x},\logvar{y}) \rightarrow 
         \vee_{k:(j,k)\in H}  C_k(\logvar{y})) \right),  
\]
\[
Z_5 \definitionaxiom \forall \logvar{x}: \neg D_Y(\logvar{x}) \rightarrow   \left( \bigwedge_{j = 0}^{c-1} C_j(\logvar{x}) 
\!\rightarrow\!
(\forall \logvar{y}: \mathsf{NORTH}(\logvar{x},\logvar{y}) \rightarrow 
       \vee_{k:(j,k)\in V} C_k(\logvar{y})) \right)\!.
\]

We create a parvariable that signals the   accepting state:
\[
Z_6 \definitionaxiom \exists \logvar{x} : \bigvee_{j: C_j \mbox{\scriptsize contains } q_a} C_j(\logvar{x}).
\]

Finally, we must also impose the initial conditions.
Take the tiles in the first row so that symbols in the input of $\mathbb{M}$  
are encoded as $m$ tiles, with the first tile  $t^0=(\diamond,(q_0\sigma_*^0),\sigma_*^1)$
and the following ones $t^j=(\sigma_*^{j-1},\sigma_*^j, \sigma_*^{j+1})$ up 
to $t^{m-1}=(\sigma_*^{m-2},\sigma_*^{m-1},\llcorner\!\lrcorner)$. So
the next tile will be $(\sigma_*^{m-1},\llcorner\!\lrcorner,\llcorner\!\lrcorner)$,
and after that all tiles in the first row will contain only blanks.
Now take individuals $a_i$ for $i \in \{0,\dots,m-1\}$ and create an assignment
$\{C^0_i(a_i)=1\}$ for each $a_i$, where $C^0_i$ is
the $i$th tile encoding the initial conditions. Denote by $\mathbf{E}$
the set of such assignments.
 
Now $\pr{Z_6|\mathbf{E} \wedge \bigwedge_{i=1}^6 Z_i}>1/2$ if and only if the number of 
correct arrangments of tiles that contain
the accepting state is larger than the total number of possible valid arrangements. 
Hence an inference
with the constructed relational Bayesian network specification decides the language
$\mathcal{L}$ we started with, as desired.
\end{proof}


\begin{Theorem} 
$\mathsf{INF}[\mathsf{ALC}]$ is $\mathsf{PEXP}$-complete with respect
to many-one reductions, when domain
size is given in binary notation.
\end{Theorem}
\begin{proof}
Membership follows from Theorem \ref{theorem:INF-FFFO}.
To prove hardness, consider that by imposing an
assessment $\pr{X(\logvar{x},\logvar{y})}=1$, we transform 
$\exists \logvar{y} : X(\logvar{x},\logvar{y}) \wedge Y(\logvar{y})$
into $\exists \logvar{y}: Y(\logvar{y})$. This is all we need (together with
Boolean operators) to build the proof of $\mathsf{PEXP}$-completeness
in Theorem \ref{theorem:INF-FFFO}.
(The inferential complexity of $\mathsf{ALC}$ 
has been derived, with a different proof, by Cozman and Polastro \cite{CozmanUAI2009}.)
\end{proof}


\begin{Theorem} 
$\mathsf{QINF}[\mathsf{FFFO}]$ is $\mathsf{PEXP}$-complete with respect 
to many-one reductions, when the domain is specified in binary notation.
\end{Theorem}
\begin{proof}
Membership 
is obvious as the inferential complexity is already in $\mathsf{PEXP}$.
To show hardness, take a Turing machine $\mathbb{M}$ that solves some
$\mathsf{PEXP}$-complete problem within $2^n$ steps. That is, there 
is  a $\mathsf{PEXP}$-complete language $\mathcal{L}$ such that 
$\ell \in \mathcal{L}$ if and only if the input string $\ell$ is 
accepted by more than half of the computation paths of $\mathbb{M}$ within time $2^n$. 

Such a Turing machine $\mathbb{M}$ has alphabet,
states and transitions as in the proof of Theorem \ref{theorem:INF-FFFO}.
Assume that $\mathbb{M}$ repeats its configuration as soon as it 
enters into the accepting or the rejecting state, as in the proof of
Theorem \ref{theorem:INF-FFFO}. 

To encode $\mathbb{M}$ we resort to a construction by 
Gradel \cite{Gradel2007} where relations of arity two are used. 
We use:
(a) for each state $q$ of $\mathbb{M}$, a unary relation $X_q$;
(b) for each symbol $\sigma$ in the alphabet of $\mathbb{M}$, a binary
relation $Y_\sigma$;
(c) a binary relation $Z$. 
The idea is that $X_q(\logvar{x})$ means that $\mathbb{M}$ is in state $q$ at computation
step $\logvar{x}$, while $Y_\sigma(\logvar{x},\logvar{y})$ means that $\sigma$
is the symbol at the $\logvar{y}$th position in the tape at computation step
$\logvar{x}$, and $Z(\logvar{x},\logvar{y})$ means that the machine head is at
the $\logvar{y}$th position in the tape at computation step $\logvar{x}$. 
Impose 
$\pr{X_q(\logvar{x})=1}=\pr{Y_\sigma(\logvar{x},\logvar{y})=1}=\pr{Z(\logvar{x},\logvar{y})=1}=1/2$.

We use a distinguished relation $<$, assumed not to be in
the vocabulary. This relation is to be a linear order on the domain; 
to obtain this behavior, just introduce $\pr{<\!(\logvar{x},\logvar{y})=1}=1/2$ and  
\begin{eqnarray*}
Z_1 & \definitionaxiom & ( \forall \logvar{x} : \neg (\logvar{x} < \logvar{x}) ) \wedge \\
& & (\forall \logvar{x} : \forall \logvar{y} : \forall \logvar{z} : 
    (\logvar{x} < \logvar{y} \wedge \logvar{y} < \logvar{z} ) \rightarrow \logvar{x} < \logvar{z}) \wedge \\
& & (\forall \logvar{x} : \forall \logvar{y} : 
    (\logvar{x}<\logvar{y}) \vee (\logvar{y}<\logvar{x}) \vee (\logvar{x}=\logvar{y})).
\end{eqnarray*}
We will later set evidence on $Z_1$ to force $<$ to be a linear order. 
The important point is that we can assume that a domain of size $2^{n}$ is
given and all elements of this domain are ordered according to $<$. 

Clearly we can define a successor relation using $<$:
\[
\mathsf{successor}(\logvar{x},\logvar{y}) \definitionaxiom 
(\logvar{x}<\logvar{y}) \wedge \left( \neg \exists \logvar{z} : 
   (\logvar{x}<\logvar{z}) \wedge (\logvar{z}<\logvar{y})\right).
\]
Also, we can define a relation that signals the ``first'' individual:
\[
\mathsf{first}(\logvar{x}) \definitionaxiom \neg \exists \logvar{y} : \logvar{y}<\logvar{x}.
\]

We must guarantee that at any given step the machine is in a single state,
each cell of the tape has a single symbol, and the head is at a a single position of the tape:
\[
Z_2 \definitionaxiom \forall \logvar{x} : 
  \bigvee_q \left( X_q(\logvar{x})  \wedge \bigwedge_{q' \neq q} \neg X_{q'}(\logvar{x}) \right),
\]
\[
Z_3 \definitionaxiom \forall \logvar{x} : \forall \logvar{y} : 
\bigvee_\sigma \left( Y_\sigma(\logvar{x},\logvar{y}) \wedge 
	\bigwedge_{\sigma' \neq \sigma} \neg Y_{\sigma'}(\logvar{x},\logvar{y}) \right),
\]
\[
Z_4 \definitionaxiom \forall \logvar{x} : 
  (\exists \logvar{y} : Z(\logvar{x},\logvar{y}) \wedge   
     \forall \logvar{z} : (\logvar{z} \neq \logvar{y}) \rightarrow \neg Z(\logvar{x},\logvar{z})).
\]

We also have to guarantee that computations do not change the content of a cell
that is not visited by the head:
\[
Z_5 \definitionaxiom \forall \logvar{x} : \forall \logvar{y} : \forall \logvar{z} :
\bigwedge_\sigma 
\left( \neg Z(\logvar{x},\logvar{y}) \wedge  Y_\sigma(\logvar{x},\logvar{y}) \wedge 
\mathsf{successor}(\logvar{x},\logvar{z}) \right)
\rightarrow Y_\sigma(\logvar{z},\logvar{y}).
\]
We must encode the changes made by the transition function:
\begin{eqnarray*}
Z_6 & \definitionaxiom & \forall \logvar{x} : \forall \logvar{y} : \forall \logvar{z} :
\bigwedge_{q,\sigma} 
\left(  Z(\logvar{x},\logvar{y}) \wedge Y_\sigma(\logvar{x},\logvar{y})  \wedge X_q(\logvar{x})
\wedge \mathsf{successor}(\logvar{x},\logvar{z}) \right) 
\rightarrow \\ 
& & 
\bigvee_{(q' \sigma',1) \in \delta(q,\sigma)} 
  \left( X_{q'}(\logvar{z}) \wedge Y_{\sigma'}(\logvar{z},\logvar{y}) \wedge
    (\forall \logvar{w} : \mathsf{successor}(\logvar{y},\logvar{w}) \rightarrow Z(\logvar{z},\logvar{w}) )\right) \\
& &   \vee
\bigvee_{(q' \sigma',0) \in \delta(q,\sigma)} 
  \left( X_{q'}(\logvar{z}) \wedge Y_{\sigma'}(\logvar{z},\logvar{y}) \wedge Z(\logvar{z},\logvar{y}) \right) \\
& &   \vee
\bigvee_{(q' \sigma',1) \in \delta(q,\sigma)} 
  \left( X_{q'}(\logvar{z}) \wedge Y_{\sigma'}(\logvar{z},\logvar{y}) \wedge
    (\forall \logvar{w} : \mathsf{successor}(\logvar{w},\logvar{y}) \rightarrow Z(\logvar{z},\logvar{w}) )\right)\!.
\end{eqnarray*}

We must also guarantee that all cells to the right of a blank cell
are also blank:
\[
Z_7 \definitionaxiom \forall \logvar{x} : \forall \logvar{y} : \forall \logvar{z} :
    Y_{\llcorner\!\lrcorner}(\logvar{x},\logvar{y}) \wedge \mathsf{successor}(\logvar{y},\logvar{z}) 
       \rightarrow Y_{\llcorner\!\lrcorner}(\logvar{x},\logvar{z}).
\]

Finally, we must signal the accepting state:
\[
Z_8 \definitionaxiom \exists \logvar{x} : X_{q_a}(\logvar{x}).
\]

We have thus created a set of formulas that encode the behavior of the 
Turing machine. Now take the input string $\ell$, equal to 
$\sigma^0_*, \sigma^1_*,\dots,\sigma^{m-1}_*$, and encode it as a query
as follows. Start by ``creating'' the first individual in the ordering
 by taking the assignment $\{\mathsf{first}(a_0)=1\}$. 
Then   introduce $\{Z(a_0,a_0)=1\}$ to initialize the head.
Introduce $\{Y_{\sigma^0_*}(a_0,a_0)=1\}$ to impose the initial condition
on the first cell, and 
for each subsequent initial condition $\sigma^i_*$ we set $\{Y_{\sigma^i_*}(a_0,a_i)=1\}$
and $\{\mathsf{successor}(a_{i-1},a_i)=1\}$ where
$a_i$ is a fresh individual. Finally, set $\{Y_{\llcorner\!\lrcorner}(a_0,a_{m})=1\}$
and $\{\mathsf{successor}(a_{m-1},a_m)=1\}$ and $\{X_{q_0}(a_0)=1\}$. 
These assignments are denoted by $\mathbf{E}$. 

Now $\pr{Z_8|\mathbf{E} \wedge \bigwedge_{i=1}^8 Z_i} > 1/2$ for a domain
of size $2^n$ if and only if the number of interpretations reaching the accepting state 
is larger than the total number of possible interpretations encoding computation
paths. 
\end{proof}


\begin{Theorem} 
$\mathsf{QINF}[\mathsf{FFFO}]$ is $\mathsf{PP}$-complete with respect 
to many-one reductions when the domain is specified in unary notation.
\end{Theorem}
\begin{proof}
To prove hardness, take a MAJSAT problem where $\phi$ is in CNF with $m$ clauses
and propositions $A_1,\dots,A_n$. Make sure each clause has at most $n$ literals by
removing repeated literals, or by removing clauses with a  proposition and its negation).
Make sure $m=n$: if $m<n$, then add trivial
clauses such as $A_1 \vee \neg A_1$; if instead $n<m$, then add fresh
propositions $A_{n+1},\dots,A_m$. These changes 
do not change the output of  MAJSAT. 
Introduce unary relations $\mathsf{sat}(\logvar{x})$ and
impose $\pr{\mathsf{sat}(\logvar{x})}=1/2$.
Take a domain $\{\mathtt{1},\dots,\mathtt{n}\}$; 
the elements of the domain serve a dual purpose, indexing
both propositions and clauses.
Introduce  relations $\mathsf{sat}(\logvar{x})$,  
$\mathsf{positiveLit}(\logvar{x},\logvar{y})$ and 
$\mathsf{negativeLit}(\logvar{x},\logvar{y})$,
assessments 
$\pr{\mathsf{sat}(\logvar{x})=1}=\pr{\mathsf{positiveLit}(\logvar{x},\logvar{y})=1}=
\pr{\mathsf{negativeLit}(\logvar{x},\logvar{y})=1}=1/2$,
and definition axioms:
\begin{eqnarray*}
\mathsf{clause}(\logvar{x}) & \definitionaxiom & 
\exists \logvar y : (\mathsf{positiveLit}(\logvar{x},\logvar{y}) \wedge \mathsf{sat}(\logvar{y}))  \vee  
(\mathsf{negativeLit}(\logvar{x},\logvar{y}) \wedge \neg \mathsf{sat}(\logvar{y})), \\
\mathsf{query} & \definitionaxiom & \forall \logvar{x} : \mathsf{clause}(\logvar{x}).
\end{eqnarray*}
Take evidence $\mathbf{E}$ as follows.
For each clause, run over the literals. Consider the $i$th clause,
and its non-negated literal $A_j$: set $\mathsf{positiveLit}(i,j)$ to true.
And consider negated literal $\neg A_j$: set $\mathsf{negativeLit}(i,j)$ to true.
Set all other groundings of $\mathsf{positiveLit}$ and 
$\mathsf{negativeLit}$  to false.  
Note that $\pr{\mathbf{E}}=2^{-2n^2}>0$. 
Now decide whether $\pr{\mathsf{query}=1|\mathbf{E}}>1/2$.  If
YES, the MAJSAT problem is accepted, if NO, it is not accepted. Hence
we have the desired polynomial reduction (the query is quadratic on
domain size; all other elements are linear on domain size).

To prove membership in $\mathsf{PP}$, we describe a Turing machine $\mathbb{M}$ that  
decides whether $\pr{Q|\mathbf{E}}>\gamma$.
The machine guesses a truth assignment for each one of the polynomially-many
grounded root nodes (and writes the guess in the working tape). Note that each 
grounded root node $X$ is associated with an assessment $\pr{X=1}=c/d$, where 
$c$ and $d$ are the smallest such integers. 
Then the machine replicates its computation paths 
out of the guess on $X$: there are $c$ paths with identical behavior for
guess $\{X=1\}$, and $d-c$ paths with identical behavior for guess $\{X=0\}$.

Now the machine verifies whether 
the  set of  guessed truth assignment satisfies $\mathbf{E}$; if not, move to state $q_1$.
If yes, then verify whether the guessed truth assignment fails to satisfy $\mathbf{Q}$;
if not, move to state $q_2$. And if yes, then move to state $q_3$. 
The key point is that there is a logarithmic space, hence polynomial time, algorithm
 that can verifiy whether a set of assignments holds once the 
root nodes are set \cite[Section 6.2]{Libkin2004}.

Suppose that out of $N$ computation paths that $\mathbb{M}$ can take, $N_1$ of them 
reach $q_1$, $N_2$ reach $q_2$, and $N_3$ reach $q_3$. By construction,
\begin{equation}
\label{equation:Park}
N_1/N = 1-\pr{\mathbf{E}}, \quad
N_2/N=\pr{\neg \mathbf{Q},\mathbf{E}}, \quad
N_3/N=\pr{\mathbf{Q},\mathbf{E}},
\end{equation}
where we abuse notation by taking $\neg \mathbf{Q}$ to mean that some
assignment in $\mathbf{Q}$ is not true.
Note that up to this point we do not have any rejecting
nor accepting path, so the specification of $\mathbb{M}$ is not complete.

The remainder of this proof just reproduces a construction by Park
in his proof of   $\mathsf{PP}$-completeness for propositional Bayesian 
networks~\cite{Darwiche2009}.  
Park's construction adds rejecting/accepting computation paths emanating
from $q_1$, $q_2$ and $q_3$. It uses numbers
\[
a = \left\{ \begin{array}{ll} 1 & \mbox{ if } \gamma < 1/2, \\
                                       1/(2\gamma) & \mbox{ otherwise,}
\end{array}\right. \qquad \quad
b = \left\{ \begin{array}{ll} (1-2\gamma)/(2-2\gamma) & \mbox{ if } \gamma < 1/2, \\
                                       0 & \mbox{ otherwise.}
\end{array}\right.
\]
and the smallest integers  $a_1$, $a_2$, $b_1$, $b_2$ such
that $a = a_1/a_2$ and $b = b_1/b_2$.
Now, out of $q_1$  
branch into $a_2b_2$ computation paths that immediately stop at
the accepting state, and $a_2b_2$ computation paths that immediately stop
at the rejecting state.\footnote{The number of created paths may
be exponential in the numbers $a2$ and $b_2$; however it is always 
possible to construct a polynomial sequence of steps that encodes an 
exponential number of paths (say the number of paths has $B$ bits; then 
build $B$ distinct branches, each one of them multiplying alternatives so as
to simulate an exponential). This sort of branching scheme is also assumed 
whenever needed.}
Out of $q_2$ branch into
$2a_2b_1$ paths that immediately stop at the accepting state, and 
$2(b_2-b_1)a_2$ paths that immediately stop at the rejecting state.
Out of $q_3$ branch into
$2a_1b_2$ paths that immediately stop at the accepting state, and 
$2(a_2-a_1)b_2$ paths that immediately stop at the rejecting state.
For the whole machine $\mathbb{M}$, 
the number of computation paths that end up at the accepting state is
$a_2b_2N_1+2a_2b_1N_2+2a_1b_2N_3$, and the total number of
computation paths is
$a_2b_2N_1+a_2b_2N_1+2b_1a_2N_2+2(b_2-b_1)a_2N_2+2a_1b_2N_3+2(a_2-a_1)b_2N_3=2a_2b_2N$.
Hence the number of accepting paths divided by the total number of paths
is $(N_1(1/2)+(b_1/b_2)N_2+(a_1/a_2)N_3)/N$.
This ends Park's construction. By combining this
 construction with Expression (\ref{equation:Park}), we obtain
\[
\frac{\frac{N_1}{2}+\frac{b_1N_2}{b_2}+\frac{a_1N_3}{a_2}}{N} > 1/2 
\quad
\Leftrightarrow 
\quad 
\frac{1-\pr{\mathbf{E}}}{2} + b\pr{\neg \mathbf{Q},\mathbf{E}} + a\pr{\mathbf{Q},\mathbf{E}} > 1/2
\]
\[
\Leftrightarrow
\quad 
a \pr{\mathbf{Q},\mathbf{E}}+b\pr{\neg \mathbf{Q},\mathbf{E}}>\pr{\mathbf{E}}/2
\quad 
\Leftrightarrow
\quad 
a \pr{\mathbf{Q}|\mathbf{E}} + b \pr{\neg \mathbf{Q}|\mathbf{E}} > 1/2,
\]
as we can assume that $\pr{\mathbf{E}}>0$ (otherwise the number of accepting paths
is equal to the number of rejecting paths), and then
\[
\left\{
\begin{array}{ccl}
\mbox{ if } \gamma < 1/2: & 
\pr{\mathbf{Q}|\mathbf{E}}+\frac{1-2\gamma}{2-2\gamma}(1-\pr{\mathbf{Q}|\mathbf{E}}) > 1/2 
& \Leftrightarrow \pr{\mathbf{Q}|\mathbf{E}}>\gamma; \\
\mbox{ if } \gamma \geq 1/2: & 
(1/(2\gamma)) \pr{\mathbf{Q}|\mathbf{E}} > 1/2
& \Leftrightarrow  \pr{\mathbf{Q}|\mathbf{E}}>\gamma.
\end{array} \right.
\]
Hence the number of accepting computation paths of $\mathbb{M}$ 
is larger than half the total number of computation paths 
if and only if $\pr{\mathbf{Q}|\mathbf{E}}>\gamma$.
This completes the proof of membership.
\end{proof}


\begin{Theorem}
Suppose $\mathsf{NETIME} \neq \mathsf{ETIME}$.
Then $\mathsf{DINF}[\mathsf{FFFO}]$ is not solved in deterministic exponential time,
when the domain size is given in binary notation.
\end{Theorem}
\begin{proof}
Jaeger describes results  implying, in case  $\mathsf{NETIME} \neq \mathsf{ETIME}$,
that there is a sentence $\phi \in \mathsf{FFFO}$ such that the {\em spectrum} of
$\phi$ cannot be recognized in deterministic exponential time \cite{Jaeger2014}. 
Recall: the spectrum of
a sentence is a set containing each integer $N$, in binary notation, such that $\phi$ 
has a model whose domain size is $N$  \cite{Gradel2007}. So take $N$ in binary notation,
 the relational Bayesian network specification $A \definitionaxiom \phi$,
and decide whether $\pr{A}>0$ for domain size $N$; if yes, then $N$ is in the spectrum of $\phi$.
\end{proof}


\begin{Theorem} 
$\mathsf{DINF}[\mathsf{FFFO}]$ is $\mathsf{PP}_1$-complete with respect
to many-one reductions, when the domain is given in unary notation.
\end{Theorem}

\begin{proof}
To prove membership, just consider the Turing machine used in the proof of Theorem \ref{theorem:QINF-FFFO-unary},
now with a fixed query. This is a polynomial-time nondeterministic Turing machine that gets 
the domain size in unary (that is, as a sequence of $1$s) and produces the desired output. 

To prove hardness, take a Turing machine with input alphabet consisting of symbol $1$,
and that solves a $\mathsf{PP}_1$-complete problem in 
$N^m$ steps for input consisting of $N$ symbols $1$. 
Take the probabilistic assessment and the definition axioms for 
$\mathsf{successor}$, $\mathsf{first}$, and $Z_1$  as in the proof
of Theorem \ref{theorem:QINF-FFFO-binary}.
Now introduce relations $X_q$, $Y_\sigma$ and $Z$ as in that proof,
with the difference that $\logvar{x}$ is substituted for $m$ logvars $\logvar{x}_i$,
and likewise $\logvar{y}$ is substituted for $m$ logvars $\logvar{x}_j$. 
For instance, we now have $Z(\logvar{x}_1,\dots,\logvar{x}_m,\logvar{y}_1,\dots,\logvar{y}_m)$.
Repeat definition axioms for $Z_2,\dots,Z_8$ as presented in the proof 
of Theorem \ref{theorem:QINF-FFFO-binary}, with appropriate changes in
the arity of relations.  In doing so we have an encoding
for the Turing machine where the computation steps are indexed by a vector
$[\logvar{x}_1,\dots,\logvar{x}_m]$, and the tape is indexed by a vector 
$[\logvar{y}_1,\dots,\logvar{y}_m]$.
The remaining problem is to insert the input. To do so, introduce:
\begin{eqnarray*}
Z_9 & \definitionaxiom & \forall \logvar{x} :  \forall \logvar{y}_1 : \dots \forall \logvar{y}_m : 
\mathsf{first}(\logvar{x}) \rightarrow \\
& & \hspace*{1cm}
\left(  \bigwedge_{i \in \{2,\dots, m\}} \mathsf{first}(\logvar{y}_i) 
\rightarrow
 Y_1(\overbrace{\logvar{x}, \dots, \logvar{x}}^{m \, \mathrm{logvars}}, \logvar{y}_1,\dots,\logvar{y}_m) \right) \\
& & \hspace*{1cm}
\wedge
\left(  \neg \bigwedge_{i \in \{2,\dots, m\}} \mathsf{first}(\logvar{y}_i) 
\rightarrow
 Y_{\llcorner\!\lrcorner}(\overbrace{\logvar{x}, \dots, \logvar{x}}^{m \, \mathrm{logvars}}, \logvar{y}_1,\dots,\logvar{y}_m) \right).
\end{eqnarray*}
Now $\pr{Z_8|\bigwedge_{i=1}^9 Z_i} > 1/2$ for a domain of size $N$
if and only if the number of interpretations that set an accepting state to true
is larger than half the total number of interpretations encoding computation paths.
\end{proof}


\begin{Theorem}
$\mathsf{INF}[\mathsf{FFFO}]$ is $\mathsf{PSPACE}$-complete with respect to 
many-one reductions, when relations have
bounded arity and the domain size is given in unary notation.
\end{Theorem}
\begin{proof}
To prove membership, construct a Turing machine that goes over the truth
assignments for all of the polynomially-many grounded root nodes. The machine
generates an assignment, writes it using polynomial space, and  verifies whether 
$\mathbf{E}$ can be satisfied: there is a polynomial space algorithm
to do this, as we basically need to do model checking in first-order logic
\cite[Section 3.1.4]{Gradel2007}. While cycling through truth assignments,
keep adding the probabilities of the truth assignments that satisfy $\mathbf{E}$. 
If the resulting probability for $\mathbf{E}$ is zero, reject;
otherwise, again go through every truth assignment of
the root nodes, now keeping track of how many of them satisfy 
$\{\mathbf{Q},\mathbf{E}\}$, and adding the probabilities for these assignments.
Then divide the probability of $\{\mathbf{Q},\mathbf{E}\}$ by the probability
of $\mathbf{E}$, and compare the result with the rational number $\gamma$. 

To show hardness,  consider the definition axiom
$Y \definitionaxiom Q_1 \logvar{x}_1 : \dots Q_n \logvar{x}_n : \phi(\logvar{x}_1,\dots,\logvar{x}_n)$,
where each $Q_i$ is a quantifier (either $\forall$ or $\exists$) and $\phi$ is a quantifier-free
formula containing only Boolean operators,  a unary relation $X$, and 
logvars $\logvar{x}_1,\dots,\logvar{x}_n$. 
The relation $X$ is associated with assessment $\pr{X(\logvar{x})=1}=1/2$.
Take domain $\mathcal{D}=\{\mathtt{0},\mathtt{1}\}$ and 
evidence $\mathbf{E}=\{X(\mathtt{0})=0,X(\mathtt{1})=1\}$.
Then $\pr{Y=1|\mathbf{E}}>1/2$ if and only if 
$Q_1 \logvar{x}_1 : \dots Q_n \logvar{x}_n : \phi(\logvar{x}_1,\dots,\logvar{x}_n)$ 
is satisfiable. Deciding the latter satisfiability question is in fact
equivalent to deciding the satisfiability of a Quantified Boolean Formula,
a $\mathsf{PSPACE}$-complete problem
 \cite[Section 6.5]{Libkin2004}.  
\end{proof} 

Now consider the bounded variable fragment $\mathsf{FFFO}^k$.
It is important to notice that if the body of every definition axiom
belongs to $\mathsf{FFFO}^k$ for an integer $k$, then all definition axioms
together are equivalent to a single formula in $\mathsf{FFFO}^k$. Hence
results on logical inference for $\mathsf{FFFO}^k$ can be used to derive
inferential, query and domain complexities. 


\begin{Theorem} 
$\mathsf{INF}[\mathsf{FFFO}^k]$ is $\mathsf{PP}$-complete with respect to 
many-one reductions, for  all $k \geq 0$, when the domain size is given in unary notation.
\end{Theorem}
\begin{proof}
Hardness is trivial: even $\mathsf{Prop}(\wedge,\neg)$ is $\mathsf{PP}$-hard.
To prove membership, use the Turing machine described in the proof of 
membership in Theorem \ref{theorem:QINF-FFFO-unary}, with a small
difference: when it is necessary to check whether $\mathbf{E}$ 
(or $\mathbf{Q} \cup \mathbf{E}$) holds
given a guessed assignment for root nodes, use the appropriate model
checking algorithm \cite{Vardi82}, as this verification can be done
in polynomial time. 
\end{proof}


\begin{Theorem} 
$\mathsf{QINF}[\mathsf{FFFO}^k]$ is $\mathsf{PP}$-complete with respect 
to many-one reductions, for all $k \geq 2$, when
domain size is given in unary notation.
\end{Theorem}
\begin{proof}
To prove membership, note that $\mathsf{QINF}[\mathsf{FFFO}]$ is in $\mathsf{PP}$
by Theorem~\ref{theorem:QINF-FFFO-unary}.
To prove hardness, note that the proof of hardness in 
Theorem~\ref{theorem:QINF-FFFO-unary} uses only $\mathsf{FFFO}^2$. 
\end{proof}


\begin{Theorem} 
$\mathsf{DINF}[\mathsf{FFFO}^k]$ is $\mathsf{PP}_1$-complete with respect to 
many-one reductions, for $k>2$, and is in $\mathsf{P}$ for $k \leq 2$, when  
the domain size is given in unary notation.
\end{Theorem}
\begin{proof}
For $k \leq 2$, results in the literature show how to count 
the number of satisfying models of a formula in polynomial
time \cite{Broeck2011,Broeck2014}. 

For $k > 2$, membership obtains as in the proof of Theorem
\ref{theorem:DINF-FFFO-unary}.  Hardness has been in essence proved
by Beame et al.\ \cite[Lemmas 3.8, 3.9]{Beame2015}.
We adapt their arguments, simplifying them by removing the
need to enumerate the counting Turing machines.
Take a Turing machine $\mathbb{M}$ that solves a $\#\mathsf{P}_1$-complete 
problem in $N^m$ steps for an input consisting of $N$ ones. By padding the
input, we can always guarantee that $\mathbb{M}$ runs in time linear in the input.
To show this, consider that for the input sequence with $N$ ones, 
we can generate another sequence $S(N)$ consisting of
 $f(N) = (2N+1)2^{m \lceil \log_2 N \rceil}$ ones.
Because $(2^{1+\log_2 N})^m \geq 2^{m \lceil \log_2 N \rceil}$, 
we have $(2N+1)2^mN^m > f(N)$, and consequently $S(N)$ can be generated
in polynomial time. Modify $\mathbb{M}$ so that the new machine:
(a) receives $S(N)$;
(b) in linear time produces the binary representation of $S(N)$, using
an auxiliary tape;\footnote{For instance: go from left to right replacing
pairs $11$ by new symbols $\clubsuit\heartsuit$; if a blank is reached in
the middle of such a pair, then add a $1$ at the first blank in the auxiliary tape,
and if a blank is reached after such a pair, then add a $0$ at the first blank
in the auxiliary tape; then mark the current end of the auxiliary tape with a
symbol $\spadesuit$ and return from the end of the  main tape, erasing it
and adding a $1$ to the end of the auxiliary tape for each $\heartsuit$ in the
main tape; now copy the $1$s after $\spadesuit$ from the auxiliary tape to 
the main tape (and remove these $1$s from the auxiliary tape), and repeat. 
Each iteration has cost smaller than
$(U+U+\log U)c$ for some constant $c$, where $U$ is the number of ones
in the main tape; thus the total cost from input of size $f(N)$ is
smaller than $3c(f(N)+f(N)/2+f(N)/4+\dots) \leq 6cf(N)$.}
(c) then discards the trailing zeroes to obtain $2N+1$;
(d) obtains $N$;
(e) writes $N$ ones in its tape; 
(f) then runs the original computation in $\mathbb{M}$. 
Because $2^{m \lceil \log_2 N \rceil} \geq N^m$, we have $f(N) > N^m$,
and consequently the new machine runs in time that is overall linear in
the input size $f(N)$, and in space within $f(N)$. 
Suppose, to be concrete, that the new machine runs in time that is smaller
than $M f(N)$ for some integer $M$. We just have to encode this machine
in $\mathsf{FFFO}^3$, by reproducing a clever
construction due to Beame et al.~\cite{Beame2015}.

We use the Turing machine encoding 
described in the proof of Theorem~\ref{theorem:DINF-FFFO},
but instead of using a single relation $Z(\logvar{x},\logvar{y})$ to indicate the 
head position at step $\logvar{x}$, we use 
\[
Z^{1,1}(\logvar{x},\logvar{y}),\dots,Z^{M,1}(\logvar{x},\logvar{y}),
 Z^{1,2}(\logvar{x},\logvar{y}),\dots,Z^{M,2}(\logvar{x},\logvar{y}),
\] 
with the understanding that for a fixed $\logvar{x}$ we have that $Z^{i,j}(\logvar{x},\logvar{y})$ 
yields the position $\logvar{y}$ of the head in step $\logvar{x}$ and sub-step $i$, either in
the main tape (tape $1$) or in the auxiliary tape (tape $2$). 
So, $Z^{1,j}$ is followed by $Z^{2,j}$ and so on until $Z^{M,j}$ for a fixed 
step $\logvar{x}$. Similarly, we use $X^{t}_q(\logvar{x})$,
$Y^{t,1}_\sigma(\logvar{x},\logvar{y})$ and $Y^{t,2}_\sigma(\logvar{x},\logvar{y})$
for $t \in \{1,\dots,M\}$. Definition axioms must be changed accordingly;
for instance, we have
\[
Z_2 \definitionaxiom 
\forall \logvar{x} : \bigwedge_t \bigvee_q \left( X^t_q(\logvar{x}) \wedge 
 \bigwedge_{q' \neq q} \neg X^t_{q'}(\logvar{x}) \right),
\]
and
\[
Z_3  \definitionaxiom 
\forall \logvar{x} : \bigwedge_t \forall \logvar{y} : \bigwedge_{j \in \{1,2\}}
\bigvee_\sigma \left(
Y^{t,j}_\sigma(\logvar{x},\logvar{y}) \wedge 
             \bigwedge_{\sigma' \neq \sigma} \neg Y^{t,j}_{\sigma'}(\logvar{x},\logvar{y}) 
\right).
\]
As another example, we can change $Z_4$ as follows.
First, introduce auxiliary definition axioms:
\[
W^t_1(\logvar{x}) \definitionaxiom \exists \logvar{y} :
Z^{t,1}(\logvar{x},\logvar{y}) \wedge 
    (\forall \logvar{z} : (\logvar{z} \neq \logvar{y}) \rightarrow \neg Z^{t,1}(\logvar{x},\logvar{z}))  \wedge
    (\forall \logvar{z} : \neg Z^{t,2}(\logvar{x},\logvar{z})),
\]
\[
W^t_2(\logvar{x})  \definitionaxiom  \exists \logvar{y} :
Z^{t,2}(\logvar{x},\logvar{y}) \wedge 
    (\forall \logvar{z} : (\logvar{z} \neq \logvar{y}) \rightarrow \neg Z^{t,2}(\logvar{x},\logvar{z}))  \wedge
    (\forall \logvar{z} : \neg Z^{t,1}(\logvar{x},\logvar{z})),
\]
and then write:
\[
Z_4 \definitionaxiom
\forall \logvar{x} : \bigwedge_t  W^t_1(\logvar{x})  \wedge  W^t_2(\logvar{x}).
\]
Similar changes must be made to  $Z_7$ and $Z_8$:
\[
Z_7 \definitionaxiom
\forall \logvar{x} : \bigwedge_t \forall \logvar{y} : \bigwedge_{j \in \{1,2\}} \forall \logvar{z} :
Y^{t,j}_{\llcorner\!\lrcorner}(\logvar{x},\logvar{y}) \wedge \mathsf{successor}(\logvar{y},\logvar{z})
\rightarrow Y^{t,j}_{\llcorner\!\lrcorner}(\logvar{x},\logvar{z}),
\] 
\[
Z_8 \definitionaxiom
\exists \logvar{x} : \bigvee_t X^t_{q_a}(\logvar{x}).
\]
The changes to $Z_5$ and $Z_6$ are similar, but require more tedious repetition; we
omit the complete expressions but explain the procedure. Basically, $Z_5$ and $Z_6$
encode the transitions of the Turing machine. So, instead of just taking the
successor of a computation step $\logvar{x}$, we must operate in substeps:
the successor of step $\logvar{x}$ substep $t$ is $\logvar{x}$ substep $t+1$,
unless $t=M$ (in which case we must move to the successor of $\logvar{x}$, substep $1$). 
We can also capture the behavior of the Turing machine with two transition functions, 
one per tape, and it is necessary to encode each one of them appropriately. 
It is enough to have $M$ different versions of $Z_5$ and $2M$ different versions of
$Z_6$, each one of them responsible for one particular substep transition. 

To finish, we must encode the initial conditions. Introduce: 
\[
\mathsf{last}(\logvar{x}) \definitionaxiom \neg \exists \logvar{y} : \logvar{x} < \logvar{y}
\]
and
\begin{eqnarray*}
Z_9 & \definitionaxiom &
\left(
\forall \logvar{x} : \forall \logvar{y} : (\mathsf{first}(\logvar{x}) \wedge \neg \mathsf{last}(\logvar{y}))
       \rightarrow Y^{1,1}_1(\logvar{x},\logvar{y}) 
\right) \\
& & \wedge 
\left(
\forall \logvar{x} : \forall \logvar{y} : (\mathsf{first}(\logvar{x}) \wedge \mathsf{last}(\logvar{y}))
       \rightarrow Y^{1,1}_{\llcorner\!\lrcorner}(\logvar{x},\logvar{y}) 
\right) \\
& & \wedge
\left(
\forall \logvar{x} : \forall \logvar{y} : \mathsf{first}(\logvar{x}) \rightarrow Y^{1,2}_{\llcorner\!\lrcorner}(\logvar{x},\logvar{y}) 
\right).
\end{eqnarray*}
Now $\pr{Z_8|\bigwedge_{i=1}^9 Z_i} > 1/2$ for a domain of size $f(N)+1$ if
and only if the number of interpretations that set an accepting state to true
is larger than half the total number of interpretations encoding computation paths. 
\end{proof}

\addtocounter{Theorem}{1}


\begin{Theorem} 
Suppose relations have bounded arity.
$\mathsf{INF}[\mathsf{QF}]$ and $\mathsf{QINF}[\mathsf{QF}]$ are
$\mathsf{PP}$-complete with respect to many-one  reductions, 
and $\mathsf{DINF}[\mathsf{QF}]$ requires constant computational effort. 
These results hold {\em even if domain size is given in binary notation}.
\end{Theorem}
\begin{proof}
Consider first $\mathsf{INF}[\mathsf{Q}]$. 
To prove membership, take a relational Bayesian network
specification $\mathbb{S}$ with relations $X_1,\dots,X_n$, all with arity no larger 
than $k$. Suppose we ground this specification on a domain of size $N$. 
To compute $\pr{\mathbf{Q}|\mathbf{E}}$, the only relevant  groundings are the 
ones that are ancestors of each of the ground atoms
in $\mathbf{Q}\cup\mathbf{E}$.  Our strategy will be to bound the number
of such relevant groundings. To do that,  take a grounding  $X_i(a_1,\dots,a_{k_i})$  
in $\mathbf{Q}\cup\mathbf{E}$, and suppose that $X_i$ is not a root node in the
parvariable graph. Each parent $X_j$ of $X_i$ in the parvariable graph may
appear in several different forms in the definition axiom related to $X_i$;
that is, we may have $X_j(\logvar{x}_2, \logvar{x}_3), X_j(\logvar{x}_9, \logvar{x}_1),\dots$,
and each one of these combinations leads to a distinct grounding. There are in fact at most
$k_i^{k_i}$ ways to select individuals from the grounding $X_i(a_1,\dots,a_{k_i})$
so as to form groundings of $X_j$. So for each parent of $X_i$ in the 
parvariable graph there will be at most $k^k$ relevant groundings. 
And each parent of these parents will again have at most $k^k$ relevant
groundings; hence there are at most $(n-1)k^k$ relevant groundings that 
are ancestors of $X_i(a_1,\dots,a_{k_i})$. 
We can take the union of all groundings that are ancestors of groundings
of $\mathbf{Q}\cup\mathbf{E}$, and the number of such groundings is still
polynomial in the size of the input. Thus in polynomial time we can build 
a polynomially-large
Bayesian network that is a fragment of the grounded Bayesian network. 
Then we can run a Bayesian network inference in this smaller network
(an effort within $\mathsf{PP}$); note that domain size is actually not important so it
can be specified either in unary or binary notation.
To prove hardness, note that $\mathsf{INF}[\mathsf{Prop}(\wedge,\neg)]$
is $\mathsf{PP}$-hard, and a propositional specification can be reproduced
within $\mathsf{QF}$. 

Now consider $\mathsf{QINF}[\mathsf{QF}]$. 
To prove membership, note that even $\mathsf{INF}[\mathsf{QF}]$ is in $\mathsf{PP}$.
To prove hardness, take an instance of $\#3\mathsf{SAT}(>)$ consisting of a sentence $\phi$
in 3CNF, with propositions $A_1,\dots,A_n$, and an integer $k$. 
Consider the relational Bayesian network specification consisting of eight definition axioms:
\begin{eqnarray*}
\mathsf{clause0}(\logvar{x},\logvar{y},\logvar{z}) & \definitionaxiom & 
 \neg \mathsf{sat}(\logvar{x}) \vee \neg \mathsf{sat}(\logvar{y}) \vee \neg \mathsf{sat}(\logvar{z}), \\
\mathsf{clause1}(\logvar{x},\logvar{y},\logvar{z}) & \definitionaxiom & 
 \neg \mathsf{sat}(\logvar{x}) \vee \neg \mathsf{sat}(\logvar{y}) \vee \mathsf{sat}(\logvar{z}), \\
\mathsf{clause2}(\logvar{x},\logvar{y},\logvar{z}) & \definitionaxiom & 
 \neg \mathsf{sat}(\logvar{x}) \vee  \mathsf{sat}(\logvar{y}) \vee \neg \mathsf{sat}(\logvar{z}), \\
\vdots &  \vdots & \vdots \\
\mathsf{clause7}(\logvar{x},\logvar{y},\logvar{z}) & \definitionaxiom & 
  \mathsf{sat}(\logvar{x}) \vee  \mathsf{sat}(\logvar{y}) \vee  \mathsf{sat}(\logvar{z}), 
\end{eqnarray*}
and $\pr{\mathsf{sat}(\logvar{x})=1} = 1/2$. 
Now the query is just a set of assignments $\mathbf{Q}$ ($\mathbf{E}$ is empty)
containing an assignment per clause. If a clause is $\neg A_2 \vee A_3 \vee \neg A_1$,
then take the corresponding assignment $\{ \mathsf{clause2}(a_2,a_3,a_1) = 1 \}$, and so on. 
The $\#3\mathsf{SAT}(>)$ problem is solved by deciding whether $\pr{\mathbf{Q}}>k/2^n$ with 
domain of size $n$; hence the desired hardness is proved.


And $\mathsf{DINF}[\mathsf{QF}]$ requires constant effort: in fact, domain
size is not relevant to a fixed inference, as can be seen from the proof of
inferential complexity above. 
\end{proof}


\begin{Theorem} 
Suppose the domain size is specified in unary notation.
Then $\mathsf{INF}[\mathsf{EL}]$ and 
$\mathsf{QINF}[\mathsf{EL}]$ are $\mathsf{PP}$-complete with respect to
many-one reductions, even if the query contains only positive
assignments, and $\mathsf{DINF}[\mathsf{EL}]$ is in $\mathsf{P}$.
\end{Theorem}
\begin{proof}
$\mathsf{INF}[\mathsf{EL}]$ belongs to $\mathsf{PP}$ by 
Theorem \ref{theorem:INF-FFFO-k} as $\mathsf{EL}$ belongs
to $\mathsf{FFFO}^2$. Hardness is obtained from hardness of
query complexity.

So, consider $\mathsf{QINF}[\mathsf{EL}]$. Membership follows
from membership of $\mathsf{INF}[\mathsf{EL}]$, so we focus on 
hardness. Our strategy is to reduce $\mathsf{INF}[\mathsf{Prop}(\vee)]$
to  $\mathsf{QINF}[\mathsf{EL}]$, using most of the construction
in the proof of Theorem \ref{theorem:and-or}.
So take a sentence $\phi$ in 3CNF with  propositions $A_1,\dots,A_n$ and $m$ 
clauses, and an integer $k$. 
The goal is to decide whether $\#\mbox{(1-in-3)}\phi > k$. 
We can assume that no clause contains a repeated literal.

We start by adapting several steps in the proof of Theorem \ref{theorem:and-or}.
First, associate each literal with a random variable $X_{ij}$
(where $X_{ij}$ stands for a {\em negated} literal).
In the present proof we use a parvariable $X(\logvar{x})$; the idea is that 
$\logvar{x}$ is the integer $3(i-1) + j$ for some $i\in\{1,\dots,n\}$ and $j\in\{1,2,3\}$
(clearly we can obtain $(i,j)$ from $\logvar{x}$ and vice-versa). 
Then associate $X$ with the assessment 
\[
\pr{X(\logvar{x})=1}=\varepsilon,
\]
where $\varepsilon$ is exactly as in the proof for $\mathsf{INF}[\mathsf{Prop}(\vee)]$.

The next step in the proof of Theorem \ref{theorem:and-or} is 
to introduce a number of definition axioms 
of the form $Y_{iuv} \definitionaxiom X_{iu} \vee X_{iv}$, together
with assignments $\{Y_{iuv}=1\}$. There are $3m$ such axioms.
Then additional axioms are added to guarantee that configurations are sensible.
Note that we can compute in polynomial time the total number of definition
axioms that are to be created.  We denote this number by $N$, as we will
use it as the size of the domain. In any case, we can easily bound $N$:
first, each clause produces $3$ definition axioms as in
Expression \ref{equation:Respectful}; second, to guarantee that configurations
are sensible, every time a literal is identical to another literal, or identical to
the negation of another literal, 
four definition axioms are inserted (there are  $3m$ literals, and 
for each one there may be $2$ identical/negated literals in the other $m-1$
clauses). Thus we have that  $N \leq 3m+4 \times 3m \times 2(m-1)=24m^2-21m$.
Suppose we order these definition axioms from $1$ to $N$ by some appropriate
scheme. 

To encode these $N$ definition axioms, we introduce two other parvariables $Y(\logvar{x})$
and $Z(\logvar{x},\logvar{y})$, with definition axiom 
\[
Y(\logvar{x}) \definitionaxiom  \exists \logvar{y} : Z(\logvar{x},\logvar{y}) \wedge X(\logvar{y})
\]
and assessment 
\[
\pr{Z(\logvar{x},\logvar{y})=1}=\eta,
\]
for some $\eta$ to be determined later. 
The idea is this. We take a domain with size $N$, and for each $\logvar{x}$ 
from $1$ to $N$, we set $Z(\logvar{x},\logvar{y})$ to $0$ if $X(\logvar{x})$ 
does not appear in the definition axiom indexed by $\logvar{x}$,
and we set $Z(\logvar{x},\logvar{y})$ to $1$ if $X(\logvar{x})$ appears in the
definition axiom indexed by $\logvar{x}$. 
We collect all these assignments in a set $\mathbf{E}$. 
Note that $\mathbf{E}$ in effect ``creates'' all the desired definition axioms
by selecting two instances of $X$ per instance of~$Y$. 

Note that if we enforce $\{Y(\logvar{x})=1\}$ for all $\logvar{x}$,
we obtain the same construction used in the proof of 
Theorem \ref{theorem:and-or}, we one difference: in that proof we had 
$3m$ variables $X_{ij}$, while here we have $N$ variables $X(\logvar{x})$
(note that $N \geq 3m$, and $N>3m$ for $m>1$). 
 
Consider grounding this relational Bayesian network specification and computing  
\[
\pr{X(\mathtt{1})=x_1,\dots,X(\mathtt{N})=x_N,Y(\mathtt{1})=y_1,\dots,Y(\mathtt{N})=y_N|\mathbf{E}}.
\]
This distribution is encoded by a Bayesian network consisting of nodes
$X(\mathtt{1}),\dots,X(\mathtt{N})$ and nodes $Y(\mathtt{1}),\dots,Y(\mathtt{N})$,
where all nodes $Z(\logvar{x},\logvar{y})$ are removed as they are set by 
$\mathbf{E}$; also, each node $Y(\logvar{x})$ has two parents, and all 
nodes $X(\mathtt{3m+1}),\dots,X(\mathtt{N})$ have no children. 
Denote by $\mathbf{L}$ a generic configuration of $X(\mathtt{1}),\dots,X(\mathtt{3m})$,
and by $\mathbf{Q}$ a configuration of $Y(\mathtt{1}),\dots,Y(\mathtt{N})$ where all
variables are assigned value $1$. 
As in the proof of Theorem \ref{theorem:and-or}, we have
$\pr{\mathbf{L}}=\alpha$ if $\mathbf{L}$ is gratifying-sensible-respectful,
and $\pr{\mathbf{L}}\leq\beta$ if $\mathbf{L}$ is respectful but not gratifying.
If $\#\mbox{(1-in-3)}\phi>k$, then 
$\pr{\mathbf{Q}|\mathbf{E}}=\sum_{\mathbf{L}}\pr{\mathbf{L},\mathbf{Q}} \geq (k+1)\alpha$.
And if $\#\mbox{(1-in-3)}\phi \leq k$, then $\pr{\mathbf{Q}|\mathbf{E}} \leq k\alpha+4^m\beta$. 
Define $\delta_1 = (k+1)\alpha$ and $\delta_2 = k\alpha+4^m\beta$
and choose $\varepsilon<1/(1+4^m)$ to guarantee that $\delta_1 > \delta_2$, 
so that we can differentiate between the two cases with an inference. 

We have thus solved our original problem using a fixed Bayesian network
specification plus a query $(\mathbf{Q},\mathbf{E})$. Hence $\mathsf{PP}$-hardness 
of $\mathsf{QINF}[\mathsf{EL}]$ obtains. However, note that $\mathbf{Q}$
contains only positive assignments, but $\mathbf{E}$
contains both positive and negative assignments. We now constrain ourselves
to positive assignments.

Denote by $\mathbf{E}_1$  the assignments of the form $\{Z(\logvar{x},\logvar{y})=1\}$
in $\mathbf{E}$, and 
denote by $\mathbf{E}_0$  the assignments of the form $\{Z(\logvar{x},\logvar{y})=0\}$
in $\mathbf{E}$.
Consider:
\[
\pr{\mathbf{Q}|\mathbf{E}_1} =
  \pr{\mathbf{Q}|\mathbf{E}_0,\mathbf{E}_1} \pr{\mathbf{E}_0|\mathbf{E}_1} +
   \pr{\mathbf{Q}|\mathbf{E}_0^c,\mathbf{E}_1} \pr{\mathbf{E}_0^c|\mathbf{E}_1},
\]
where $\mathbf{E}_0^c$ is the event consisting of configurations of those variables 
that appear in $\mathbf{E}_0$ such that at least one of these variables is assigned $1$
(of course, such variables are assigned $0$ in $\mathbf{E}_0$).

We have that $\pr{\mathbf{Q}|\mathbf{E}_0,\mathbf{E}_1} = \pr{\mathbf{Q}|\mathbf{E}}$
by definition. And variables in $\mathbf{E}_0$ and $\mathbf{E}_1$ are independent,
hence $\pr{\mathbf{E}_0|\mathbf{E}_1}=\pr{\mathbf{E}_0}=(1-\eta)^M$ where
$M$ is the number of variables in $\mathbf{E}_0$ (so $M \leq N^2$). 
Consequently, $\pr{\mathbf{E}_0^c|\mathbf{E}_1}=1-(1-\eta)^M$. 
Thus we obtain:
\[
\pr{\mathbf{Q}|\mathbf{E}_1} =
  (1-\eta)^M  \pr{\mathbf{Q}|\mathbf{E}} +
   (1-(1-\eta)^M) \pr{\mathbf{Q}|\mathbf{E}_0^c,\mathbf{E}_1}.
\] 

Now reason as follows.
If $\#\mbox{(1-in-3)}\phi>k$, then 
$\pr{\mathbf{Q}|\mathbf{E}_1} \geq  (1-\eta)^M \delta_1$.
And if $\#\mbox{(1-in-3)}\phi \leq k$, then
$\pr{\mathbf{Q}|\mathbf{E}_1} \leq (1- (1-\eta)^M) +  (1-\eta)^M \delta_2$.
To guarantee that $ (1-\eta)^M \delta_1 > (1- (1-\eta)^M) +  (1-\eta)^M \delta_2$,
we must have $(1-\eta)^M   > 1/(1+\delta_1-\delta_2)$. 
We do so by selecting $\eta$ appropriately. 
Note first that $1/(1+\delta_1-\delta_2) \in (0,1)$ by our choice
of $\varepsilon$; note also that $1+(x-1)/M > x^{1/M}$ for any $x\in(0,1)$,
so select
\[
1-\eta > 1+ \left(  \frac{1}{1+\delta_1-\delta_2} - 1 \right)/M;
\]
that is, $\eta < (1-1/(1+\delta_1-\delta_2))/M$. 
By doing so, we can differentiate between the two cases
with an inference, so the desired hardness is proved.

Domain complexity is polynomial because $\mathsf{EL}$ is in
$\mathsf{FFFO}^2$~\cite{Broeck2011,Broeck2014}.
\end{proof}
 

\begin{Theorem} 
Suppose the domain size is specified in unary notation. Then
$\mathsf{DINF}[\mathsf{DLLite^{nf}}]$ is in $\mathsf{P}$; also, 
$\mathsf{INF}[\mathsf{DLLite^{nf}}]$ and $\mathsf{QINF}[\mathsf{DLLite^{nf}}]$
are in $\mathsf{P}$ when the query $(\mathbf{Q},\mathbf{E})$ 
contains only positive assignments. 
\end{Theorem}
\begin{proof}
We prove the polynomial complexity of $\mathsf{INF}[\mathsf{DLLite}]$
with positive queries  by a quadratic-time reduction to multiple
problems of counting weighted edge covers with uniform
weights in a   particular class of graphs. Then we  use the fact
that the latter problem can be solved in
quadratic time (hence the total time is quadratic). 

From now on we simply use $\mathbf{Q}$ to refer to a set of
assignments whose probability is of interest.

We first transform the relational Bayesian network specification
 into an equal-probability
model. Collapse each role $\mathsf{r}$ and its inverse $\mathsf{r}^-$
into a single node $\mathsf{r}$. For each (collapsed) role $\mathsf{r}$,
insert variables $\mathsf{e}_{\mathsf{r}} \equiv \exists \mathsf{r}$ and
$\mathsf{e}^-_{\mathsf{r}} \equiv \exists \mathsf{r}^-$; replace each
appearance of the formula $\exists \mathsf{r}$ by the variable
$\mathsf{e}_{\mathsf{r}}$, and each appearance of $\exists \mathsf{r}^-$
by $\mathsf{e}^-_{\mathsf{r}}$. This transformation does not change the
probability of $\mathbf{Q}$, and it allows us to easily refer to
groundings of formulas $\exists \mathsf{r}$ and $\exists \mathsf{r}^-$
as groundings of $\mathsf{e}_{\mathsf{r}}$ and
$\mathsf{e}_{\mathsf{r}}^-$, respectively.

Observe that only the nodes with assignments in $\mathbf{Q}$ and their
ancestors are relevant for the computation of $\pr{\mathbf{Q}}$, as
every other node in the Bayesian network is barren
\cite{Darwiche2009}. Hence, we can assume without loss of generality
that $\mathbf{Q}$ contains only leaves of the network. If $\mathbf{Q}$
contains only root nodes, then $\pr{\mathbf{Q}}$ can be computed
trivially as the product of marginal probabilities which are readily
available from the specification. Thus assume that $\mathbf{Q}$ assigns
a positive value to at least one leaf grounding
$\mathsf{s}(a)$, where $a$ is some individual in the
domain.  Then by construction $\mathsf{s}(a)$ is associated
with a logical sentence $X_1 \wedge \dotsb \wedge X_k$, where each $X_i$
is either a grounding of non-primitive unary relation in individual
$a$, a grounding of a primitive unary relation in $a$,
or the negation of a grounding of a primitive unary relation in
$a$. It follows that
$\pr{\mathbf{Q}}= \pr{\mathsf{s}(a)=1|X_1=1,\dots,X_k=1}\pr{\mathbf{Q}'}=\pr{\mathbf{Q}'}$,
where $\mathbf{Q}'$ is $\mathbf{Q}$ after removing the assignment
$\mathsf{s}(a)=1$ and adding the assignments
$\{X_1=1,\dots,X_k=1\}$. Now it might be that $\mathbf{Q}'$ contains
both the assignments $\{X_i=1\}$ and $\{X_i=0\}$. Then
$\pr{\mathbf{Q}}=0$ (this can be verified efficiently). So assume there
are no such inconsistencies. The problem of computing $\pr{\mathbf{Q}}$
boils down to computing $\pr{\mathbf{Q}'}$; in the latter problem the
node $\mathsf{s}(a)$ is discarded for being barren. Moreover,
we can replace any assignment $\{\neg \mathsf{r}(a)=1\}$ in
$\mathbf{Q}'$ for some primitive concept $\mathsf{r}$ with the
equivalent assignment $\{\mathsf{r}(a)=0\}$. By repeating this
procedure for all internal nodes which are not groundings of
$\mathsf{e}_{\mathsf{r}}$ or $\mathsf{e}_{\mathsf{r}}^-$, we end up with
a set $\mathbf{A}$ containing positive assignments of groundings of
roles and of concepts $\mathsf{e}_{\mathsf{r}}$ and
$\mathsf{e}_{\mathsf{r}}^-$, and (not necessarily positive) assignments
of groundings of primitive concepts. Each grounding of a primitive
concept or role is (a root node hence) marginally independent from all
other groundings in $\mathbf{A}$; hence
$\pr{\mathbf{A}}=\pr{\mathbf{B}|\mathbf{C}}\prod_{i}\pr{A_i}$, where
each $A_i$ is an assignment to a root node, $\mathbf{B}$ are (positive)
assignments to groundings of concepts $\mathsf{e}_{\mathsf{r}}$ and
$\mathsf{e}_{\mathsf{r}}^-$ for relations $\mathsf{r}$, and
$\mathbf{C} \subseteq \{A_1,A_2,\dots\}$ are groundings of roles (if
$\mathbf{C}$ is empty then assume it expresses a tautology). Since the
marginal probabilities $\pr{A_i}$ are available from the specification
the joint $\prod_i \pr{A_i}$ can be computed in linear time in the
input. We thus focus on computing $\pr{\mathbf{B}|\mathbf{C}}$ as
defined (if $\mathbf{B}$ is empty, we are done). To recap, $\mathbf{B}$
is a set of assignments $\mathsf{e}_{\mathsf{r}}(a)=1$ and
$\mathsf{e}_{\mathsf{r}}^-(b)=1$ and $\mathbf{C}$ is a set of
assignments $\mathsf{r}(c,d)=1$ for arbitrary roles
$\mathsf{r}$ and individuals $a,b,c$ and
$d$.

For a role $\mathsf{r}$, let $\mathcal{D}_{\mathsf{r}}$ be the set of
individuals $a \in \mathcal{D}$ such that
$\mathsf{e}_{\mathsf{r}}(a)=1$ is in $\mathbf{B}$, and let
$\mathcal{D}^-_{\mathsf{r}}$ be the set of individuals
$a \in \mathcal{D}$ such that $\mathbf{B}$ contains
$\mathsf{e}^-_{\mathsf{r}}(a)=1$. Let $\mathsf{gr}(\mathsf{r})$
be the set of all groundings of relation $\mathsf{r}$, and let
$\mathsf{r}_1,\dots,\mathsf{r}_k$ be the roles in the (relational)
network. By the factorization property of Bayesian networks it follows
that 
\vspace{-1ex}
\begin{multline*}
\pr{\mathbf{B}|\mathbf{C}} = \sum_{\mathsf{gr}(\mathsf{r}_1)} \dotsb \sum_{\mathsf{gr}(\mathsf{r}_k)} \prod_{i=1}^k \prod_{a \in \mathcal{D}_{\mathsf{r}_i}} \pr{\mathsf{e}_{\mathsf{r}_i}(a)=1|\pa{\mathsf{e}_{\mathsf{r}_i}(a)},\mathbf{C}} \times \\ \prod_{a \in \mathcal{D}^-_{\mathsf{r}_i}} \pr{\mathsf{e}^-_{\mathsf{r}_i}(a)=1|\pa{\mathsf{e}^-_{\mathsf{r}_i}(a)},\mathbf{C}} \pr{\mathsf{gr}(\mathsf{r}_k)|\mathbf{C}} \, ,
\end{multline*}
\vspace{-1ex}
which by distributing the products over sums is equal to 
\begin{multline*}
\prod_{i=1}^k \! \sum_{\mathsf{gr}(\mathsf{r}_i)} \! \prod_{a \in \mathcal{D}_{\mathsf{r}}} \pr{\mathsf{e}_{\mathsf{r}}(a)\!=\!1|\pa{\mathsf{e}_{\mathsf{r}}(a)},\mathbf{C}}  \times \\
\prod_{a \in \mathcal{D}^-_{\mathsf{r}}} \! \pr{\mathsf{e}^-_{\mathsf{r}}(a)\!=\!1|\pa{\mathsf{e}^-_{\mathsf{r}}(a)},\mathbf{C}}\pr{\mathsf{gr}(\mathsf{r}_k)|\mathbf{C}} \, .
\end{multline*}
Consider an assignment $\mathsf{r}(a,b)=1$ in
$\mathbf{C}$. By construction, the children of the grounding
$\mathsf{r}(a,b)$ are
$\mathsf{e}_{\mathsf{r}}(a)$ and
$\mathsf{e}^-_{\mathsf{r}}(b)$. Moreover, the assignment
$\mathsf{r}(a,b)=1$ implies that
$\pr{\mathsf{e}_{\mathsf{r}}(a)=1|\pa{\mathsf{e}_{\mathsf{r}}(a)},\mathbf{C}}=1$
(for any assignment to the other parents) and
$\pr{\mathsf{e}^-_{\mathsf{r}}(b)=1|\pa{\mathsf{e}_{\mathsf{r}}(a)},\mathbf{C}}=1$
(for any assignment to the other parents). This is equivalent in the
factorization above to removing $\mathsf{r}(a,b)$ from
$\mathbf{C}$ (as it is independent of all other groundings), and
removing individuals $a$ from $\mathcal{D}_{\mathsf{r}}$ and
$b$ from $\mathcal{D}^-_{\mathsf{r}}$. So repeat this procedure
for every grounding in $\mathbf{C}$ until this set is empty (this can be
done in polynomial time). The inference problem becomes one of computing
\[
\gamma(\mathsf{r})=\sum_{\mathsf{gr}(\mathsf{r}_i)} \prod_{a \in \mathcal{D}_{\mathsf{r}}} \pr{\mathsf{e}_{\mathsf{r}}(a)=1|\pa{\mathsf{e}_{\mathsf{r}}(a)}}\prod_{a \in \mathcal{D}^-_{\mathsf{r}}} \pr{\mathsf{e}^-_{\mathsf{r}}(a)=1|\pa{\mathsf{e}^-_{\mathsf{r}}(a)}}\pr{\mathsf{gr}(\mathsf{r}_k)}
\]
for every relation $\mathsf{r}_i$, $i=1,\dots,k$. We will show that
this problem can be reduced to a tractable instance of counting weighted
edge covers.

To this end, consider the graph $G$ whose node set $V$ can be
partitioned into sets
$V_1 = \{ \mathsf{e}^-_{\mathsf{r}}(a): a \in \mathcal{D} \setminus \mathcal{D}^-_{\mathsf{r}} \}$,
$V_2= \{ \mathsf{e}_{\mathsf{r}}(a): a \in\mathcal{D}_{\mathsf{r}}\}$,
$V_3 = \{ \mathsf{e}^-_{\mathsf{r}}(a): a \in\mathcal{D}^-_{\mathsf{r}}\}$,
$V_4 = \{ \mathsf{e}_{\mathsf{r}}(a): a \in\mathcal{D} \setminus \mathcal{D}_{\mathsf{r}}\}$,
and for $i=1,2,3$ the graph obtained by considering nodes
$V_i \cup V_{i+1}$ is bipartite complete.  An edge with endpoints
$\mathsf{e}_{\mathsf{r}}(a)$ and
$\mathsf{e}^-_{\mathsf{r}}(b)$ represents the grounding
$\mathsf{r}(a,b)$; we identify every edge with its
corresponding grounding. We call this graph the \emph{intersection
  graph} of $\mathbf{B}$ with respect to $\mathsf{r}$ and
$\mathcal{D}$. The parents of a node in the graph correspond exactly to
the parents of the node in the Bayesian network.  For example, the graph
in Figure~\ref{fig:graph3-2} represents the assignments
\( \mathbf{B} = \{ \mathsf{e}_{\mathsf{r}}(a)=1 , \mathsf{e}_{\mathsf{r}}(b)=1 , \mathsf{e}_{\mathsf{r}}(d)=1 , \mathsf{e}_{\mathsf{r}}^-(b)=1 , \mathsf{e}_{\mathsf{r}}^-(c)=1 \} \),
with respect to domain
$\mathcal{D}=\{a,b,c,d,\mathsf{e}\}$. The
black nodes (resp., white nodes) represent groundings in (resp., not in)
$\mathbf{B}$. For clarity's sake, we label only a few edges.

\begin{figure}[t]
\begin{center}
  \begin{tikzpicture}
    \tikzstyle{every node}=[minimum size=5pt,inner sep=0,draw,circle]
    \tikzstyle{white}=[fill=white]
    \tikzstyle{black}=[fill=black]

    \begin{scope}
    \node[white,label=left:$\mathsf{e}_{\mathsf{r}}^-(a)$] (a1) at (0,2) {};
    \node[white,label=left:$\mathsf{e}_{\mathsf{r}}^-(d)$] (a2) at (0,1) {};
    \node[white,label=left:$\mathsf{e}_{\mathsf{r}}^-(e)$] (a3) at (0,0) {};
    \node[black,label={[label distance=-2pt]below:$\mathsf{e}_{\mathsf{r}}(a)$}] (b1) at (2,2) {};
    \node[black,label={[label distance=-2pt]below:$\mathsf{e}_{\mathsf{r}}(b)$}] (b2) at (2,1) {};
    \node[black,label={[label distance=-2pt]below:$\mathsf{e}_{\mathsf{r}}(d)$}] (b3) at (2,0) {};
    \node[black,label={[label distance=-2pt]below:$\mathsf{e}_{\mathsf{r}}^-(b)$}] (c2) at (4,1.5) {};
    \node[black,label={[label distance=-2pt]below:$\mathsf{e}_{\mathsf{r}}^-(c)$}] (c3) at (4,0.5) {};
    \node[white,label=right:$\mathsf{e}_{\mathsf{r}}(c)$] (d1) at (6,1.5) {};
    \node[white,label=right:$\mathsf{e}_{\mathsf{r}}(e)$] (d2) at (6,0.5) {};

    \foreach \x in {1,2,3}    {
      \foreach \y in {1,2,3}      {
        \draw (a\x) -- (b\y);
      }    }

    \foreach \x in {2,3}    {
      \foreach \y in {1,2}      {
        \draw (c\x) -- (d\y);
      }    }

    \foreach \x in {1,2,3}    {
      \foreach \y in {2,3}      {
        \draw (b\x) -- (c\y);
      }    }

    \draw[draw=none] (a3) edge
    node[rectangle,sloped,draw=none,anchor=center,fill=white,rounded corners,inner
    sep=0] {$\mathsf{r}(d,e)$} (b3);

    \draw[draw=none] (b3) edge
    node[rectangle,sloped,draw=none,anchor=center,fill=white,rounded corners,inner
    sep=0] {$\mathsf{r}(d,c)$} (c3);

    \draw[draw=none] (c2) edge
    node[rectangle,sloped,draw=none,anchor=center,fill=white,rounded corners,inner
    sep=0] {$\mathsf{r}(c,b)$} (d1);

    \end{scope}
  \end{tikzpicture}
\vspace*{-3ex}
\caption{Representing assignments by graphs.}
\label{fig:graph3-2}
\end{center}
\end{figure}
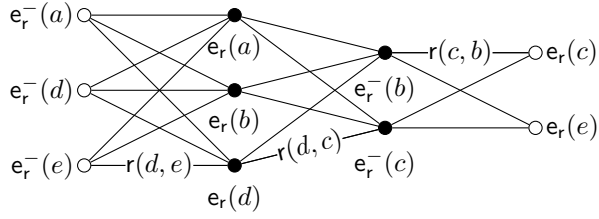

Before showing the equivalence between the inference problem and
counting edges covers, we need to introduce some graph-theoretic notions
and notation.  Consider a (simple, undirected) graph $G=(V,E)$. Denote
by $E_G(u)$ the set of edges incident on a node $u \in V$, and by
$N_G(u)$ the open neighborhood of $u$. For $U \subseteq V$, we say that
$C \subseteq E$ is a $U$-cover if for each node $u \in U$ there is an
edge $e \in C$ incident in $u$ (i.e., $e \in E_G(u)$). For any fixed
real $\lambda$, we say that $\lambda^{|C|}$ is the weight of cover $C$.
The \emph{partition function} of $G$ is
$ Z(G,U,\lambda) = \sum_{C \in EC(G,U)} \lambda^{|C|}$, where
$U \subseteq V$, $EC(G,U)$ is the set of $U$-covers of $G$ and $\lambda$
is a positive real.  If $\lambda=1$ and $U=V$, the partition function is
the number of edge covers. The following result connects counting edge
covers to marginal inference in DL-Lite Bayesian networks.

\begin{Lemma} \label{lemma:intersection-graph} 
Let
  $G=(V_1,V_2,V_3,V_4,E)$ be the intersection graph of $\mathbf{B}$ with
  respect to a relation $\mathsf{r}$ and domain $\mathcal{D}$. Then
  $\gamma(\mathsf{r}) = Z(G,V_2 \cup V_3, \alpha/(1-\alpha))/(1-\alpha)^{|E|}$,
  where $\alpha=\pr{\mathsf{r}(\logvar{x},\logvar{y})}$.
\end{Lemma}
\begin{proof}[Proof of Lemma \ref{lemma:intersection-graph}]
  Let $B=V_2 \cup V_3$, and consider a $B$-cover $C$. The assignment
  that sets to true all groundings $\mathsf{r}(a,b)$
  corresponding to edges in $C$, and sets to false the remaining
  groundings of $\mathsf{r}$ makes
  $\pr{\mathsf{e}_{\mathsf{r}}(a)=1|\pa{\mathsf{e}_{\mathsf{r}}(a)}}=\pr{\mathsf{e}^-_{\mathsf{r}}(b)=1|\pa{\mathsf{e}^-_{\mathsf{r}}(b)}}=1$
  for every $a \in \mathcal{D}_{\mathsf{r}}$ and
  $b \in \mathcal{D}^-_{\mathsf{r}}$; it makes
  $\pr{\mathsf{gr}(\mathsf{r})}=\pr{\mathsf{r}}^{|C|}(1-\pr{\mathsf{r}})^{|E|-|C|}=(1-\alpha)^{|E|}\alpha^{|C|}/(1-\alpha)^{|C|}$,
  which is the weight of the cover $C$ scaled by $(1-\alpha)^{|E|}$. Now
  consider a set of edges $C$ which is not a $B$-cover and obtains an
  assignment to groundings $\mathsf{gr}(\mathsf{r})$ as before. There is
  at least one node in $B$ that does not contain any incident edges in
  $C$. Assume that node is $\mathsf{e}(a)$; then all parents of
  $\mathsf{e}(a)$ are assigned false, which implies that
  $\pr{\mathsf{e}_{\mathsf{r}}(a)=1|\pa{\mathsf{e}_{\mathsf{r}}(a)}}=0$. The
  same is true if the node not covered is a grounding
  $\mathsf{e}^-(a)$. Hence, for each $B$-cover $C$ the
  probability of the corresponding assignment equals its weight up to
  the factor $(1-\alpha)^{|E|}$. And for each edge set $C$ which is not
  a $B$-cover its corresponding assignment has probability zero.
$\Box$
\end{proof}

We have thus established that, if a particular class of edge cover
counting problems is polynomial, then marginal inference in DL-Lite
Bayesian networks is also polynomial. Because the former is shown to be
true in \ref{appendix:Counting}, this concludes the proof of
Theorem \ref{theorem:DLLITE}.
\end{proof}

 
\begin{Theorem}  
Given a relational Bayesian network $\mathbb{S}$ based on $\mathsf{DLLite^{nf}}$,
a set of positive assignments to grounded relations $\mathbf{E}$,
and a domain size $N$ in unary notation,  
$\mathsf{MLE}(\mathbb{S},\mathbf{E},N)$ can be solved in polynomial time.
\end{Theorem} 
\begin{proof}
In this theorem we are interested in finding an assignment $\mathbf{X}$
to all groundings that maximizes $\pr{\mathbf{X} \wedge \mathbf{E}}$,
where $\mathbf{E}$ is a set of positive assignments. Perform the
substitution of formulas $\exists\mathsf{r}$ and $\exists\mathsf{r}^-$
by logically equivalent concepts $\mathsf{e}_{\mathsf{r}}$ and
$\mathsf{e}_{\mathsf{r}}^-$ as before. Consider a non-root grounding
$\mathsf{s}(a)$ in $\mathbf{E}$ which is not the grounding of
$\mathsf{e}_{\mathsf{r}}$ or $\mathsf{e}_{\mathsf{r}}^-$; by
construction, $\mathsf{s}(a)$ is logically equivalent to a
conjunction $X_1 \wedge \dotsb \wedge X_k$, where $X_1,\dots,X_k$ are
unary groundings. Because $\mathsf{s}(a)$ is assigned to true,
any assignment $\mathbf{X}$ with nonzero probability assigns
$X_1,\dots,X_k$ to true. Moreover, since $\mathsf{s}(a)$ is an
internal node, its corresponding probability is one. Hence, if we
include all the assignments $X_i=1$ to its parents in $\mathbf{E}$, the
MPE value does not change. As in the computation of inference, we might
generate an inconsistency when setting the values of parents; in this
case halt and return zero (and an arbitrary assignment). So assume we
repeated this procedure until $\mathbf{E}$ contains all ancestors of the
original groundings which are groundings of unary relations, and that no
inconsistency was found. Note that at this point we only need to assign
values to nodes which are either not ancestors of any node in the
original set $\mathbf{E}$, and to groundings of (collapsed) roles
$\mathsf{r}$.

Consider the groundings of primitive concepts $\mathsf{r}$ which are not
ancestors of any grounding in $\mathbf{E}$. Setting its value to
maximize its marginal probability does not introduce any inconsistency
with respect to $\mathbf{E}$. Moreover, for any assignment to these
groundings, we can find a consistent assignment to the remaining
groundings (which are internal nodes and not ancestors of $\mathbf{E}$),
that is, an assignment which assigns positive probability. Since this is
the maximum probability we can obtain for these groundings, this is a
partial optimum assignment.

We are thus only left with the problem of assigning values to the
groundings of relations $\mathsf{r}$ which are ancestors of
$\mathbf{E}$. Consider a relation $\mathsf{r}$ such that
$\pr{\mathsf{r}}\geq 1/2$. Then assigning all groundings of $\mathsf{r}$
to true maximizes their marginal probability and satisfies the logical
equivalences of all groundings in $\mathbf{E}$. Hence, this is a maximum
assignment (and its value can be computed efficiently). So assume there
is a relation $\mathsf{r}$ with $\pr{\mathsf{r}} < 1/2$ such that a
grounding of $\mathsf{e}_{\mathsf{r}}$ or $\mathsf{e}_{\mathsf{r}}^-$
appear in $\mathbf{E}$. In this case, the greedy assignment sets every
grounding of $\mathsf{r}$; however, such an assignment is inconsistent
with the logical equivalence of $\mathsf{e}_{\mathsf{r}}$ and
$\mathsf{e}^-_{\mathsf{r}}$, hence obtains probability zero. Now
consider an assignment that assigns exactly one grounding
$\mathsf{r}(a,b)$ to true and all the other to
false. This assignment is consistent with
$\mathsf{e}_{\mathsf{r}}(a)$ and
$\mathsf{e}_{\mathsf{r}}(b)$, and maximizes the probability;
any assignment that sets more groundings to true has a lower probability
since it replaces a term $1-\pr{\mathsf{r}} \geq 1/2$ with a term
$\pr{\mathsf{r}} < 1/2$ in the joint probability. More generally, to
maximize the joint probability we need to assign to true as few
groundings $\mathsf{r}(a,b)$ which are ancestors of
$\mathbf{E}$ as possible. This is equivalent to a minimum cardinality
edge covering problem as follows.

For every relation $\mathsf{r}$ in the relational network, construct the
bipartite complete graph $G_{\mathsf{r}}=(V_1,V_2,E)$ such that $V_1$ is
the set of groundings $\mathsf{e}_{\mathsf{r}}(a)$ that appears 
and have no parent $\mathsf{r}(a,b)$ in $\mathbf{E}$,
and $V_2$ is the set of groundings
$\mathsf{e}^-_{\mathsf{r}}(a)$ that appears and have no parents
in $\mathbf{E}$. We identify an edge connecting
$\mathsf{e}_{\mathsf{r}}(a)$ and
$\mathsf{e}^-_{\mathsf{r}}(b)$ with the grounding
$\mathsf{r}(a,b)$. For any set $C \subseteq E$,
construct an assignment by attaching true to the groundings
$\mathsf{r}(a,b)$ in $C$ and false to every other
grounding $\mathsf{r}(a,b)$. This assignment is
consistent with $\mathbf{E}$ if and only if $C$ is an edge cover; hence
the minimum cardinality edge cover maximizes the joint probability (it
is consistent with $\mathbf{E}$ and attaches true to the least number of
groundings of $\mathbf{r}$). This concludes the proof of Theorem
\ref{mpe}.
\end{proof}


\begin{Theorem}
$\mathsf{INF}[\mathsf{PLATE}]$ and $\mathsf{QINF}[\mathsf{PLATE}]$ are 
$\mathsf{PP}$-complete with respect
to many-one reductions, and $\mathsf{DINF}[\mathsf{PLATE}]$ requires constant
computational effort. These results hold {\em even if the domain size is given in
binary notation}. 
\end{Theorem}
\begin{proof} 
Consider first $\mathsf{INF}[\mathsf{PLATES}]$. 
To prove membership, take a plate model with relations $X_1,\dots,X_n$.
Suppose we ground this specification on a domain of size $N$. 
To compute $\pr{\mathbf{Q}|\mathbf{E}}$, the only relevant  groundings are the 
ones that are ancestors of each of the ground atoms
in $\mathbf{Q}\cup\mathbf{E}$.  Our strategy will be to bound the number
of such relevant groundings. To do that,  take a grounding  $X_i(a_1,\dots,a_{k_i})$  
in $\mathbf{Q}\cup\mathbf{E}$, and suppose that $X_i$ is not a root node. 
Each parent $X_j$ of $X_i$ may appear once in the definition axiom related to $X_i$. 
And each parent of these parents will again have a limited number of parent
groundings; in the end there are at most $(n-1)$ relevant groundings that 
are ancestors of $X_i(a_1,\dots,a_{k_i})$. 
We can take the union of all groundings that are ancestors of groundings
of $\mathbf{Q}\cup\mathbf{E}$, and the number of such groundings is still
polynomial in the size of the input. Thus in polynomial time we can build 
a polynomially-large
Bayesian network that is a fragment of the grounded Bayesian network. 
Then we can run a Bayesian network inference in this smaller network
(an effort within $\mathsf{PP}$); note that domain size is actually not important so it
can be specified either in unary or binary notation.
To prove hardness, note that $\mathsf{INF}[\mathsf{Prop}(\wedge,\neg)]$
is $\mathsf{PP}$-hard, and a propositional specification can be reproduced
within $\mathsf{PLATES}$. 

Now consider $\mathsf{QINF}[\mathsf{PLATES}]$. 
First, to prove membership, note that even $\mathsf{INF}[\mathsf{PLATES}]$ is in $\mathsf{PP}$.
To prove hardness, reproduce the proof of Theorem \ref{theorem:QF} by encoding
a $\#3\mathsf{SAT}(>)$ problem, specified by sentence $\phi$ and integer $k$, with the definition axioms:
\begin{eqnarray*}
\mathsf{clause0}(\logvar{x},\logvar{y},\logvar{z}) & \definitionaxiom & 
 \neg \mathsf{left}(\logvar{x}) \vee \neg \mathsf{middle}(\logvar{y}) \vee \neg \mathsf{right}(\logvar{z}), \\
\mathsf{clause1}(\logvar{x},\logvar{y},\logvar{z}) & \definitionaxiom & 
 \neg \mathsf{left}(\logvar{x}) \vee \neg \mathsf{middle}(\logvar{y}) \vee \mathsf{right}(\logvar{z}), \\
\mathsf{clause2}(\logvar{x},\logvar{y},\logvar{z}) & \definitionaxiom & 
 \neg \mathsf{left}(\logvar{x}) \vee  \mathsf{middle}(\logvar{y}) \vee \neg \mathsf{right}(\logvar{z}), \\
\vdots &  \vdots & \vdots \\
\mathsf{clause7}(\logvar{x},\logvar{y},\logvar{z}) & \definitionaxiom & 
  \mathsf{left}(\logvar{x}) \vee  \mathsf{middle}(\logvar{y}) \vee  \mathsf{right}(\logvar{z}),  \\
\mathsf{equal}(\logvar{x},\logvar{y},\logvar{z}) & \definitionaxiom &
  \mathsf{left}(\logvar{x}) \leftrightarrow \mathsf{middle}(\logvar{y}) \leftrightarrow \mathsf{right}(\logvar{z}),
\end{eqnarray*}
and 
$\pr{\mathsf{left}(\logvar{x})=1} = \pr{\mathsf{middle}(\logvar{x})=1} 
   = \pr{\mathsf{right}(\logvar{x})=1} 1/2$. 
The resulting plate model is depicted in Figure \ref{figure:DataComplexity}. 
The query is again just a set of assignments $\mathbf{Q}$ ($\mathbf{E}$ is empty)
containing an assignment per clause. If a clause is $\neg A_2 \vee A_3 \vee \neg A_1$,
then take the corresponding assignment $\{ \mathsf{clause2}(a_2,a_3,a_1) = 1 \}$, and so on. 
Moreover, add the assignments $\{ \mathsf{equal}(a_i,a_i,a_i)=1 \}$ for each $i \in \{1,\dots,n\}$,
to guarantee that $\mathsf{left}$, $\mathsf{middle}$ and $\mathsf{right}$ have identical
truth assignments for all elements of the domain. 
The $\#3\mathsf{SAT}(>)$ is solved by deciding whether $\pr{\mathbf{Q}}>k/2^n$ with 
domain of size $n$; hence the desired hardness is proved.
  
\begin{figure}
\begin{center}
\begin{tikzpicture}[thick,->, >=latex,scale=1.1]
\draw[red,ultra thick] (0,0.2) rectangle (8,3);
\draw[red,ultra thick] (2,0.1) rectangle (10,2.9);
\draw[red,ultra thick] (2.1,0) rectangle (7.9,4);
\node[rectangle,rounded corners,draw] (L) at (1,1.5) {$\mathsf{left}(\logvar{x})$};
\node[rectangle,rounded corners,draw] (M) at (6.6,3.5) {$\mathsf{middle}(\logvar{y})$};
\node[rectangle,rounded corners,draw] (R) at (9,1.5) {$\mathsf{right}(\logvar{z})$};
\node[rectangle,rounded corners,draw] (E) at (3.8,2.4) {$\mathsf{equal}(\logvar{x},\logvar{y},\logvar{z})$};
\node[rectangle,rounded corners,draw] (O1) at (3.55,1.5) {$\mathsf{clause0}(\logvar{x},\logvar{y},\logvar{z})$};
\node[rectangle,rounded corners,draw] (O8) at (6.55,0.7) {$\mathsf{clause7}(\logvar{x},\logvar{y},\logvar{z})$};
\node at (5.05,1.2) {$\ddots$};
\node at (0.2,2.8) {$\logvar{x}$};
\node at (9.8,2.75) {$\logvar{z}$};
\node at (2.3,3.75) {$\logvar{y}$};
\draw (L) edge[out=50,in=180] (E);
\draw (L)--(O1);
\draw (L) edge[out=-25,in=180] (O8);
\draw (R) edge[out=155,in=0] (E);
\draw (R)--(O1);
\draw (R) edge[out=-130,in=0] (O8);
\draw (M)--(E);
\draw (M) edge[out=-100,in=10] (O1);
\draw (M)--(O8);
\end{tikzpicture}
\end{center}
\vspace*{-2ex}
\caption{A plate model that decides  a $\#3\mathsf{SAT}(>)$ problem.}
\label{figure:DataComplexity}
\end{figure}
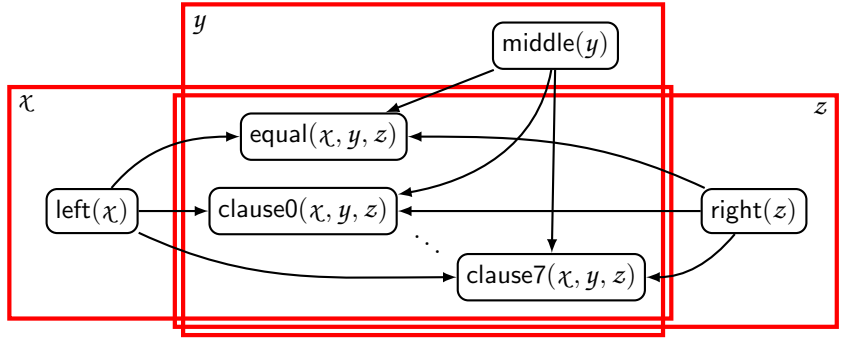

And $\mathsf{DINF}[\mathsf{PLATES}]$ requires constant effort: in fact, domain
size is not relevant to a fixed inference, as can be seen from the proof of
inferential complexity above. 
\end{proof}


\begin{Theorem} 
Consider the class of functions that gets as input a relational Bayesian network
specification based on $\mathsf{FFFO}$, a domain size $N$ (in binary or unary
notation), and a set of assignments $\mathbf{Q}$, and returns 
$\pr{\mathbf{Q}}$. This class of functions is $\#\mathsf{EXP}$-equivalent.
\end{Theorem}
\begin{proof}
Build a relational Bayesian network specification as in the proof
of Theorem \ref{theorem:INF-FFFO}.
Note that the $p = \pr{\mathbf{E} \wedge \bigwedge_{i=1}^6 Z_i}$ is 
the probability that a tiling is built satisfying all horizontal and
vertical restrictions and the initial condition, and moreover containing
the accepting state $q_a$. 

If we can recover the
number of tilings of the torus from this probability, we obtain the
number of accepting computations of the exponentially-bounded Turing
machine we started with.
Assume we have $p$. There are $2^{2n}$ elements in
our domain; if the plate model is grounded, there are 
$2^{2n}(2n+c)$ grounded root random variables, hence there are $2^{2^{2n}(2n+c)}$
interpretations. Hence $p \times 2^{2^{2n}(2n+c)}$ is
the number of truth assignments that build the board satisfying all
horizontal and vertical constraints and the initial conditions. However, this 
number is {\em not}
equal to the number of tilings of the board. To see this, consider the
grounded Bayesian network where each $a$ in the domain is associated
with a ``slice'' containing groundings $X_i(a)$, $Y_i(a)$, $C_j(a)$ and
so on.  If a particular configuration of these indicator variables
corresponds to a tiling, then we can produce the same tiling by
permuting all elements of the domain with respect to the slices of the
network. Intuitively, we can fix a tiling and imagine that we are
labelling each point of the torus with an element of the domain; clearly
every permutation of these labels produces the same tiling (this
intuition is appropriate because each $a$ corresponds to a different
point in the torus). So, in order to produce the number of tilings of the
torus, we must compute $p \times 2^{2^{2n}(2n+c)}/(2^{2n}!)$,
where we divide the number of satisfying truth assignments by
the number of repeated tilings.  
\end{proof}


\begin{Theorem}  
Consider the class of functions that gets as input a relational Bayesian network
specification based on $\mathsf{FFFO}$   with relations with bounded arity, 
a domain size $N$ in unary
notation, and a set of assignments $\mathbf{Q}$, and returns 
$\pr{\mathbf{Q}}$. This class of functions is $\natural\mathsf{PSPACE}$-equivalent.
\end{Theorem}
\begin{proof}
First we describe a counting Turing machine that produces a count proportional
to $\pr{\mathbf{Q}}$ using a polynomial number of nondeterministic guesses. 
This nondeterministic machine guesses a truth assignment
for each one of the polynomially-many grounded root nodes (and writes the guess
in the working tape). Note that each grounded root node $X$ is associated with
an assessment $\pr{X=1}=c/d$, where $c$ and $d$ are integers. The machine 
must replicate its computation paths to handle such rational assessments exactly
as in the proof of Theorem \ref{theorem:QINF-FFFO-unary}. The machine then verifies, 
in each computation path, whether the guessed truth assignment satisfies $\mathbf{Q}$;
if it does, then accept; if not, then reject. Denote by $R$ the number of grounded 
root nodes and by $\#A$ the number of accepting paths of this machine;
then $\pr{\mathbf{Q}}=\#A/2^R$. 

Now we show that $\mathbf{Q}$ is $\natural\mathsf{PSPACE}$-hard with respect
to weighted reductions. Define $\varphi(\logvar{x}_1,\dots,\logvar{x}_m)$ to be a
quantified Boolean formula with free logvars $\logvar{x}_1,\dots,\logvar{x}_m$:
\[
\forall \logvar{y}_1 : Q_2 \logvar{y}_2 : \dots Q_M \logvar{x}_M : 
		\phi(\logvar{x}_1,\dots,\logvar{x}_m),
\]
where each logvar can only be $\mathsf{true}$ or $\mathsf{false}$, each $Q_j$
is a quantifier (either $\forall$ or $\exists$). And define $\#\varphi$ to be the number 
of instances of $\logvar{x}_1,\dots,\logvar{x}_m$ such that 
$\varphi(\logvar{x}_1,\dots,\logvar{x}_m)$ is $\mathsf{true}$. Denote by $\natural\mathsf{QBF}$ 
the function that gets a formula  $\varphi(\logvar{x}_1,\dots,\logvar{x}_m)$ and returns 
$\#\varphi$; Ladner shows that $\natural\mathsf{QBF}$ is $\natural\mathsf{PSPACE}$-complete
\cite[Theorem 5(2)]{Ladner89}.  So, adapt the hardness proof of Theorem \ref{theorem:INF-FFFO-PSPACE}:
introduce the definition axiom  
\[
Y \definitionaxiom \forall \logvar{y}_1 : \dots Q_m \logvar{y}_m : \phi'(X_1,\dots,X_m),
\]
where $\phi'$ has the same structure of $\phi$ but logvars are replaced as follows.
First, each $\logvar{x}_j$ is replaced
by a relation $X_j$ of arity zero (that is, a proposition).
Second, each logvar $\logvar{y}_j$ is replaced by the atom
$X(\logvar{y}_j)$ where $X$ is a fresh unary relation.
These relations are associated with assessments $\pr{X_j=1}=1/2$
and $\pr{X(\logvar{x})=1}=1/2$. This completes the relational Bayesian network specification.
Now for domain $\{\mathtt{0},\mathtt{1}\}$, first compute
$\pr{\mathbf{Q}}$ for $\mathbf{Q} = \{Y=1,X(\mathtt{0})=0,X(\mathtt{1})=1\}$ and then 
compute $2^m (\pr{\mathbf{Q}}/(1/4))$. The latter number is the desired value of
$\natural\mathsf{QBF}$; note that $\pr{\mathbf{Q}}/(1/4) = \pr{Y=1|X(\mathtt{0})=0,X(\mathtt{1})=1}$.
\end{proof}


\begin{Theorem} 
Consider the class of functions that gets as input a relational Bayesian network
specification based on $\mathsf{FFFO}^k$ for $k \geq 2$, a domain size $N$ in unary
notation, and a set of assignments $\mathbf{Q}$, and returns 
$\pr{\mathbf{Q}}$. This class of functions is $\#\mathsf{P}$-equivalent.
\end{Theorem}
\begin{proof}
Hardness is trivial: even $\mathsf{Prop}(\wedge,\neg)$ is $\#\mathsf{P}$-equivalent,
as  $\mathsf{Prop}(\wedge,\neg)$ suffice to specify any propositional Bayesian
network, and equivalence then obtains~\cite{Roth96}. To prove membership,
use the Turing machine described in the proof of membership in Theorem
\ref{theorem:INF-FFFO-k} without assignments $\mathbf{E}$ (that is, the machine
only processes $\mathbf{Q}$) and without Park's construction. At the end the
machine produces the number $\#A$ of computation paths that satisfy $\mathbf{Q}$;
then return $\#A/2^R$, where $R$ is the number of grounded root nodes.
\end{proof}
 

\begin{Theorem} 
Consider the class of functions that get as input a plate model
 based on $\mathsf{FFFO}$, a domain size $N$ in unary
notation, and a set of assignments $\mathbf{Q}$, and returns 
$\pr{\mathbf{Q}}$. This class of functions is $\#\mathsf{P}$-equivalent.
\end{Theorem}
\begin{proof}
Hardness is trivial: a propositional Bayesian network can be encoded with a plate model. 
To prove membership, build the same fragment of the grounded Bayesian network
as described in the proof of Theorem \ref{theorem:Plates}: inference with the plate
model is then reduced to inference with this polynomially large Bayesian network.
\end{proof}

\section{A tractable class of model counting problems}
\label{appendix:Counting}

``Model counting'' usually refers to the problem of counting the number
of satisfying truth-value assignments of a given Boolean formula.  Many
problems in artificial intelligence and combinatorial optimization can
be either specialized to or generalized from model counting. For
instance, propositional satisfiability (i.e., the problem of deciding
whether a satisfying truth-value assignment exists) is a special case of
model counting; probabilistic reasoning in graphical models such as
Bayesian networks can be reduced to a weighted variant of model counting
\cite{Bacchus2009,Darwiche2009}; validity of conformal plans can be
formulated as model counting \cite{Palacios2005}. Thus, characterizing the
theoretical complexity of the problem is both of practical and
theoretical interest.


In unrestricted form, the problem is complete for the class $\#\mathsf{P}$ (with
respect to various reductions). Even very restrictive
versions of the problem are complete for $\#\mathsf{P}$. For example, the problem
is $\#\mathsf{P}$-complete even when the formulas are in conjunctive normal form
with two variables per clause, there is no negation, and the variables
can be partitioned into two sets such that no clause contains two
variables in the same block \cite{ProvanBall83}. The problem is also
$\#\mathsf{P}$-complete when the formula is monotone and each variable appears at
most twice, or when the formula is monotone, the clauses contain two
variables and each variables appears at most $k$ times for any
$k \geq 5$ \cite{Vadhan2001}. A few tractable classes have been found: for
example, Roth \cite{Roth96} developed an algorithm for counting the
number of satisfying assignments of formulas in conjunctive normal form
with two variables per clause, each variable appearing in at most two
clauses. Relaxing the constraint on the number of variables per clauses
takes us back to intractability: model counting restricted to formulas
in conjunctive normal form with variables appearing in at most two
clauses is $\#\mathsf{P}$-complete \cite{Bubley97}.

Researchers have also investigated the complexity with respect to the
graphical representation of formulas. Computing the number of satisfying
assingments for monotone formulas in conjunctive normal form, with at
most two variables per clause, with each variable appearing at most four
times is $\#\mathsf{P}$-complete even when the primal graph (where nodes are
variables and an edge connects variables that coappear in a clause) is
bipartite and planar \cite{Vadhan2001}. The problem is also $\#\mathsf{P}$-complete
for monotone conjunctive normal form formulas whose primal graph is
3-regular, bipartite and planar. In fact, even deciding whether the
number of satisfying assignments is even (i.e., counting \emph{modulo
  two}) in conjunctive normal form formulas where each variable appears
at most twice, each clause has at most three variables, and the
incidence graph (where nodes are variables and clauses, and edges connect
variables appearing in clauses) of the formula is planar is known to be
$\mathsf{NP}$-hard by a randomized reduction \cite{Xia2006}. Interestingly,
counting the number of satisfying assignments \emph{modulo seven} (!) of
that same class of formulas is polynomial-time computable
\cite{Valiant2006}.

In this appendix, we present another class of tractable model counting
problems defined by its graphical representation. In particular, we
develop a polynomial-time algorithm for formulas in monotone conjunctive
normal form whose clauses can be partitioned into two sets such that (i)
any two clauses in the same set have the same number of variables which
are not shared between them, and (ii) any two clauses in different sets
share exactly one variable. These formulas lead to intersection graphs
(where nodes are clauses, and edges connect clauses which share
variables) which are bipartite complete. We state our result in the
language of edge coverings; the use of a graph problem makes
communication easier with no loss of generality. 

The basics of model
counting and the particular class of problems we
consider are presented in~\ref{modelcounting}. 
We then examine the problem of counting edge covers in black-and-white 
graphs in~\ref{edgecovers}, and describe a polynomial-time algorithm for 
counting edge covers of a certain class of
black-and-white graphs in~\ref{algorithm1}. Restrictions are removed 
in~\ref{algorithm2}, and we comment on possible extensions of the
algorithms in~\ref{extension}.

\subsection{Model counting: some needed concepts} \label{modelcounting}

Say that two clauses do not intersect if the variables in one clause do
not appear in the other. If $X$ is the largest set of variables that
appear in two clauses, we say that the clauses intersect (at $X$). For
instance, the clauses $X_1 \vee X_2 \vee X_3$ and
$\neg X_2 \vee \neg X_4$ intersect at $\{X_2\}$. A clause containing $k$
variables is called a $k$-clause, and $k$ is called the size of the
clause.  
The degree
of a variable in a CNF formula is the number of clauses in which either
the variable or its negation appears. A CNF formula where every variable
has degree at most two is said {\em read-twice}. If any two clauses
intersect in at most one variable, the formula is said {\em linear}. The
formula $(X_1 \vee X_2) \wedge (\neg X_1 \vee \neg X_3)$ is a linear
read-twice $2$CNF containing two $2$-clauses that intersect at 
$X_1$. The degree of $X_1$ is two, while the degree of either $X_2$ or
$X_3$ is one. To recap, a formula is \emph{monotone} if no variable appears
negated, such as in $X_1 \vee X_2$.

We can graphically represent the dependencies between variables and
clauses in a CNF formula in many ways. The \emph{incidence graph} of a
CNF formula is the bipartite graph with variable-nodes and
clause-nodes. The variable-nodes correspond to variables of the formula,
while the clause-nodes correspond to clauses. An edge is drawn between a
variable-node and a clause-node if and only if the respective variable
appears in the respective clause. The \emph{primal graph} of a CNF
formula is a graph whose nodes are variables and edges connect variables
that co-appear in some clause. The primal graph can be obtained from the
incidence graph by deleting clause-nodes (along with their edges) and
pairwise connecting their neighbors. The \emph{intersection graph} of a
CNF formula is the graph whose nodes correspond to clauses, and an edge
connects two nodes if and only if the corresponding clauses
intersect. 
The intersection graph can be obtained from the incidence graph by
deleting variable-nodes and pairwise connecting their
neighbors. Figure~\ref{fig:graphs} shows examples of graphical
illustrations of a Boolean formula. We represent clauses as rectangles
and variables as circles.

\begin{figure}
  \centering
  \begin{tikzpicture}[very thick]

    \begin{scope}
      \tikzstyle{clause}=[draw=black,rectangle,rounded corners,minimum size=4pt,inner sep=4pt]
      \tikzstyle{var}=[draw=black,circle,minimum size=20pt,inner sep=0]

      \node[clause] (c1) at (0,-2) {$\phi_1=X_1 \vee \neg X_2$};
      \node[clause] (c2) at (0,0) {$\phi_2=\neg X_1 \vee X_2$};
      \node[clause] (c3) at (4,0) {$\phi_3=X_2 \vee \neg X_3$};
      \node[clause] (c4) at (4,-2) {$\phi_4=\neg X_2 \vee X_3$};

      \node[var] (x1) at (-2,-1) {$X_1$};
      \node[var] (x2) at (2,-1) {$X_2$};
      \node[var] (x3) at (6,-1) {$X_3$};

      \draw (x1) -- (c2);
      \draw (x2) -- (c2);
      \draw (x2) -- (c3);
      \draw (x3) -- (c4);
      \draw (x1) -- (c1);
      \draw (x3) -- (c3);
      \draw (x2) -- (c1);
      \draw (x2) -- (c4);
    \end{scope}

    \begin{scope}[yshift=-3.5cm]
      \tikzstyle{every node}=[draw=black,circle,minimum size=20pt,inner sep=0]

      \node (x1) at (-2,0) {$X_1$};
      \node (x2) at (2,0) {$X_2$};
      \node (x3) at (6,0) {$X_3$};
      \draw (x1) -- (x2);
      \draw (x2) -- (x3);

    \end{scope}

    \begin{scope}[yshift=-5cm]

      \tikzstyle{every node}=[draw=black,rectangle,rounded corners,minimum size=4pt,inner sep=4pt]

      \node (c1) at (0,-2) {$\phi_1=X_1 \vee \neg X_2$};
      \node (c2) at (0,0) {$\phi_2=\neg X_1 \vee X_2$};
      \node (c3) at (4,0) {$\phi_3=X_2 \vee \neg X_3$};
      \node (c4) at (4,-2) {$\phi_4=\neg X_2 \vee X_3$};

      \draw (c1) -- (c2);
      \draw (c2) -- (c3);
      \draw (c1) -- (c3);
      \draw (c1) -- (c4);
      \draw (c2) -- (c4);
      \draw (c3) -- (c4);


    \end{scope}

  \end{tikzpicture}
  \caption{Graphical illustrations of the formula
    $(X_1 \vee \neg X_2) \wedge (\neg X_1 \vee X_2) \wedge (X_2 \vee \neg X_3) \wedge (\neg X_2 \vee X_3)$. Top:
    incidence graph. Middle: primal graph. Bottom: intersection graph.}
  \label{fig:graphs}
\end{figure}
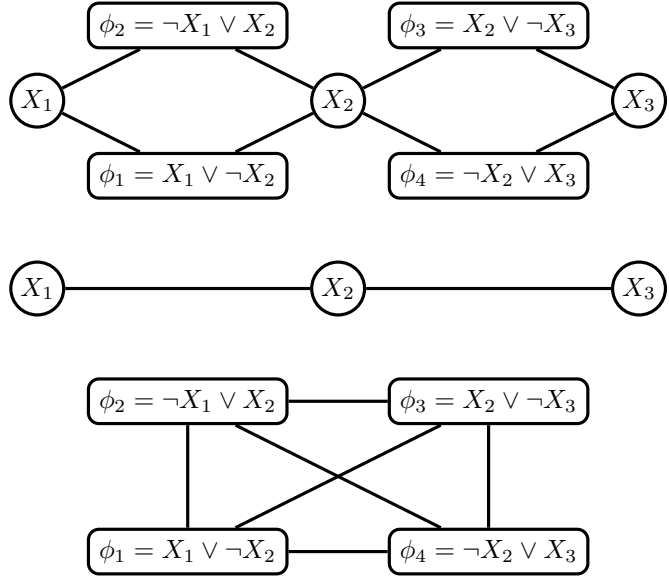

A CNF formula $\phi$ is
satisfied by an assignment $\sigma$ (written $\sigma \models \phi$) if
each clause contains either a nonnegated variable $X_i$ such that
$\sigma(X_i)=1$ or a negated variable $X_j$ such that
$\sigma(X_j)=0$. In this case, we say that $\sigma$ is a model of
$\phi$. For monotone CNF formulas, this condition simplifies to the
existence of a variable $X_i$ in each clause for which
$\sigma(X_i)=1$. Hence, monotone formulas are always satisfiable (by the
trivial model that assigns every variable the value one). The
\emph{model count} of a formula $\phi$ is the number
$Z(\phi)= |\{ \sigma: \sigma \vDash \phi\}|$ of models of the
formula. The \emph{model counting problem} is to compute the model count
of a given CNF formula $\phi$.


In this appendix, we consider linear monotone CNF formulas whose
intersection graph is bipartite complete, and such that all clauses in
the same part have the same size. These assumptions imply that each
variable appears in at most two clauses (hence the formula is
read-twice). We call CNF formulas satisfying all of these assumptions 
linear monotone clause-bipartite complete (LinMonCBPC) formulas.  Under
these assumptions, we show that model counting can be performed in
quadratic time in the size of the input. It is our hope that in future
work some of these assumptions can be relaxed. However, due to the
results mentioned previously, we do not expect that much can be
relaxed without moving to $\#\mathsf{P}$-completeness.

The set of model counting problems generated by LinMonCBPC formulas is
equivalent to the following problem. Take integers $m,n,M,N$ such that
$N>n>0$ and $M>m>0$, and compute how many $\{0,1\}$-valued matrices of
size $M$-by-$N$ exist such that (i) each of the first $m$ rows has at
least one cell with value one, and (ii) each of the first $n$ columns
has at least one cell with value one. Call $A_{ij}$ the value of the
$i$th row, $j$th column. The problem is equivalent to computing the
number of matrices $A_{M\times N}$ with $\sum_{j=1}^N A_{ij} > 0$, for
$i=1,\dotsc,m$, and $\sum_{i=1}^M A_{ij} > 0$, for $j=1,\dotsc,n$. This
problem can be encoded as the model count of the CNF formula whose
clauses are
\begin{gather*}
   A_{11} \vee A_{12} \vee \dotsb \vee A_{1n} \vee \dotsb \vee A_{1N} ,\\
   A_{21} \vee A_{22} \vee \dotsb \vee A_{2n} \vee \dotsb \vee A_{2N},\\
   \vdots \\
   A_{m1} \vee A_{n2} \vee \dotsb \vee A_{mn} \vee \dotsb \vee A_{1N},\\
   A_{11} \vee A_{21} \vee \dotsb \vee A_{m1} \vee \dotsb \vee A_{M1},\\
   \vdots \\
   A_{1n} \vee A_{2n} \vee \dotsb \vee A_{mn} \vee \dotsb \vee A_{MN}.
\end{gather*}
The first $m$ clauses are the row constraints, while the last $n$
clauses are the columns constraints. The row constraints have size $n$,
and the column constraints have size $m$. The $i$th row constraint
intersects with the $j$th column constraint at the variable
$A_{ij}$. For example, given integers $m=3,n=2,M=5,N=6$, the equivalent
model counting problem has clauses
\begin{align*}
 \phi_1 &: A_{11} \vee A_{12} \vee A_{13} \vee A_{14} \vee A_{15} \vee A_{16} , \\
 \phi_2 &: A_{21} \vee A_{22} \vee A_{23} \vee A_{24} \vee A_{25} \vee A_{26} , \\
 \phi_3 &: A_{31} \vee A_{32} \vee A_{33} \vee A_{34} \vee A_{35} \vee A_{36}, \\
 \phi_4 &: A_{11} \vee A_{21} \vee A_{31} \vee A_{41} \vee A_{51} , \\
 \phi_5 &: A_{12} \vee A_{22} \vee A_{32} \vee A_{42} \vee A_{52} .
\end{align*} 
The intersection graph of that formula is show in Figure~\ref{fig:intgraphex}.
\begin{figure}
  \centering
  \begin{tikzpicture}[very thick]
    \tikzstyle{every node}=[draw=black,rectangle,rounded corners,minimum size=4pt,inner sep=4pt]

    \node (c1) at (0,0) {$\phi_1$};
    \node (c2) at (0,-1) {$\phi_2$};
    \node (c3) at (0,-2) {$\phi_3$};
    \node (c4) at (2,-0.5) {$\phi_4$};
    \node (c5) at (2,-1.5) {$\phi_5$};

    \draw (c1) -- (c4);
    \draw (c2) -- (c4);
    \draw (c3) -- (c4);
    \draw (c1) -- (c5);
    \draw (c2) -- (c5);
    \draw (c3) -- (c5);

  \end{tikzpicture}
  \caption{Intersection graph for the LinMonCBPC formula   described in the text.}
  \label{fig:intgraphex}
\end{figure}
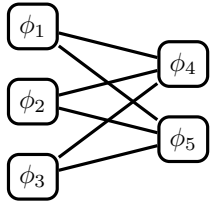
Note that for the complexity of both problems be equivalent we must have
the integers in the matrix problem be given in unary notation (otherwise
building the equivalent formula takes time exponential in the input).


\subsection{Counting edge covers and its connection to model counting} \label{edgecovers}

A \emph{black-and-white graph} (bw-graph) is a triple $G=(V,E,\chi)$
where $(V,E)$ is a simple undirected graph and
$\chi: V \rightarrow \{0,1\}$ is binary valued function on the node set
(assume 0 means white and 1 means black).\footnote{In \cite{Lin2014}
  and \cite{Liu2014}, graphs are uncolored, but edges might contain
  empty endpoints.  These are analogous to white node endpoints in our
  terminology. We prefer defining coloured graphs and allow only simple
  edges to make our framework close to standard graph
  theory terminology. } We denote by $E_G(u)$ the set of edges incident
in a node $u$, and $N_G(u)$ the open neighborhood of $u$ (i.e., not
including $u$). Let $G=(V,E,\chi)$ be a bw-graph. An edge
$e=(u,v) \in E$ can be classified into one of three
categories:\footnote{The classifications of edges given here are
  analogous to those defined in \cite{Lin2014,Liu2014}, but not
  fully equivalent. Regular edges are analogous to the \emph{normal
    edges} defined in \cite{Lin2014,Liu2014}.}
\begin{itemize} 
\item {\bf free  edge:} if $\chi(u)=\chi(v)=0$; 
\item {\bf dangling edge:} if  $\chi(u)\neq\chi(v)$; or 
\item {\bf regular edge:} if $\chi(u)=\chi(v)=1$.
\end{itemize}
In the graph in Figure~\ref{fig:graph3-2}(b), the edge $(f,g)$ is a
dangling edge while the edge $(g,j)$ is a free edge. The edge $(f,g)$ in
the graph in Figure~\ref{fig:graph3-2}(a) is a regular edge.
 
An \emph{edge cover} of a bw-graph $G$ is a set $C \subseteq E$ such
that for each node $v \in V$ with $\chi(v)=1$ there is \emph{at least
  one} edge $e \in C$ incident in it. An edge cover for the graph in
Figure~\ref{fig:graph3-2}(a) is $\{(a,d),(d,g),(e,g),(f,g),(h,j)\}$. We
denote by $Z(G)$ the number of edge covers of a bw-color graph
$G$. Computing $Z(G)$ is $\#\mathsf{P}$-complete \cite{Cai2012}, and admits an
FPTAS~\cite{Lin2014,Liu2014}. 

Consider a LinMonCBPC formula and let $(L,R,E_{LR})$ be its intersection
graph, where $L$ and $R$ are the two partitions. Call $s_L$ and $s_R$
the sizes of a clause in $L$ and $R$, respectively (by construction, all
clauses in the same part have the same size), and let $k_L=s_L-|R|$ and
$k_R=s_R-|L|$. The value of $k_L+k_R$ is the number of variables that
appear in a single clause. Since the graph is bipartite complete,
$k_L,k_R \geq 0$. Obtain a bw-graph $G=(V_1 \cup V_2 \cup V_3 \cup V_4,E,\chi)$ such
that
\begin{enumerate}
\item $V_1=\{1,\dotsc,k_L\}$, $V_2=L$, $V_3=R$ and
  $V_4=\{1,\dotsc,k_R\}$;
\item All nodes in $V_1 \cup V_4$ are white, and all nodes in
  $V_2 \cup V_3$ are black; 
\item There is an edge connecting $(u,v)$ in $E$ for every $u \in V_1$
  and $v \in V_2$, for every $(u,v) \in E_{LR}$, and for every
  $u \in V_3$ and $v \in V_4$.
\end{enumerate}
We call $\mathcal{B}$ the family of graphs that can obtained by the
procedure above. Figure~\ref{fig:graph3-2}(a) depicts an example of a
graph in $\mathcal{B}$ obtained by applying the procedure to the formula
represented in the Figure~\ref{fig:intgraphex}. By construction, for any
two nodes $u,v \in V_i$, $i=1,\dotsc,4$, it follows that $N_G(u)=N_G(v)$
and $(u,v) \not\in E$. The following result shows the equivalence
between edge covers and model counting.

\begin{Proposition} \label{edge-cover-is-model-counting} Consider a
  LinMonCBPC formula $\phi$ and suppose that $G=(V_1,V_2,V_3,V_4,E,\chi)$ is a
  corresponding bw-graph in $\mathcal{B}$. Then number of edge covers of
  $G$ equals the model counting of $\phi$, that is, $Z(G)=Z(\phi)$.
\end{Proposition}
\begin{proof}
  Let $u_i$ denote the node in $G$ corresponding to a clause $\phi_i$
  in $\phi$. Label each edge $(u_i,v_j)$ for $\phi_i \in L$ and
  $\phi_j \in R$ with the variable corresponding to the intersection of
  the two clauses. For each $\phi_i \in L$, label each dangling edge
  $(u,u_i)$ incident in $u_i$ with a different variable that appears
  only at $\phi_i$. Similarly, label each dangling edge $(u_j,u)$ with a
  different variable that appears only at $\phi_j\in R$. Note that the
  labeling function is bijective, as every variable in $\phi$ labels
  exactly one edge of $G$.

  Now consider a satisfying assignment $\sigma$ of $\phi$ and let $C$ be
  set of edges labeled with the variables $X_i$ such that
  $\sigma(X_i)=1$. Then $C$ is an edge cover since every clause (node in
  $G$) has at least one variable (incident edge) with $\sigma(X_i)=1$
  and the corresponding edge is in $C$. To show the converse holds,
  consider an edge cover $C$ for $G$, and construct an assignment such
  that $\sigma(X_i)=1$ if the edge labeled by $X_i$ is in $C$ and
  $\sigma(X_i)=0$ otherwise. Then $\sigma$ satisfies $\phi$, since for
  every clause $\phi_i$ (node $u_i$) there is a variable in $\phi_i$
  with $\sigma(X_i)$ (incident edge in $u_i$ in $C$). Since there are as
  many edges as variables, the correspondence between edge covers and
  satisfying assignment is one-to-one.
\end{proof}


\subsection{A dynamic programming approach to counting edge covers} \label{algorithm1}

In this section we derive an algorithm for computing the number of edge
covers of a graph in $\mathcal{B}$.  Let $e$ be an edge and $u$ be a
node in $G=(V,E,\chi)$. We define the following operations and
notation:
\begin{itemize}
\item {\bf edge removal:} $G-e=(V,E \setminus \{e\},\chi)$.
\item {\bf node whitening:}  $G-u=(V,E,\chi')$,  where $\chi'(u)=0$ and
  $\chi'(v)=\chi(v)$ for $v \neq u$.
\end{itemize}
Note that these operations do not alter the node set, and that they are
associative (e.g., $G-e-f=G-f-e$, $G-u-v=G-v-u$, and $G-e-u=G-u-e$).
Hence, if $E=\{e_1,\dotsb,e_d\}$ is a set of edges, we can write $G-E$
to denote $G-e_1-\dotsb-e_d$ applied in any arbitrary order. The same is
true for node whitening and for any combination of node whitening and
edge removal. These operations are illustrated in the examples in
Figure~\ref{fig:graph3-2}.

\begin{figure}
  \centering
  \begin{tikzpicture}
    \tikzstyle{every node}=[minimum size=5pt,inner sep=0,draw,circle]
    \tikzstyle{white}=[fill=white]
    \tikzstyle{black}=[fill=black]

    \begin{scope}
    \node[white,label=left:a] (a) at (0,2) {};
    \node[white,label=left:b] (b) at (0,1) {};
    \node[white,label=left:c] (c) at (0,0) {};
    \node[black,label=above:d] (d) at (1,2) {};
    \node[black,label=above:e] (e) at (1,1) {};
    \node[black,label=above:f] (f) at (1,0) {};
    \node[black,label=above:g] (h) at (2,1) {};
    \node[black,label=above:h] (i) at (2,0) {};
    \node[white,label=right:i] (k) at (3,1) {};
    \node[white,label=right:j] (l) at (3,0) {};

    \draw (a) -- (d);
    \draw (a) -- (e);
    \draw (b) -- (d);
    \draw (b) -- (e);
    \draw (b) -- (f);
    \draw (a) -- (f);
    \draw (c) -- (d);
    \draw (c) -- (e);
    \draw (c) -- (f);
    \draw (d) -- (h);
    \draw (d) -- (i);
    \draw (e) -- (h);
    \draw (e) -- (i);
    \draw (f) -- (h);
    \draw (f) -- (i);
    \draw (h) -- (k);
    \draw (h) -- (l);
    \draw (i) -- (k);
    \draw (i) -- (l);

    \node[draw=none] at (1.5,-0.5) {(a)};
    \end{scope}

    \begin{scope}[xshift=4.1cm] 
    \node[white,label=left:a] (a) at (0,2) {};
    \node[white,label=left:b] (b) at (0,1) {};
    \node[white,label=left:c] (c) at (0,0) {};
    \node[black,label=above:d] (d) at (1,2) {};
    \node[black,label=above:e] (e) at (1,1) {};
    \node[black,label=above:f] (f) at (1,0) {};
    \node[white,label=above:g] (h) at (2,1) {};
    \node[white,label=above:h] (i) at (2,0) {};
    \node[white,label=right:i] (k) at (3,1) {};
    \node[white,label=right:j] (l) at (3,0) {};

    \draw (a) -- (d);
    \draw (a) -- (e);
    \draw (b) -- (d);
    \draw (b) -- (e);
    \draw (b) -- (f);
    \draw (a) -- (f);
    \draw (c) -- (d);
    \draw (c) -- (e);
    \draw (c) -- (f);
    \draw (d) -- (h);
    \draw (e) -- (h);
    \draw (f) -- (h);
    \draw (h) -- (l);

    \node[draw=none] at (1.5,-0.5) {(b)};
    \end{scope}

   \begin{scope}[xshift=8.2cm] 
    \node[white,label=left:a] (a) at (0,2) {};
    \node[white,label=left:b] (b) at (0,1) {};
    \node[white,label=left:c] (c) at (0,0) {};
    \node[black,label=above:d] (d) at (1,2) {};
    \node[black,label=above:e] (e) at (1,1) {};
    \node[black,label=above:f] (f) at (1,0) {};
    \node[white,label=above:g] (h) at (2,1) {};
    \node[white,label=above:h] (i) at (2,0) {};
    \node[white,label=right:i] (k) at (3,1) {};
    \node[white,label=right:j] (l) at (3,0) {};

    \draw (a) -- (d);
    \draw (a) -- (e);
    \draw (b) -- (d);
    \draw (b) -- (e);
    \draw (b) -- (f);
    \draw (a) -- (f);
    \draw (c) -- (d);
    \draw (c) -- (e);
    \draw (c) -- (f);
    \draw (d) -- (i);
    \draw (e) -- (i);
    \draw (f) -- (i);
    \draw (i) -- (l);

    \node[draw=none] at (1.5,-0.5) {(c)};
    \end{scope}

  \end{tikzpicture}
  \caption{(a) A graph $G$ in $\mathcal{B}$. (b) The graph $G-E_G(h)-(i,g)-h-g$. (c) The graph $G-E_G(g)-(i,h)-g-h$.}
  \label{fig:graphs3-2}
\end{figure}
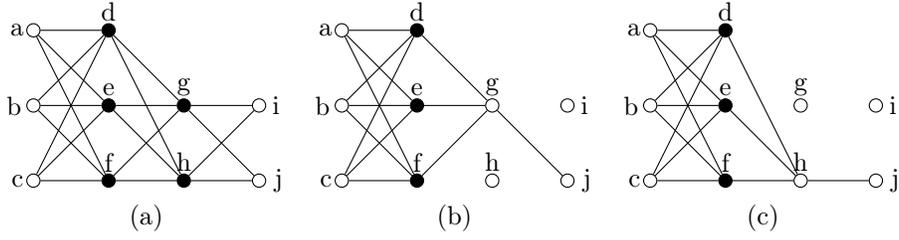

The following result shows that the number of edge covers can be
computed recursively on smaller graphs:
\begin{Proposition} \label{recursion1}
  Let $e=(u,v)$ be a dangling edge with $u$ colored black. Then:
  \[ Z(G) = 2Z(G-e-u) - Z(G-E_G(u)-u) \, . \]
\end{Proposition}

\begin{proof}
  There are $Z(G-e-u)$ edge covers of $G$ that contain $e$ and $Z(G-e)$
  edge covers that do not contain $e$. Hence, $Z(G)=Z(G-e-u)+Z(G-e)$. Now,
  consider the graph $G'=G-e-u$. There are $Z(G-e)$ edge covers of $G'$
  that contain at least one edge of $E_{G'}(u)$ and $Z(G-E_G(u)-u)$ edge
  covers that contain no edge of $E_{G'}(u)$. Thus
  $Z(G-e-u)=Z(G-e)+Z(G-E_G(u)-u)$. Substituting  $Z(G-e)$ in
  the first identity gives us the desired result.
\end{proof}

Free edges and isolated white nodes can be removed by adjusting the edge
count correspondingly:
\begin{Proposition} \label{free-edges}
We have:
  \begin{enumerate}
  \item Let $e=(u,v)$ be a free edge of $G$. Then $Z(G)=2Z(G-e)$. 
  \item If $u$ is an  isolated white node (i.e., $N_G(u)=\emptyset$)
    then $Z(G)=Z(G-u)$.
  \end{enumerate}
\end{Proposition}

\begin{proof}
  (1) If $C$ is an edge cover of $G-e$ then both $C$ and $C \cup \{e\}$
  are edge covers of $G$. Hence, the number of edge covers containing
  $e$ equals the number $Z(G-e)$ of edge covers not containing $e$. (2)
  Every edge cover of $G$ is also an edge cover of $G-u$ and vice-versa.
\end{proof}

We can use the formulas in Propositions~\ref{recursion1} and
\ref{free-edges} to compute the edge cover count of a graph
recursively. Each recursion computes the count as a function of the
counts of two graphs obtained by the removal of edges and whitening of
nodes. Such a naive approach requires an exponential number of
recursions (in the number of edges or nodes of the initial graph) and
finishes after exponential time. We can transform such an approach into
a polynomial-time algorithm by exploiting the symmetries of the graphs
produced during the recursions. In particular, we take advantage of the
invariance of edge cover count to isomorphisms of a graph, as we discuss
next.

We say that two bw-graphs $G=(V,E,\chi)$ and $G'=(V',E',\chi')$ are
\emph{isomorphic} if there is a bijection $\gamma$ from $V$ to $V'$ (or
vice-versa) such that
%
(i) $\chi(v) = \chi'(\gamma(v))$ for all $v \in V$, and
(ii) $(u,v) \in E$ if and only if $(\gamma(u),\gamma(v)) \in E'$.
In other words, two bw-graphs are isomorphic if there is a
color-preserving renaming of nodes that preserves the binary relation
induced by $E$.  The function $\gamma$ is called an \emph{isomorphism}
from $V$ to $V'$. The graphs in Figures \ref{fig:graph3-2}(b) and
\ref{fig:graph3-2}(c) are isomorphic by an isomorphism that maps $g$ in
$h$ and maps any other node into itself. If $C$ is an edge cover of $G$
and $\gamma$ is an isomorphism between $G$ and $G'$, then
$C'=\{ (\gamma(u),\gamma(v)): (u,v) \in C \}$ is an edge cover for $G'$
and vice-versa. Hence, $Z(G)=Z(G')$. The following result shows how to
obtain isomorphic graphs with a combination of node whitenings and edge
removals.

\begin{Proposition} \label{isonodes} %
  Consider a bw-graph $G$ with nodes $v_1,\dotsc,v_n$, where
  $N_G(v_1) = \dotsb = N_G(v_n) \neq \emptyset$ and
  $\chi_G(v_1)=\dotsb=\chi_G(v_n)$. For any node $w \in N_G(v_1)$, 
  mapping $\gamma: \{v_1,\dotsc,v_n\} \rightarrow \{v_1,\dotsc,v_n\}$,
  and nonnegative integers $k_1$ and $k_2$ such that $k_1+k_2 \leq n$
  the graphs
  \( G' = G-E_G(v_1)-\dotsb-E_G(v_{k_1})-(w,v_{k_1+1})-\dotsb-(w,v_{k_1+k_2})-v_1-\dotsb-v_{k_1+k_2} \)
  and
  \( G'' = G-E_G(\gamma(v_1))-\dotsb-E_G(\gamma(v_{k_1}))-(w,\gamma(v_{k_1+1}))-\dotsb-(w,\gamma(v_{k_1+k_2}))-\gamma(v_1)-\dotsb-\gamma(v_{k_1+k_2}) \)
  are isomorphic.
\end{Proposition}

\begin{proof}
  Let $\gamma'$ be the bijection on the nodes of $G$ that extends
  $\gamma$, that is, $\gamma'(u)=u$ for $u \not\in \{v_1,\dotsc,v_n\}$
  and $\gamma'(u)=\gamma(v_i)$, for $i=1,\dotsc,n$.  We will show that
  $\gamma'$ is an isomorphism from $G'$ to $G''$. First note that
  $\chi_G(u)=\chi_G(\gamma(u))$ for every node $u$. The only nodes that
  have their color (possibly) changed in $G'$ with respect to $G$ are
  the nodes $v_1,\dotsc,v_{k_1+k+2}$, and these are white nodes in
  $G'$. Likewise, the only nodes that would (possibly) changed color in
  $G''$ were $\gamma(v_1),\dotsc,\gamma(v_{k_1+k_2})$ and these are
  white in $G''$. Hence, $\chi_{G'}(u)=\chi_{G''}(\gamma(u))$ for every
  node $u$.

  Now let us look at the edges. First note that since $N_G(v_i)$ is
  constant through $i=1,\dotsc,n$, $G'$ and $G''$ have the same number
  of edges. Hence, it suffices to show that for each edge $(u,v)$ in
  $G'$ the edge $(\gamma'(u),\gamma'(v))$ is in $G''$. The only edges
  modified in obtaining $G'$ and $G''$ are, respectively, those
  incident in $v_1,\dotsc,v_{k_1+k_2}$ and in
  $\gamma(v_1),\dotsc,\gamma(v_{k_1+k_2})$. Consider an edge $(u,v)$
  where $u,v \not\in \{v_1,\dotsc,v_n\}$ (hence not in $E_G(v_i)$ for
  any $i$). If $(u,v)=(\gamma'(u),\gamma'(v))$ is in $G'$ then it is
  also in $G''$. 
  Now consider an edge $(u,v_i)$ in $G$ where
  $u \not\in \{w,v_{k_1+1},\dotsc,v_n\}$ and $k_1< i \leq k_1+k_2$. Then
  $(u,v_i)$ is in $G'$ and $(\gamma'(u),\gamma'(v_i))$ is in $G''$. Note
  that $u$ could be in $N_G(v_i)$ for $k_1+k_2<i \leq n$.
\end{proof}

According to the proposition above, the graphs in
Figures~\ref{fig:graph3-2}(b) and \ref{fig:graph3-2}(c) are isomorphic
by a mapping from $g$ to $h$ (and with $w=i$). Hence, the number of edge
covers in either graph is the same.


The algorithms $\mathsf{RightRecursion}$ and $\mathsf{LeftRecursion}$
described in Figures~\ref{edgecount} and \ref{edgecount2}, respectively,
exploit the isomorphisms described in Proposition~\ref{isonodes} in
order to achieve polynomial-time behavior when using the recursions in
Propositions~\ref{recursion1} and \ref{free-edges}. Either algorithm
requires a base white node $w$ and integers $k_1$ and $k_2$ specifying
the recursion level (with the same meaning as in Proposition
\ref{isonodes}). Unless $k_1+k_2$ equals the number of neighbors of $w$
in the original graph, a call to either algorithm generates two more
calls to the same algorithm: one with the graph obtained by removing
edge $(w,v_h)$ and whitening $v_h$, and another by removing edges
$E(v_h)$ and whitening $v_h$. Assume that $|V_2|\geq|V_3|$ (if
$|V_3|>|V_2|$ we can simply manipulate node sets to obtain an isomorphic
graph satisfying the assumption). The $\mathsf{RightRecursion}$
algorithm first checks whether the value for the current recursion level
has been already computed; if yes, then it simply returns the cached
value; otherwise it uses the formula in Proposition~\ref{recursion1}
(and possibly the isomorphism in Proposition~\ref{isonodes}) and
generates two calls of the same algorithm on smaller graphs (i.e. with
fewer edges) to compute the edge cover counting for the current graph
and stores the result in memory. The recursion continues until the
recursion levels equates with the number of nodes in $V_3$, in which
case it checks for free edges and isolated nodes, removes them and
computes the correction factor $2^k$, where $k$ is the number of free
edges, and calls the algorithm $\mathsf{LeftRecursion}$ to start a new
recursion. At this point the graph in the input is bipartite complete
and contains only nodes in $V_1$ and $V_2$. The latter algorithm behaves
very similarly to the former except at the termination step. When all
neighbors $v_h$ of $w$ have been whitened the graph no longer contains
black nodes, and the corresponding edge cover count can be directly
computed using the formulas in Proposition~\ref{free-edges}. Note that a
different cache function must be used when we call
$\mathsf{LeftRecursion}$ from $\mathsf{RightRecursion}$ (this can be
done by instantiating an object at that point and passing it as
argument; we avoid stating the algorithm is this way to avoid
cluttering).


\renewcommand{\algorithmicrequire}{\textbf{Input:}}
\begin{figure}
\begin{mdframed}
\begin{algorithmic}[1]
  \IF{$\mathsf{Cache}(w,k_1,k_2)>0$}
  \RETURN $\mathsf{Cache}(w,k_1,k_2)$
  \ELSE
  \IF{$k_1+k_2 < m$}
  \STATE Let $h \gets k_1 + k_2 + 1$
  \STATE $\mathsf{Cache}(w,k_1,k_2) \gets 2 \times \mathsf{RightRecursion}(G-(v_h,w)-v_h,w,k_1,k_2+1)-\mathsf{RightRecursion}(G-E_G(v_h)-v_h,w,k_1+1,k_2)$
  \RETURN $\mathsf{Cache}(w,k_1,k_2)$
  \ELSE
  \STATE Let $k=|\{(u,v): u \in V_4\}|$ be the number of free edges
  \STATE Remove any edges with an endpoint in $V_4$ and all the resulting isolated nodes
  \STATE Set $V_1 \gets V_1 \cup V_3$, $V_3 \gets \emptyset$
  \IF{$V_1$ is empty}
  \RETURN 0
  \ENDIF
  \STATE Select an arbitrary $w' \in V_1$
  \RETURN $2^k \times \mathsf{LeftRecursion}(G,w',0,0)$
  \ENDIF
  \ENDIF
\end{algorithmic}
\end{mdframed}
\caption{Algorithm $\mathsf{RightRecursion}$: Takes a graph $G=(V_1,V_2,V_3,V_4,E)$ with $V_3=\{v_{1},\dotsc,v_{m}\}$, $m>0$, a node $w \in V_4$, and nonnegative integers $k_1$ and $k_2$; outputs $Z(G)$.}
\label{edgecount}
\end{figure}

\begin{figure}
\begin{mdframed}
\begin{algorithmic}[1]
  \IF{$\mathsf{Cache}(w,k_1,k_2)$ is undefined}
  \IF{$k_1+k_2 < n$}
  \STATE Let $h \gets k_1 + k_2 + 1$
  \STATE $\mathsf{Cache}(w,k_1,k_2) \gets 2 \times \mathsf{LeftRecursion}(G-(u_h,w)-u_h,k_1,k_2+1)-\mathsf{LeftRecursion}(G-E_G(u_h)-u_h,k_1+1,k_2)$ 
  \ELSE
  \STATE $\mathsf{Cache}(w,k_1,k_2) \gets 2^{|E|}$
  \ENDIF
  \ENDIF
  \RETURN $\mathsf{Cache}(w,k_1,k_2)$
\end{algorithmic}
\end{mdframed}
\caption{Algorithm $\mathsf{Le
ftRecursion}$: Takes a bipartite graph $G=(V_1,V_2,E)$ with $V_2=\{u_{1},\dotsc,u_{n}\}$, $n>0$, a node $w \in V_1$, nonnegative integers $k_1$ and $k_2$; outputs $Z(G)$.}
\label{edgecount2}
\end{figure}

Note that the algorithms do not use the color of nodes, which hence does
not need to be stored or manipulated. In fact the node whitening
operations ($-v_h$ or $-u_h$) performed when calling the recursion are
redundant and can be neglected without altering the soundness of the
procedure (we decided to leave these operations as they make the
connection with Proposition~\ref{recursion1} more clear).

Figure~\ref{fig:ZG0} shows the recursion diagram of a run
of $\mathsf{RightRecursion}$. Each box in the figure represents a call of
the algorithm with the graph drawn as input. The left child of each box
is the call $\mathsf{RightRecursion}(G-(v_h,w)-v_h,w,k_1,k_2+1)$, and
the right child is the call
$\mathsf{RightRecursion}(G-E_G(v_h)-v_h,w,k_1+1,k_2)$. For instance, the
topmost box represents $\mathsf{RightRecursion}(G_0,w,0,0)$, which
computes $Z(G_0)$ as the sum of $2Z(G_1)$ and $-Z(G_{24})$, which are
obtained, respectively, from the calls corresponding to its left and
right children. The number of the graph in each box corresponds to the
order in which each call was generated. Solid arcs represent non cached
calls, while dotted arcs indicate cached calls. For instance, by the
time $\mathsf{RightRecursion}(G_{24},w,1,0)$ is called,
$\mathsf{RightRecursion}(G_{13},w,0,0)$ has already been computed so the
value of $Z(G_{13})$ is simply read from memory and returned. When
called in the graph in the top, with to rightmost node $w$, and integers
$k_1=k_2=0$, the algorithm computes the partition function $Z(G_0)$ as
the sum of $2Z(G_1)$ and $-Z(G_{24})$, where $G_1$ is obtained from the
removal of edge $(v_1,w)$ and whitening of $v_1$, while $G_{24}$ is
obtained by removing edges $E_{G_1}(v_1)$ and whitening of $v_1$. The
recursion continues until all incident edges on $w$ have been removed,
at which point it removes free edges and isolated nodes and calls
$\mathsf{LeftRecursion}$. The recursion diagram for the call of
$\mathsf{LeftRecursion}(G_4,w,0,0)$ where $w$ is the top leftmost node
of $G_4$ in the figure is shown in Figure~\ref{fig:ZG3}. The semantics
of the diagram is analogous. Note that the recursion of
$\mathsf{LeftRecursion}$ eventually reaches a graph with no black nodes,
for which the edge cover count can be directly computed (in
closed-form).

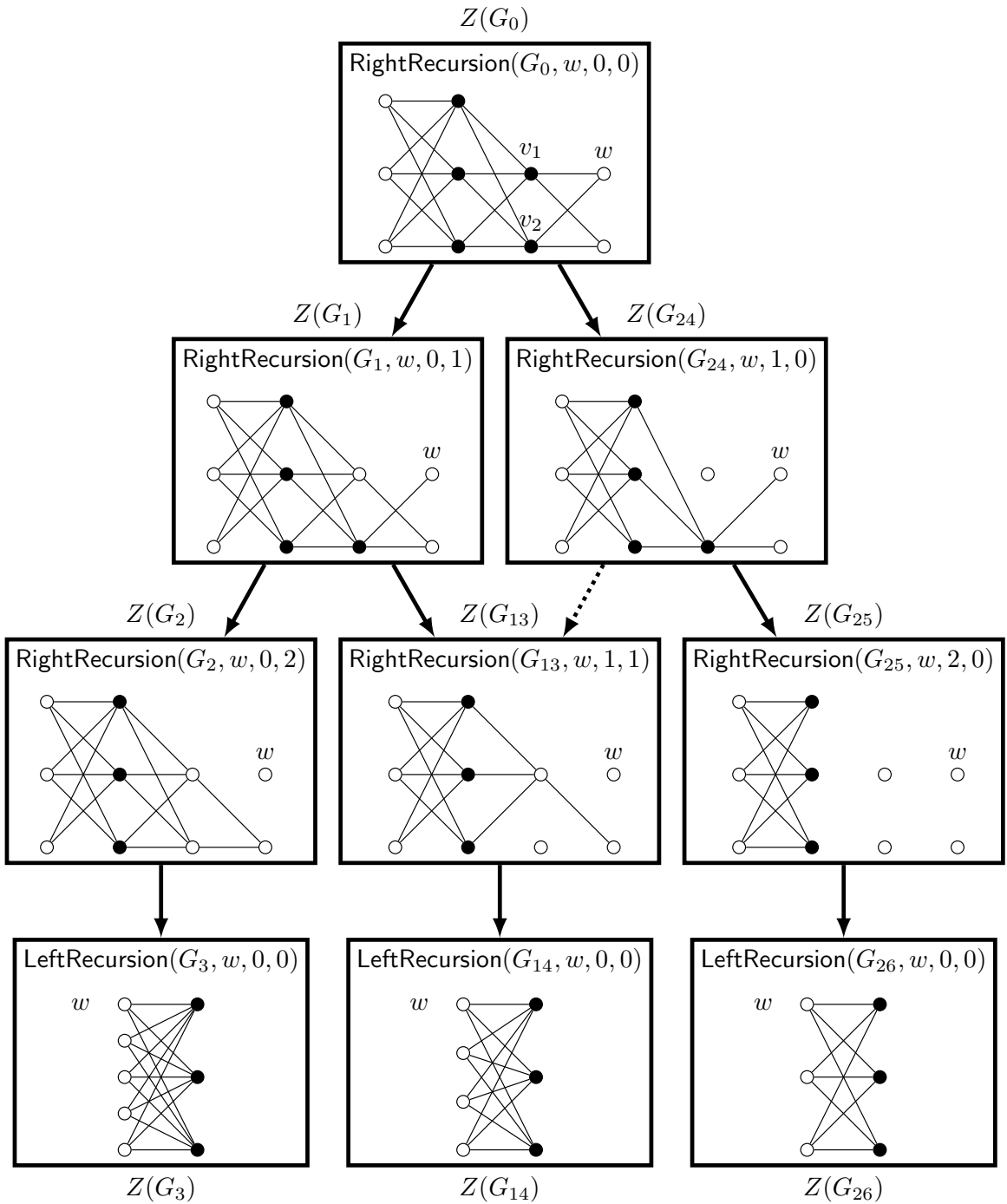
\begin{figure}
  \centering
  \resizebox{\columnwidth}{!}{
    \begin{tikzpicture}[node distance=1cm and -2cm]
      \tikzstyle{node}=[minimum size=5pt,inner sep=0,draw,circle]
      \tikzstyle{white}=[node,fill=white]
      \tikzstyle{black}=[node,fill=black]
      \tikzstyle{graph}=[draw=black,ultra thick,rectangle,minimum width=1cm,anchor=center]

      \node[graph,label=above:{$Z(G_0)$}] (G0)  {%
        \begin{tikzpicture}
          \node[white] (a) at (0,2) {};
          \node[white] (b) at (0,1) {};
          \node[white] (c) at (0,0) {};
          \node[black] (d) at (1,2) {};
          \node[black] (e) at (1,1) {};
          \node[black] (f) at (1,0) {};
          \node[black,label=above:$v_1$] (h) at (2,1) {};
          \node[black,label=above:$v_2$] (i) at (2,0) {};
          \node[white,label=above:{$w$}] (k) at (3,1) {};
          \node[white] (l) at (3,0) {};

          \foreach \from in {a,b,c}
          {
            \foreach \to in {d,e,f}
            {
              \draw (\from) -- (\to);
            }
          }

          \foreach \from in {d,e,f}
          {
            \foreach \to in {h,i}
            {
              \draw (\from) -- (\to);
            }
          }

          \foreach \from in {h,i}
          {
            \foreach \to in {k,l}
            {
              \draw (\from) -- (\to);
            }
          }
          \node[draw=none,inner sep=0,anchor=center] at (1.5,2.5) {$\mathsf{RightRecursion}(G_0,w,0,0)$};
        \end{tikzpicture}%
      };

      \node[graph,below left = of G0,label=above:{$Z(G_1)$}] (G1) {%
        \begin{tikzpicture}
          \node[white] (a) at (0,2) {};
          \node[white] (b) at (0,1) {};
          \node[white] (c) at (0,0) {};
          \node[black] (d) at (1,2) {};
          \node[black] (e) at (1,1) {};
          \node[black] (f) at (1,0) {};
          \node[white] (h) at (2,1) {};
          \node[black] (i) at (2,0) {};
          \node[white,label=above:{$w$}] (k) at (3,1) {};
          \node[white] (l) at (3,0) {};

          \foreach \from/\to in {a/d,a/e,b/d,b/e,b/f,a/f,c/d,c/e,d/h,d/i,e/h,e/i,f/h,f/i,h/l,i/k,i/l}
          { 
            \draw (\from) -- (\to);
          }

          \node[draw=none,inner sep=0,anchor=center] at (1.5,2.5) {$\mathsf{RightRecursion}(G_1,w,0,1)$};
        \end{tikzpicture}%
      };
      \node[graph,below right=of G0,label=above:{$Z(G_{24})$}] (G24)  {%
        \begin{tikzpicture}
          \node[white] (a) at (0,2) {};
          \node[white] (b) at (0,1) {};
          \node[white] (c) at (0,0) {};
          \node[black] (d) at (1,2) {};
          \node[black] (e) at (1,1) {};
          \node[black] (f) at (1,0) {};
          \node[white] (h) at (2,1) {};
          \node[black] (i) at (2,0) {};
          \node[white,label=above:{$w$}] (k) at (3,1) {};
          \node[white] (l) at (3,0) {};

          \foreach \from/\to in {a/d,a/e,b/d,b/e,b/f,a/f,c/d,c/e,d/i,e/i,f/i,i/k,i/l}
          { 
            \draw (\from) -- (\to);
          }

          \node[draw=none,inner sep=0,anchor=center] at (1.5,2.5) {$\mathsf{RightRecursion}(G_{24},w,1,0)$};
        \end{tikzpicture}%
      };
      \node[graph,below left=of G1,label=above:{$Z(G_2)$}] (G2)  {%
        \begin{tikzpicture}
          \node[white] (a) at (0,2) {};
          \node[white] (b) at (0,1) {};
          \node[white] (c) at (0,0) {};
          \node[black] (d) at (1,2) {};
          \node[black] (e) at (1,1) {};
          \node[black] (f) at (1,0) {};
          \node[white] (h) at (2,1) {};
          \node[white] (i) at (2,0) {};
          \node[white,label=above:{$w$}] (k) at (3,1) {};
          \node[white] (l) at (3,0) {};

          \foreach \from/\to in {a/d,a/e,b/d,b/e,b/f,a/f,c/d,c/e,d/h,d/i,e/h,e/i,f/h,f/i,i/l,h/l}
          { 
            \draw (\from) -- (\to);
          }

          \node[draw=none,inner sep=0,anchor=center] at (1.5,2.5) {$\mathsf{RightRecursion}(G_2,w,0,2)$};
        \end{tikzpicture}
      };


      \node[graph,below=of G2,label=below:{$Z(G_3)$}] (G3) {%
        \begin{tikzpicture}
          \node[white,label=left:$w$] (a) at (0,2) {};
          \node[white] (b) at (0,1.5) {};
          \node[white] (c) at (0,1) {};
          \node[black] (d) at (1,2) {};
          \node[black] (e) at (1,1) {};
          \node[black] (f) at (1,0) {};
          \node[white] (g) at (0,0.5) {};
          \node[white] (h) at (0,0) {};

          \foreach \from/\to in {a/d,a/e,b/d,b/e,b/f,a/f,c/d,c/e,c/f,d/g,d/h,e/g,e/h,f/g,f/h}
          { 
            \draw (\from) -- (\to);
          }

          \node[draw=none,inner sep=0,anchor=center] at (0.5,2.5) {$\mathsf{LeftRecursion}(G_3,w,0,0)$};
        \end{tikzpicture}%
      };

      \node[graph,below right=of G1,label=above:{$Z(G_{13})$}]  (G13) {
        \begin{tikzpicture}
          \node[white] (a) at (0,2) {};
          \node[white] (b) at (0,1) {};
          \node[white] (c) at (0,0) {};
          \node[black] (d) at (1,2) {};
          \node[black] (e) at (1,1) {};
          \node[black] (f) at (1,0) {};
          \node[white] (h) at (2,1) {};
          \node[white] (i) at (2,0) {};
          \node[white,label=above:{$w$}] (k) at (3,1) {};
          \node[white] (l) at (3,0) {};

          \foreach \from/\to in {a/d,a/e,b/d,b/e,b/f,a/f,c/d,c/e,d/h,e/h,f/h,h/l}
          { 
            \draw (\from) -- (\to);
          }

          \node[draw=none,inner sep=0,anchor=center] at (1.5,2.5) {$\mathsf{RightRecursion}(G_{13},w,1,1)$};
        \end{tikzpicture}%
      };
      \node[graph,below right=of G24,label=above:{$Z(G_{25})$}] (G25) {%
        \begin{tikzpicture}
          \node[white] (a) at (0,2) {};
          \node[white] (b) at (0,1) {};
          \node[white] (c) at (0,0) {};
          \node[black] (d) at (1,2) {};
          \node[black] (e) at (1,1) {};
          \node[black] (f) at (1,0) {};
          \node[white] (h) at (2,1) {};
          \node[white] (i) at (2,0) {};
          \node[white,label=above:{$w$}] (k) at (3,1) {};
          \node[white] (l) at (3,0) {};

          \foreach \from/\to in {a/d,a/e,b/d,b/e,b/f,a/f,c/d,c/e,c/f}
          { 
            \draw (\from) -- (\to);
          }

          \node[draw=none,inner sep=0,anchor=center] at (1.5,2.5) {$\mathsf{RightRecursion}(G_{25},w,2,0)$};
        \end{tikzpicture}%
      };

      \node[graph,below= of G13,label=below:{$Z(G_{14})$}] (G14) {%
        \begin{tikzpicture}
          \node[white,label=left:$w$] (a) at (0,2) {};
          \node[white] (b) at (0,1.33) {};
          \node[white] (c) at (0,0.66) {};
          \node[black] (d) at (1,2) {};
          \node[black] (e) at (1,1) {};
          \node[black] (f) at (1,0) {};
          \node[white] (g) at (0,0) {};

          \foreach \from/\to in {a/d,a/e,b/d,b/e,b/f,a/f,c/d,c/e,c/f,d/g,e/g,f/g}
          { 
            \draw (\from) -- (\to);
          }

          \node[draw=none,inner sep=0,anchor=center] at (0.5,2.5) {$\mathsf{LeftRecursion}(G_{14},w,0,0)$};
        \end{tikzpicture}%
      };
      \node[graph,below=of G25,label=below:{$Z(G_{26})$}] (G26) {%
        \begin{tikzpicture}
          \node[white,label=left:$w$] (a) at (0,2) {};
          \node[white] (b) at (0,1) {};
          \node[white] (c) at (0,0) {};
          \node[black] (d) at (1,2) {};
          \node[black] (e) at (1,1) {};
          \node[black] (f) at (1,0) {};

          \foreach \from/\to in {a/d,a/e,b/d,b/e,b/f,a/f,c/d,c/e,c/f}
          { 
            \draw (\from) -- (\to);
          }

          \node[draw=none,inner sep=0,anchor=center] at (0.5,2.5) {$\mathsf{LeftRecursion}(G_{26},w,0,0)$};
        \end{tikzpicture}%
      };
      \foreach \i/\j in {0/1,1/2,2/3,1/13,13/14,0/24,24/25,25/26}
      {
        \draw[ultra thick,->,>=latex] (G\i) -- (G\j);
      }

      \foreach \i/\j in {24/13}
      {
      \draw[ultra thick,->,>=latex,dotted] (G\i) -- (G\j);
      }

    \end{tikzpicture}
  }
  \caption{Simulation of $\mathsf{RightRecursion}(G_0,w,0,0)$.}
  \label{fig:ZG0}
\end{figure}

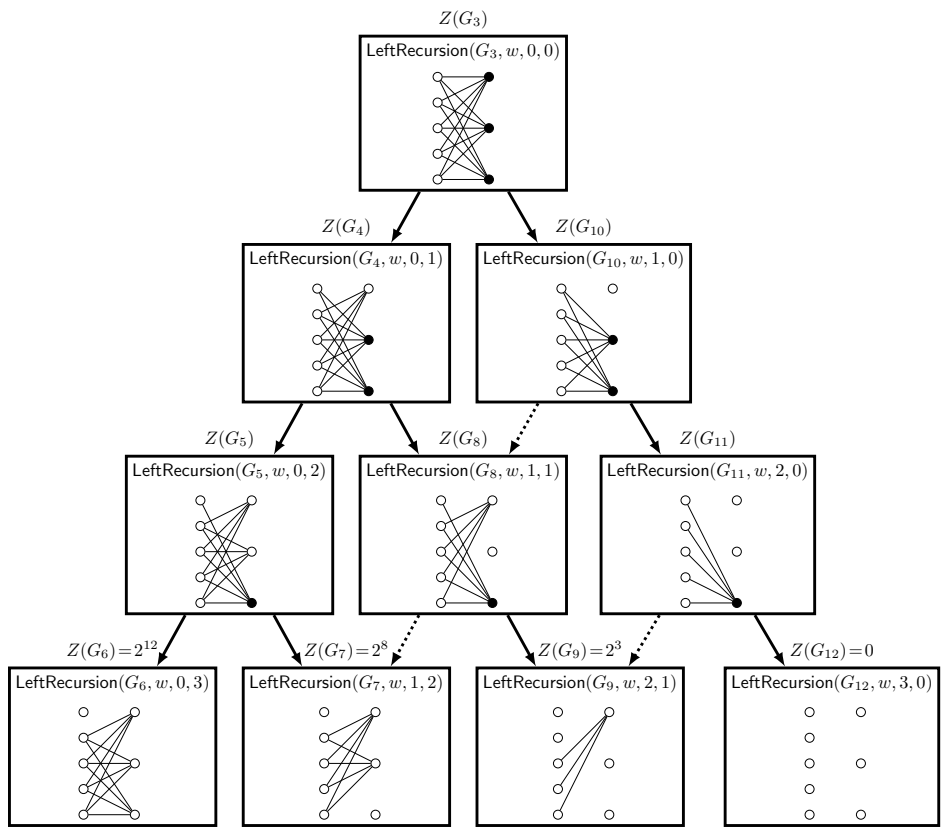
\begin{figure*}
  \centering
  \resizebox{0.8\textwidth}{!}{   
    \begin{tikzpicture}[node distance=1cm and -1.8cm]
      \tikzstyle{node}=[minimum size=5pt,inner sep=0,draw,circle]
      \tikzstyle{white}=[node,fill=white]
      \tikzstyle{black}=[node,fill=black]
      \tikzstyle{graph}=[draw=black,ultra thick,rectangle,minimum width=3.5cm,anchor=center]

      \node[graph,label=above:{$Z(G_3)$}] (G3) {%
        \begin{tikzpicture}
          \node[white] (a) at (0,2) {};
          \node[white] (b) at (0,1.5) {};
          \node[white] (c) at (0,1) {};
          \node[black] (d) at (1,2) {};
          \node[black] (e) at (1,1) {};
          \node[black] (f) at (1,0) {};
          \node[white] (g) at (0,0.5) {};
          \node[white] (h) at (0,0) {};

          \foreach \from/\to in {a/d,a/e,b/d,b/e,b/f,a/f,c/d,c/e,c/f,d/g,d/h,e/g,e/h,f/g,f/h}
          { 
            \draw (\from) -- (\to);
          }

          \node[draw=none,inner sep=0,anchor=center] at (0.5,2.5) {$\mathsf{LeftRecursion}(G_3,w,0,0)$};
        \end{tikzpicture}%
      };
      \node[graph,below left=of G3,label=above:{$Z(G_4)$}] (G4) {%
        \begin{tikzpicture}
          \node[white] (a) at (0,2) {};
          \node[white] (b) at (0,1.5) {};
          \node[white] (c) at (0,1) {};
          \node[white] (d) at (1,2) {};
          \node[black] (e) at (1,1) {};
          \node[black] (f) at (1,0) {};
          \node[white] (g) at (0,0.5) {};
          \node[white] (h) at (0,0) {};

          \draw (a) -- (e);
          \draw (b) -- (d);
          \draw (b) -- (e);
          \draw (b) -- (f);
          \draw (a) -- (f);
          \draw (c) -- (d);
          \draw (c) -- (e);
          \draw (c) -- (f);
          \draw (d) -- (g);
          \draw (d) -- (h);
          \draw (e) -- (g);
          \draw (e) -- (h);
          \draw (f) -- (g);
          \draw (f) -- (h);

          \node[draw=none,inner sep=0,anchor=center] at (0.5,2.5) {$\mathsf{LeftRecursion}(G_4,w,0,1)$};
        \end{tikzpicture}%
      };
      \node[graph,below left=of G4,label=above:{$Z(G_5)$}] (G5) {%
        \begin{tikzpicture}
          \node[white] (a) at (0,2) {};
          \node[white] (b) at (0,1.5) {};
          \node[white] (c) at (0,1) {};
          \node[white] (d) at (1,2) {};
          \node[white] (e) at (1,1) {};
          \node[black] (f) at (1,0) {};
          \node[white] (g) at (0,0.5) {};
          \node[white] (h) at (0,0) {};

          \foreach \from/\to in {b/d,b/e,b/f,a/f,c/d,c/e,c/f,d/g,d/h,e/g,e/h,f/g,f/h}
          {
            \draw (\from) -- (\to);
          }

          \node[draw=none,inner sep=0,anchor=center] at (0.5,2.5) {$\mathsf{LeftRecursion}(G_5,w,0,2)$};
        \end{tikzpicture}%
      };
      \node[graph,below left=of G5,label=above:{$Z(G_6)\!=\!2^{12}$}] (G6) {%
        \begin{tikzpicture}   
          \node[white] (a) at (0,2) {};
          \node[white] (b) at (0,1.5) {};
          \node[white] (c) at (0,1) {};
          \node[white] (d) at (1,2) {};
          \node[white] (e) at (1,1) {};
          \node[white] (f) at (1,0) {};
          \node[white] (g) at (0,0.5) {};
          \node[white] (h) at (0,0) {};

          \foreach \from/\to in {b/d,b/e,b/f,c/d,c/e,c/f,d/g,d/h,e/g,e/h,f/g,f/h}
          {
            \draw (\from) -- (\to);
          }

          \node[draw=none,inner sep=0,anchor=center] at (0.5,2.5) {$\mathsf{LeftRecursion}(G_6,w,0,3)$};
        \end{tikzpicture}%
      };
      \node[graph,below right=of G5,label=above:{$Z(G_7)\!=\!2^8$}] (G7) {%
        \begin{tikzpicture}   
          \node[white] (a) at (0,2) {};
          \node[white] (b) at (0,1.5) {};
          \node[white] (c) at (0,1) {};
          \node[white] (d) at (1,2) {};
          \node[white] (e) at (1,1) {};
          \node[white] (f) at (1,0) {};
          \node[white] (g) at (0,0.5) {};
          \node[white] (h) at (0,0) {};

          \foreach \from/\to in {b/d,b/e,c/d,c/e,d/g,d/h,e/g,e/h}
          {
            \draw (\from) -- (\to);
          }

          \node[draw=none,inner sep=0,anchor=center] at (0.5,2.5) {$\mathsf{LeftRecursion}(G_7,w,1,2)$};
        \end{tikzpicture}%
      };
      \node[graph,below right=of G4,label=above:{$Z(G_8)$}] (G8)  {%
        \begin{tikzpicture}
          \node[white] (a) at (0,2) {};
          \node[white] (b) at (0,1.5) {};
          \node[white] (c) at (0,1) {};
          \node[white] (d) at (1,2) {};
          \node[white] (e) at (1,1) {};
          \node[black] (f) at (1,0) {};
          \node[white] (g) at (0,0.5) {};
          \node[white] (h) at (0,0) {};

          \foreach \from/\to in {a/f,b/d,b/f,c/d,c/f,g/d,g/f,h/d,h/f}
          {
            \draw (\from) -- (\to);
          }

          \node[draw=none,inner sep=0,anchor=center] at (0.5,2.5) {$\mathsf{LeftRecursion}(G_8,w,1,1)$};
        \end{tikzpicture}%
      };
      \node[graph,below right=of G8,label=above:{$Z(G_9)\!=\!2^3$}] (G9)  {%
        \begin{tikzpicture}
          \node[white] (a) at (0,2) {};
          \node[white] (b) at (0,1.5) {};
          \node[white] (c) at (0,1) {};
          \node[white] (d) at (1,2) {};
          \node[white] (e) at (1,1) {};
          \node[white] (f) at (1,0) {};
          \node[white] (g) at (0,0.5) {};
          \node[white] (h) at (0,0) {};

          \foreach \from/\to in {c/d,g/d,h/d}
          {
            \draw (\from) -- (\to);
          }

          \node[draw=none,inner sep=0,anchor=center] at (0.5,2.5) {$\mathsf{LeftRecursion}(G_9,w,2,1)$};
        \end{tikzpicture}%
      };
      \node[graph,below right=of G3,label=above:{$Z(G_{10})$}] (G10) {
        \begin{tikzpicture}
          \node[white] (a) at (0,2) {};
          \node[white] (b) at (0,1.5) {};
          \node[white] (c) at (0,1) {};
          \node[white] (d) at (1,2) {};
          \node[black] (e) at (1,1) {};
          \node[black] (f) at (1,0) {};
          \node[white] (g) at (0,0.5) {};
          \node[white] (h) at (0,0) {};

          \foreach \from/\to in {a/e,b/e,b/f,a/f,c/e,c/f,e/g,e/h,f/g,f/h}
          {
            \draw (\from) -- (\to);
          }
          
          \node[draw=none,inner sep=0,anchor=center] at (0.5,2.5) {$\mathsf{LeftRecursion}(G_{10},w,1,0)$};
        \end{tikzpicture}%
      };
      \node[graph,below right=of G10,label=above:{$Z(G_{11})$}] (G11) {
        \begin{tikzpicture}
          \node[white] (a) at (0,2) {};
          \node[white] (b) at (0,1.5) {};
          \node[white] (c) at (0,1) {};
          \node[white] (d) at (1,2) {};
          \node[white] (e) at (1,1) {};
          \node[black] (f) at (1,0) {};
          \node[white] (g) at (0,0.5) {};
          \node[white] (h) at (0,0) {};

          \foreach \from/\to in {b/f,a/f,c/f,f/g,f/h}
          {
            \draw (\from) -- (\to);
          }
          
          \node[draw=none,inner sep=0,anchor=center] at (0.5,2.5) {$\mathsf{LeftRecursion}(G_{11},w,2,0)$};
        \end{tikzpicture}%
      };
      \node[graph,below right=of G11,label=above:{$Z(G_{12})\!=\!0$}] (G12) {
        \begin{tikzpicture}
          \node[white] (a) at (0,2) {};
          \node[white] (b) at (0,1.5) {};
          \node[white] (c) at (0,1) {};
          \node[white] (d) at (1,2) {};
          \node[white] (e) at (1,1) {};
          \node[white] (f) at (1,0) {};
          \node[white] (g) at (0,0.5) {};
          \node[white] (h) at (0,0) {};
          
          \node[draw=none,inner sep=0,anchor=center] at (0.5,2.5) {$\mathsf{LeftRecursion}(G_{12},w,3,0)$};
        \end{tikzpicture}%
      };

      \foreach \i/\j in {3/4,4/5,5/6,5/7,4/8,8/9,3/10,10/11,11/12}
      {
        \draw[ultra thick,->,>=latex] (G\i) -- (G\j);
      }

      \foreach \i/\j in {8/7,10/8,11/9}
      {
        \draw[ultra thick,->,>=latex,dotted] (G\i) -- (G\j);
      }
    \end{tikzpicture}%
  }
  \caption{Simulation of $\mathsf{LeftRecursion}(G_3,w,0,0)$.}
  \label{fig:ZG3}
\end{figure*}

In these diagrams, it is possible to see how the isomorphisms stated in
Proposition~\ref{isonodes} are used by the algorithms and lead to
polynomial-time behavior. For instance, in the run in
Figure~\ref{fig:ZG0}, the graph $G_{13}$ is not the graph obtained from
$G_{24}$ by removing edge $(v_2,w)$ and whitening $v_2$ but instead is
isomorphic to it. Note that both $G_{13}$ and its isomorphic graph
obtained as the left child of $G_{24}$ were obtained by one operation of
edge removal $-(w,v_i)$ and one operation of neighborhood removal
$-E(v_i)$, plus node whitenings of $v_1$ and $v_2$. Hence,
Proposition~\ref{isonodes} guarantees their isomorphism.

The polynomial-time behavior of the algorithms strongly depends on
caching the calls (dotted arcs) and exploiting known isomorphisms. For
instance, in the run in Figure~\ref{fig:ZG0}, the graph $G_{13}$ is not
the graph obtained from $G_{24}$ by removing edge $(v_2,w)$ and
whitening $v_2$ but instead is isomorphic to it. Note that both $G_{13}$
and its isomorphic graph obtained as the left child of $G_{24}$ were
obtained by one operation of edge removal $(w,v_i)$ and one operation of
neighborhood removal $E(v_i)$, plus node whitenings of $v_1$ and
$v_2$. Hence, Proposition~\ref{isonodes} guarantees their isomorphism.

Without the caching of computations, the algorithm would perform
exponentially many recursive calls (and its corresponding diagram would
be a binary tree with exponentially many nodes). The use of caching
allows us to compute only one call of $\mathsf{RightRecursion}$ for each
configuration of $k_1,k_2$ such that $k_1+k_2 \leq n$, resulting in at
most
\(
\sum_{i=0}^n (i+1) = (n+1)(n+2)/2 = O(n^2)%
\)
calls for $\mathsf{RightRecursion}$, where $n=|V_3|$. Similarly, each
call of $\mathsf{LeftRecursion}$ requires at most
\(
  \sum_{i=0}^m (i+1) = (m+1)(m+2)/2 = O(m^2)%
\)
recursive calls for $\mathsf{LeftRecursion}$, where $m=|V_2|$. Each call
to $\mathsf{RightRecursion}$ with $k_1+k_2=n$ generates a call to
$\mathsf{LeftRecursion}$ (there are $n+1$ such configurations). Hence,
the overall number of recursions (i.e., call to either function) is
\begin{equation*}
  \frac{(n+1)(n+2)}{2}+(n+1)\frac{(m+1)(m+2)}{2} = O(n^2 + n \cdot m^2) \, .
\end{equation*}
This leads us to the following result.

\begin{Theorem} \label{correctness} Let $G$ be a graph in $\mathcal{B}$
  with $w \in V_4 \neq \emptyset$. Then the call 
  $\mathsf{RightRecursion}(G,w,0,0)$ outputs $Z(G)$ in time and memory
  at most cubic in the number of nodes of $G$.
\end{Theorem}
\begin{proof} Except when $k_1+k_2=n$, $\mathsf{RightRecursion}$ calls
  the recursion given in Proposition~\ref{recursion1} with the
  isomorphisms   in Proposition~\ref{isonodes} (any graph obtained
  from $G$ by $k_1$ operations $-E_G(v_i)$ and $k_2$ operations
  $-(w,v_i)$ are isomorpohic). For $k_1+k_2$, any edge left connecting a
  node in $V_3$ and a node in $V_4$ must be a free edge (since all nodes
  in $V_4$ have been whitened), hence they can be removed according to
  Proposition~\ref{free-edges} with the appropriate correction of the
  count. By the same result, any isolated node can be removed. When the
  remaining nodes in $V_3$ are transfered to $V_1$, the resulting graph
  is bipartite complete (with white nodes in one part and black nodes in
  the other). Hence, we can call $\mathsf{LeftRecursion}$, which is
  guaranteed to compute the correct count by the same arguments.

  The cubic time and space behavior is due to $\mathsf{RightRecursion}$
  and $\mathsf{LeftRecursion}$ being called at most $O(n^2)$ and
  $O(nm^2)$, respectively, and by the fact that each call consists of
  local operations (edge removals and node whitenings) which take at
  most linear time in the number of nodes and edges of the graph.
\end{proof}

\subsection{Graphs with no dangling edges} \label{algorithm2}

The algorithm $\mathsf{RightRecursion}$ requires the existence of a
dangling edge. Now it might be that the graph contains no white nodes
(hence no dangling edges), that is, that $G$ is bipartite complete graph
for $V_2 \cup V_3$. The next result shows how to decompose the problem
of counting edge covers in smaller graphs that either contain dangling
edges, or are also bipartite complete.

\begin{Proposition} \label{no-white} Let $G$ be a bipartite complete
  bw-graph with all nodes colored black and $e=(u,v)$ be some edge. Then
  \( Z(G)=2Z(G-e-u-v)-Z(G-E_G(v)-v)-Z(G-E_G(u)-u)-Z(G-E_G(u)-E_G(v)-u-v) \).
\end{Proposition}

\begin{proof}
  The edge covers of $G$ can be partitioned according to whether they
  contain the edge $e$. The number of edge covers that contain $e$ is
  not altered if we color both $u$ and $v$ white. Thus,
  $Z(G)= Z(G-e-u-v) + Z(G-e)$.  Let $e_1,\dotsc,e_n$ be the edges
  incident in $u$ other than $e$, and $f_1,\dotsc,f_m$ be the edges
  incident in $v$ other than $v$.  We have that
  $Z(G-e-u-v)=Z(G-e-u-v)+Z(G-E_G(u)-u)+Z(G-e)+Z(G-E_G(u)-E_G(v)-u-v)$. Substituting
  $Z(G-e)$ into the first equation obtains the result. 
\end{proof}

In the result above, the graphs $G-e-u-v$, $G-E_G(v)-v$ and $G-E_G(u)-u$
are in $\mathcal{B}$ and contain dangling edges, while the graph
$G-E_G(u)-E_G(v)-u-v$ is bipartite complete. Note that Proposition
\ref{isonodes} can be applied to show that altering the edges on which
the operations are applied lead to isomorphic graphs. A very similar
algorithm to $\mathsf{LeftRecursion}$, implementing the recursion in the
result above in polynomial-time can be easily derived.


\subsection{Extensions} \label{extension}

Previous results can be used beyond  the class of graphs
$\mathcal{B}$. For instance, the algorithms can compute the edge cover count
for any graph that can be obtained from a graph $G$ in $\mathcal{B}$ by
certain sequences of edge removals and node whitenings, which includes
graphs not in $\mathcal{B}$. Graphs that satisfy the properties of the
class $\mathcal{B}$ except that every node in $V_2$ (or $V_4$ or both)
are pairwise connected can also have their edge cover count computed by
the algorithm (as this satisfies the conditions in Proposition
\ref{isonodes}). Another possibility is to consider graphs which can be
decomposed in graphs $\mathcal{B}$ by polynomially many applications of
Proposition~\ref{recursion1}.

We can also consider more general forms of counting problems. A simple
mechanism for randomly generating edge covers is to implement a Markov
Chain with starts with some trivial edge cover (e.g. one containing all
edges) and moves from an edge cover $X_t$ to an edge cover $X_{t+1}$ by
the following Glauber Dynamics-type move:
(1) Select an edge $e$ uniformly at random;
(2a) if $e \not\in X_t$, make $X_{t+1} = X_t \cup \{e\}$ with
    probability $\lambda/(1+\lambda)$;
(2b) if $e \in X_t$ and if $X_t \setminus \{e\}$ is an edge cover, make $X_{t+1}=X_t \setminus \{e\}$ with probability $1/(1+\lambda)$;
(2c) else make $X_{t+1}=X_t$.
The above Markov chain can be shown to be ergodic and to converge to a
stationary distribution which samples an edge cover $C$ with probability
$\lambda^{|C|}$ \cite{Borderwich2005,Bezakova2009}. When $\lambda=1$, the
algorithm performs uniform sampling of edge covers. A related problem is
to compute the total probability mass that such an algorithm will assign
to sets of edge covers given a bw-graph $G$, the so-called
\emph{partition function}:
$%
 Z(G,\lambda) = \sum_{C \in \mathsf{EC}(G)} \lambda^{|C|}%
$, 
defined for any real $\lambda>0$, where $\mathsf{EC}(G)$ is the set of
edge covers of $G$. For $\lambda=1$ the problem is equivalent to
counting edge covers. This is also equivalent to weighted model counting
of LinMonCBPC formulas with uniform weight $\lambda$.

The following results are analogous to Propositions \ref{recursion1} and
\ref{free-edges} for computing the partition function:
\begin{Proposition} 
  The following assertions are true:
  \begin{enumerate}
  \item Let $e=(u,v)$ be a free edge of $G$. Then $Z(G)=(1+\lambda)Z(G-e)$. 
  \item If $u$ is an  isolated white node (i.e., $N_G(u)=\emptyset$)
    then $Z(G)=Z(G-u)$.
    \item Let $e=(u,v)$ be a dangling edge with $u$ colored black. Then
  $Z(G) = (1+\lambda)Z(G-e-u) - Z(G-E_G(u)-u)$. 
  \end{enumerate}
\end{Proposition}

Hence, by modifying the
weights by which the the recursive calls are multiplied, we easily
modify algorithms
$\mathsf{RightRecursion}$ and $\mathsf{LeftRecursion}$ () so as to 
compute the partition function of graphs in $\mathcal{B}$
(or equivalently, the partition function of LinMonCBPC formulas).


\end{document}